\documentclass[english]{article}
 
\usepackage{geometry}
\usepackage[T1]{fontenc}    
\usepackage[latin9]{inputenc} 
 
\usepackage{verbatim}
\usepackage{hyperref}       
\hypersetup{colorlinks,linkcolor=red,anchorcolor=blue,citecolor=blue}
\usepackage{url}    
\usepackage{float}

\usepackage{subfigure}
\usepackage{algorithmic} 
\usepackage[ruled,vlined]{algorithm2e}
\usepackage{hyperref}
\usepackage{xr,color}
\usepackage{bm,bbm}
\usepackage{multirow}

\usepackage{amsmath,amsfonts,amssymb,amsthm,graphicx,graphics,subfigure,epsfig,enumerate,bm}

\definecolor{yc}{RGB}{255,0,0}

\let\hat\widehat

\newcommand{\be}{\bm{e}}

\newcommand{\bg}{\bm{g}}
\newcommand{\bh}{\bm{h}}

\newcommand{\bs}{\bm{s}}

\newcommand{\bu}{\bm{u}}
\newcommand{\bv}{\bm{v}}
\newcommand{\bw}{\bm{w}}
\newcommand{\bx}{\bm{x}}
\newcommand{\by}{\bm{y}}
\newcommand{\bz}{\bm{z}}

\newcommand{\bA}{\bm{A}}
\newcommand{\bB}{\bm{B}}

\newcommand{\bD}{\bm{D}}
\newcommand{\bE}{\bm{E}}
\newcommand{\bF}{\bm{F}}

\newcommand{\bH}{\bm{H}}
\newcommand{\bI}{\mathbf{I}}
\newcommand{\bJ}{\bm{J}}

\newcommand{\bQ}{\bm{Q}}
\newcommand{\bR}{\bm{R}}
\newcommand{\bS}{\bm{S}}

\newcommand{\bU}{\bm{U}}
\newcommand{\bV}{\bm{V}}

\newcommand{\bX}{\bm{X}}
\newcommand{\bY}{\bm{Y}}

\newcommand{\cC}{\mathcal{C}}

\newcommand{\cJ}{\mathcal{J}}

\newcommand{\cS}{{\mathcal{S}}}

\newcommand{\bDelta}{\bm{\Delta}}

\newcommand{\bOmega}{\bm{\Omega}}

\newcommand{\dist}{\mathop{\mathrm{dist}}}

\newcommand{\inv}{\mathop{\mathrm{inv}}}

\newcommand{\norm}[1]{\left\| #1\right\|}

\newcommand{\snorm}[1]{\left\| #1\right\|_{\psi_2}}
\newcommand{\sbra}[1]{\left ( #1 \right )}
\newcommand{\mbra}[1]{\left [ #1\right]}

\newcommand{\mbP}[1]{\mathbb{P}\sbra{#1}}
\newcommand{\mbE}[1]{\mathbb{E}\mbra{#1}}
\newcommand{\gd}[1]{\mathcal{N}(0,#1)}
  
\newcommand{\abs}[1]{\left | #1 \right |}

\definecolor{lx}{RGB}{65,105,255}

\theoremstyle{plain} 
\newtheorem{lemma}{\textbf{Lemma}} 
\newtheorem{prop}{\textbf{Proposition}}
\newtheorem{theorem}{\textbf{Theorem}}
\newtheorem{corollary}{\textbf{Corollary}}

\newtheorem{definition}{\textbf{Definition}}
\newtheorem{fact}{\textbf{Fact}} 
\theoremstyle{definition}

\theoremstyle{remark}\newtheorem{remark}{\textbf{Remark}}

\newtheorem*{theorem*}{\textbf{Theorem}}
\newtheorem*{lemma*}{\textbf{Lemma}}	

\allowdisplaybreaks

\usepackage{hyperref}       
\usepackage{url}            
\usepackage{booktabs}       
\usepackage{amsfonts}       
\usepackage{nicefrac}       
\usepackage{microtype}      
\usepackage{cleveref}

\title{Manifold Gradient Descent Solves Multi-Channel Sparse Blind Deconvolution Provably and Efficiently}

\date{}

\author{Laixi Shi \quad \quad \quad Yuejie Chi \\
Department of Electrical and Computer Engineering\\
Carnegie Mellon University \\
Email: \{laixis, yuejiec\}@andrew.cmu.edu\footnote{The work of L. Shi and Y. Chi is supported in part by ONR under the grants N00014-18-1-2142 and N00014-19-1-2404, by NSF under the grants CAREER ECCS-1818571, CCF-1806154 and CCF-1901199, and the Liang Ji-Dian Graduate Fellowship to L. Shi.}
\thanks{The results in this paper were presented in part at the 45th International Conference on Acoustics, Speech, and Signal Processing (ICASSP), May 2020. }}

\makeatletter
\def\endthebibliography{%
	\def\@noitemerr{\@latex@warning{Empty `thebibliography' environment}}%
	\endlist
}
\makeatother

\begin{document}

\maketitle

\begin{abstract}
Multi-channel sparse blind deconvolution, or convolutional sparse coding, refers to the problem of learning an unknown filter by observing its circulant convolutions with multiple input signals that are sparse. This problem finds numerous applications in signal processing, computer vision, and inverse problems. However, it is challenging to learn the filter efficiently due to the bilinear structure of the observations with respect to the unknown filter and inputs, as well as the sparsity constraint. In this paper, we propose a novel approach based on nonconvex optimization over the sphere manifold by minimizing a smooth surrogate of the sparsity-promoting loss function. It is demonstrated that  manifold gradient descent with random initializations will provably recover the filter, up to scaling and shift ambiguity, as soon as the number of observations is sufficiently large under an appropriate random data model. Numerical experiments are provided to illustrate the performance of the proposed method with comparisons to existing ones.
\end{abstract}


\setcounter{tocdepth}{2}
\tableofcontents{}

\section{Introduction}
  
In various fields of signal processing, computer vision, and inverse problems, it is of interest to identify the location of sources from  traces of responses collected from sensors. For example, neural or seismic recordings can be modeled as the convolution of a pulse shape (i.e. a filter), corresponding to characteristics of neuron or earth wave propagation, with a spike train modeling time of activations (i.e. a sparse input) \cite{ekanadham2011recovery,donoho1981minimum,lou2006blind}. Thanks to the advances of sensing technologies, in many applications, one can make multiple observations that share the same filter, but actuated by diverse sparse inputs, either spatially or temporally. Examples include underwater communications \cite{amari1997multichannel, tian2017multichannel}, neuroscience \cite{gitelman2003modeling}, seismic imaging \cite{kaaresen1998multichannel,repetti2014euclid}, image deblurring \cite{zhang2014multi, zhang2013multi}, and so on. The goal of this paper is to identify the filter as well as the sparse inputs by leveraging multiple observations in an efficient manner, a problem termed as multi-channel sparse blind deconvolution (MSBD).

  Mathematically, we model each observation $\by_i \in \mathbb{R}^n$ as a convolution, between a filter $\bg \in\mathbb{R}^n$, and a sparse input, $\bx_i\in\mathbb{R}^{n}$:
\begin{equation}\label{mmv_bd}
\by_i =  \bg \circledast \bx_i = \mathcal{C}(\bg) \bx_i, \quad i=1,\ldots, p,
\end{equation}
where the total number of observations is given as $p $. Here, we consider circulant convolution, denoted as $\circledast$, whose operation is expressed equivalently via pre-multiplying a circulant matrix $ \mathcal{C}(\bg)$ to the input, defined as
\begin{equation}
  \mathcal{C}(\bg)= \begin{bmatrix}
g_1 & g_{n} & \cdots & g_2 \\
g_2 & g_1 & \cdots & g_3 \\
\vdots & \vdots &  \ddots &\vdots \\
g_{n} & g_{n-1} & \cdots & g_{1}
\end{bmatrix}.
\end{equation}
In practice, the circulant convolution is used in situations when the filter $\bg$ satisfies periodic boundary conditions \cite{cho2009fast, yang2009efficient}, or as an approximation of the linear convolution when the filter has compact support or decays fast \cite{strohmer2002four, li2018global}. It is particularly attractive in large-scale problems to accelerate computation by taking advantage of the fast Fourier transform \cite{cho2009fast, yang2009efficient}.

 \subsection{Nonconvex Optimization on the Sphere}
 
Our goal is to recover both the filter $\bg$ and sparse inputs $\{\bm{x}_i\}_{i=1}^p$ from the observations $\{ \by_i\}_{i=1}^p$. The problem is challenging due to the bilinear form of the observations with respect to the unknowns, as well as the sparsity constraint. A direct observation tells that the unknowns are not uniquely identifiable, since for any circulant shift $\mathcal{S}_k(\cdot)$ by $k$ entries (defined in Section~\ref{sec:notation}) and a non-zero scalar $\beta\neq 0$, we have
\begin{equation}
\by_i = \left(\beta \mathcal{S}_k (\bg) \right) \circledast \left(\beta^{-1}\mathcal{S}_{-k}(\bx_i) \right) , 
\end{equation}
for $k=1,\ldots, n-1$. Hence, we can only hope to recover $\bg$ and $\{\bm{x}_i\}_{i=1}^p$ accurately up to certain circulant shift and scaling factor.

In this paper, we focus on the case that $\mathcal{C}(\bm{g})$ is invertible, which is equivalent to requiring all Fourier coefficients of $\bm{g}$ are nonzero. This condition plays a critical role in guaranteeing the identifiability of the model as long as $p$ is large enough \cite{li2015unified}. Under this assumption, there exists a unique inverse filter, $\bg_{\mathrm{inv}}\in\mathbb{R}^n$, such that
\begin{equation}\label{eq:def_equalizer}
\mathcal{C}(\bg_{\mathrm{inv}}) \mathcal{C}(\bg) = \mathcal{C}(\bg)\mathcal{C}(\bg_{\mathrm{inv}}) = \bI.
\end{equation}
   This allows us to convert the bilinear form \eqref{mmv_bd} into a linear form, by multiplying $\mathcal{C}(\bg_{\mathrm{inv}}) $ on both sides:
\begin{align}\label{eq:reformulation}
 \mathcal{C}(\bg_{\mathrm{inv}}) \by_i  & = \mathcal{C}(\bg_{\mathrm{inv}}) \mathcal{C}(\bg) \bx_i \nonumber  = \bx_i, \quad i=1,\ldots, p.
\end{align}   
Consequently, we can equivalently aim to recover $\bg_{\mathrm{inv}}$ via exploiting the sparsity of the inputs $\{\bx_i\}_{i=1}^p$. An immediate thought is to seek a vector $\bh$ that minimizes the cardinality of $ \mathcal{C}(\bh) \by_i  = \mathcal{C}(\by_i) \bh$:
\begin{equation*}
\min_{\bh\in\mathbb{R}^n}\; \frac{1}{p}  \sum_{i=1}^p\left\|\mathcal{C}(\by_i) \bh  \right\|_0 ,
\end{equation*}
where $\|\cdot\|_0$ is the pseudo-$\ell_0$ norm that counts the cardinality of the nonzero entries of the input vector. However, this simple formulation is problematic for two obvious reasons: 
\begin{enumerate}[1)]
\item first, due to scaling ambiguity, a trivial solution is $\bh=\bm{0}$;
\item second, the cardinality minimization is computationally intractable. 
\end{enumerate}
The first issue can be addressed by adding a {\em spherical} constraint $\| \bh\|_2 =1$  to avoid scaling ambiguity. The second issue can be addressed by relaxing to a convex smooth surrogate that promotes sparsity. In this paper, we consider the function
\begin{equation}\label{eq:logcosh} 
\psi_{\mu}(z) = \mu\log\cosh(z/\mu),
\end{equation}
which serves as a convex surrogate of $\|\cdot\|_0$, where $\mu>0$ controls the smoothness of the surrogate. With slight abuse of notation, we assume $\psi_{\mu}(\bm{z}) = \sum_{i=1}^n \psi_{\mu}(z_i)$ is applied in an entry-wise manner, where $\bm{z}=[z_i]_{1\leq i\leq n}$. Putting them together, we arrive at the following optimization problem:
\begin{equation}\label{eq:opt_problem}
    \min_{\bh\in \mathbb{R}^n} \; f_{o}(\bh):= \frac{1}{p}\sum_{i=1}^p \psi_{\mu}(\mathcal{C}(\by_i) \bh)  \quad {\rm s.t.} \quad \norm{\bh}_2=1,
\end{equation}
which is a nonconvex optimization problem due to the sphere constraint. As we shall see later, while this approach works well when $\mathcal{C}(\bg)$ is an orthogonal matrix, further care needs to be taken when it is a general invertible matrix in order to guarantee a benign optimization geometry. Following \cite{li2018global, sun2017complete}, we introduce the following pre-conditioned optimization problem:
\begin{equation}\label{eq:opt_problem_preconditioned}
    \min_{\bh\in \mathbb{R}^n} \;
f(\bh)=\frac{1}{p}\sum_{i=1}^p \psi_{\mu}(\mathcal{C}(\by_i)\bR \bh)  \quad {\rm s.t.} \quad \norm{\bh}_2=1,
\end{equation}
where $\bR$ is a pre-conditioning matrix depending only on the observations $\{\by_i\}_{i=1}^p$ that we will formally introduce in Section~\ref{sec:main}.

 \begin{figure*}[t]
    \centering 
    \begin{tabular}{ccc}
  \hspace{-0.15in}    \includegraphics[width=0.33\textwidth]{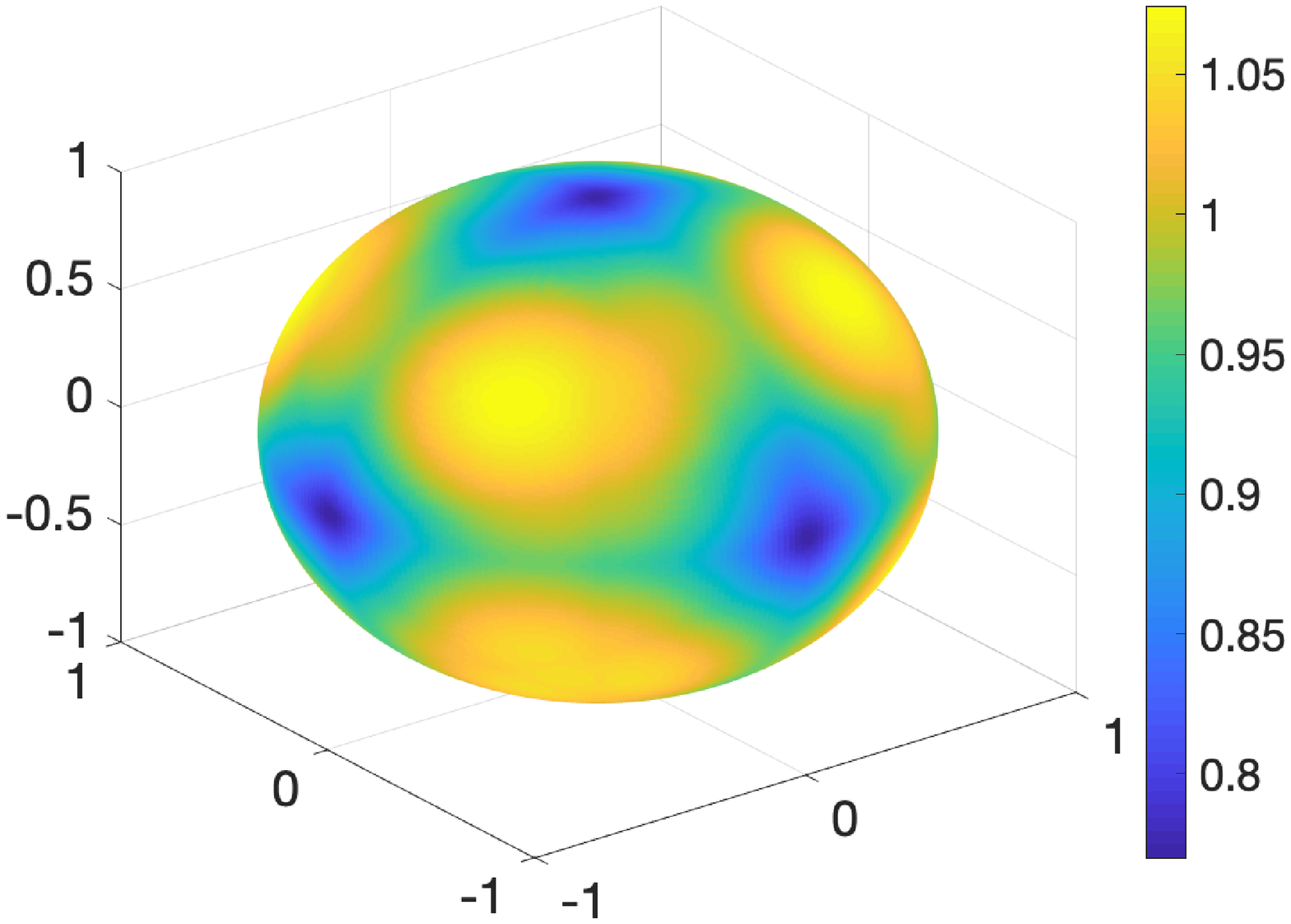}
& \hspace{-0.15in}\includegraphics[width=0.33\textwidth]{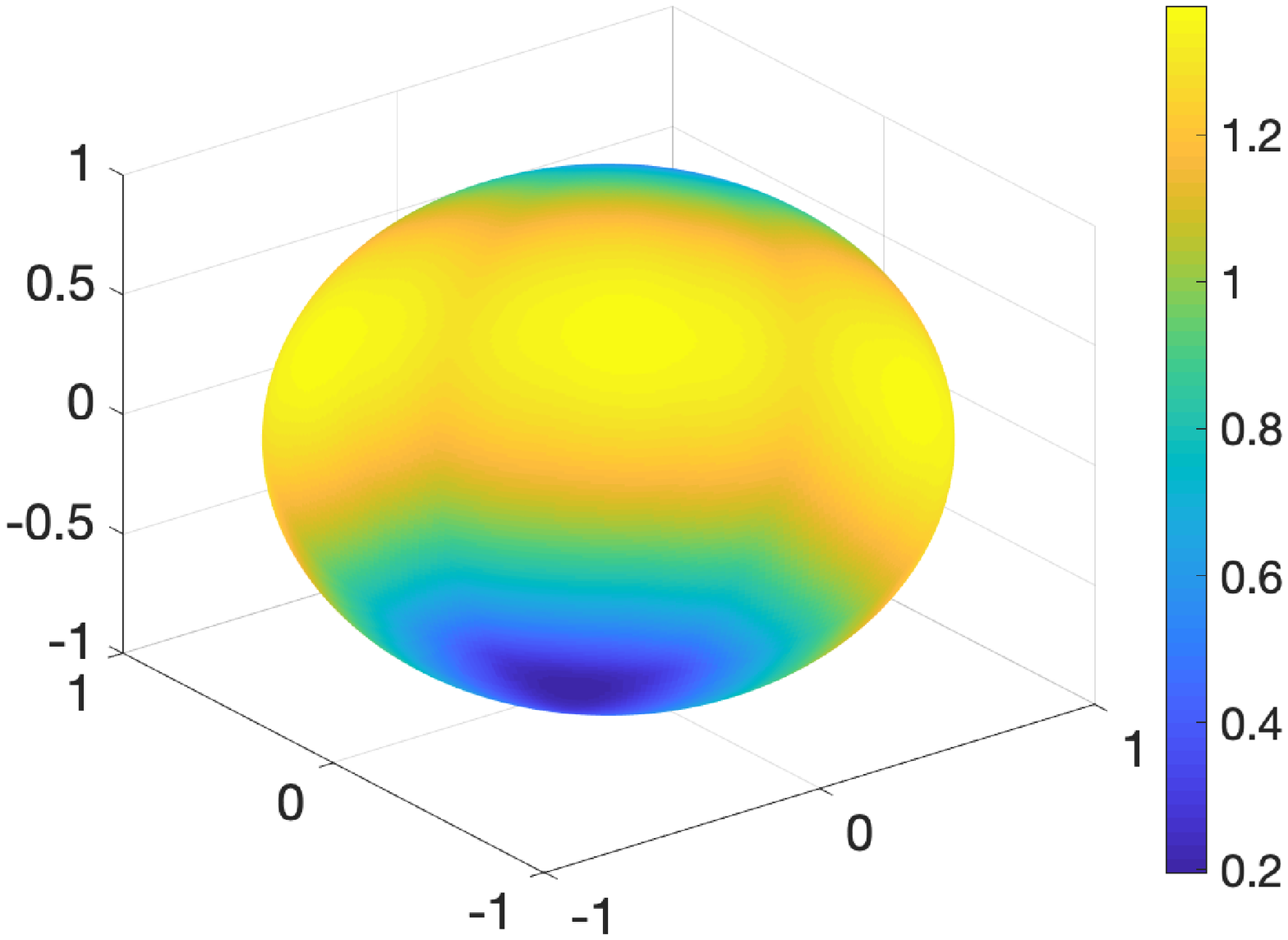}
& \hspace{-0.15in}   \includegraphics[width=0.33\textwidth]{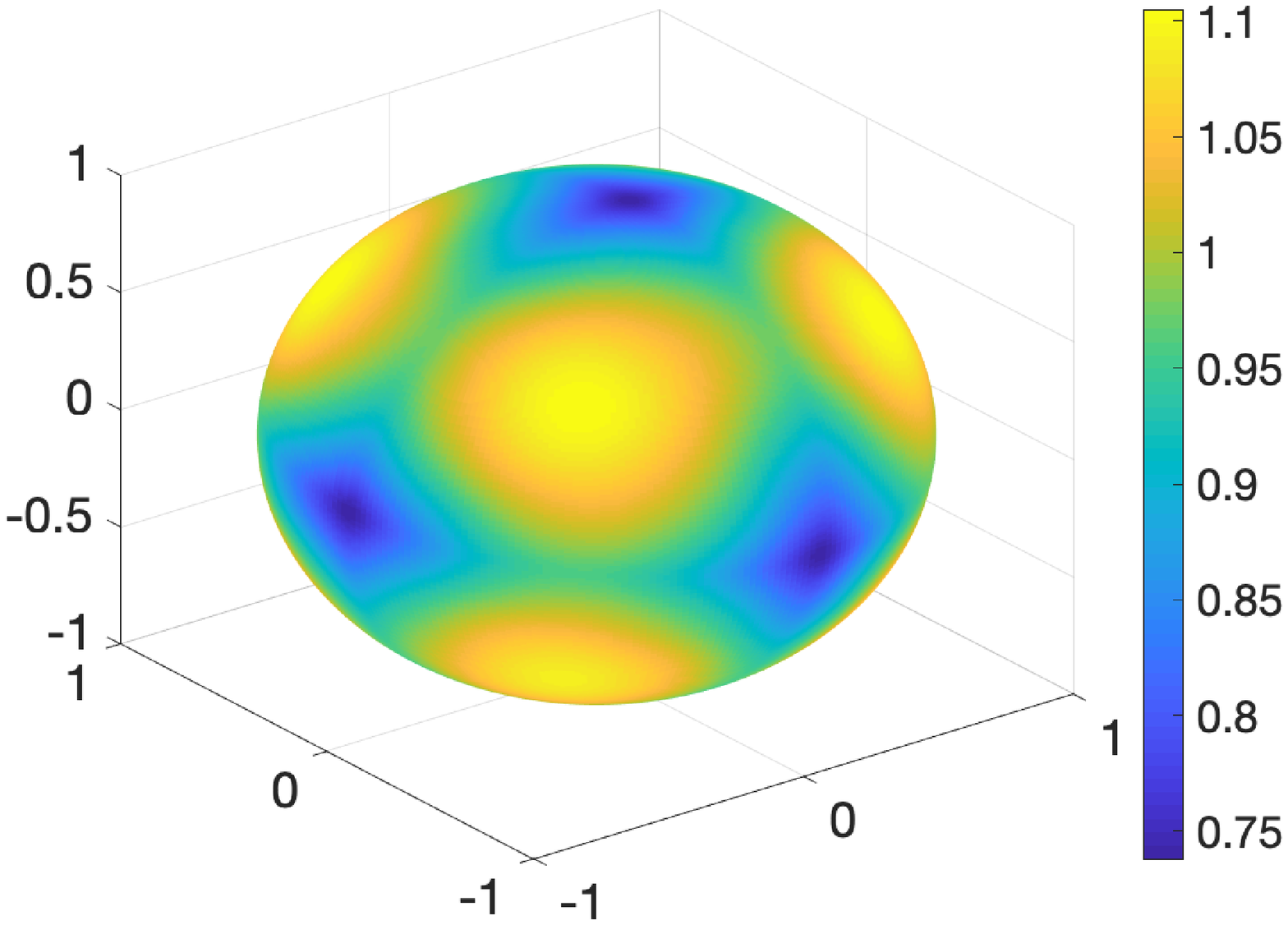}\\  
\hspace{-0.1in} (a) orthogonal filter  &\hspace{-0.2in}  (b) general filter & \hspace{-0.2in}  (c) general filter \\
 no pre-conditioning & no pre-conditioning &  with pre-conditioning
\end{tabular}
    \caption{An illustration of the landscape of the empirical loss function $f_o(\bh)$ or $f(\bh)$ with or without the pre-conditioning matrix $\bR$ in $\mathbb{R}^3$, where the sparse inputs are generated according to a Bernoulli-Gaussian model with $p=30$ observations and activation probability $\theta=0.3$. (a) orthogonal filter $\mathcal{C}(\bg) = \bI$, no pre-conditioning is applied; (b) a general filter, no pre-conditioning is applied; (c) the same general filter as (b) with pre-conditioning. }
    \label{fig:empirical_landscape}
\end{figure*}
 
\subsection{Optimization Geometry and Manifold Gradient Descent}\label{sec:intro_geometry_MGD}

Encouragingly, despite nonconvexity, under a suitable random model of the sparse inputs, the empirical loss functions exhibits benign geometric curvatures as long as the sample size $p$ is sufficiently large. As an illustration, Fig.~\ref{fig:empirical_landscape} shows the landscape of $f_o(\bh)$ and $f(\bh)$ when $n=3$ and $p=30$, and the sparse inputs $\{\bx_i\}_{i=1}^p$ follow the standard Bernoulli-Gaussian model (with an activation probability $\theta=0.3$, see Definition~\ref{def: BG}). When the filter is orthogonal, e.g. $\mathcal{C}(\bm{g})=\bI$, it can be seen from Fig.~\ref{fig:empirical_landscape} (a) that the function $f_o(\bh)$ in \eqref{eq:opt_problem} has benign geometry without pre-conditioning, where the local minimizers are approximately all shift and sign-flipped variants of the ground truth (i.e, the basis vectors), and are symmetrically distributed across the sphere. On the other end, for filters that are not orthogonal, the geometry of $f_o(\bh)$ in \eqref{eq:opt_problem} is less well-posed without pre-conditioning, as illustrated in Fig.~\ref{fig:empirical_landscape} (b). By introducing pre-conditioning, which intuitively stretches the loss surface to mirror the orthogonal case, the pre-conditioned loss function  $f(\bh)$ given in \eqref{eq:opt_problem_preconditioned} for the same non-orthogonal filter used in Fig.~\ref{fig:empirical_landscape} (b) is much easier to optimize over, as illustrated in Fig.~\ref{fig:empirical_landscape} (c). 

Motivated by this benign geometry, it is therefore natural to optimize $\bh$ over the sphere. One simple and low-complexity approach is to minimize $ f(\bh)$ over the sphere via (projected) manifold gradient descent (MGD), i.e. for $k=0,1,\ldots$
\begin{equation} \label{eq:subgrad_descent}
\bh^{(k+1)} : = \frac{\bh^{(k)}  - \eta \partial f(\bh^{(k)})}{\|\bh^{(k)} - \eta \partial f(\bh^{(k)})\|_2},
\end{equation} 
where $\eta$ is the step size, $\partial f(\bh)$ is the Riemannian manifold gradient with respect to $\bh$ (defined in Sec.~\ref{sec:mgd_guarantees}). Surprisingly, this simple approach works remarkably well even with random initializations for appropriately chosen step sizes. As an illustration, Fig.~\ref{fig:GD_example} depicts that MGD converges within a few number of iterations for the problem instance in Fig.~\ref{fig:empirical_landscape} (c). Based on such empirical success, our goal is to address the following question: {\em can we establish theoretical guarantees of MGD to recover the filter for MSBD?}

\begin{figure*}[htp]
    \centering  
 \includegraphics[width=0.48\textwidth]{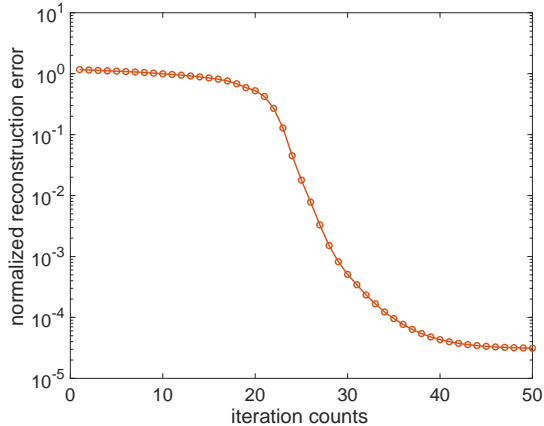}
    \caption{The normalized reconstruction error of MGD with respect to the number of iterations for  the problem instance in Fig.~\ref{fig:empirical_landscape} (c). }
    \label{fig:GD_example}
\end{figure*}

In this paper, we formally establish the benign geometry of the empirical loss function over the sphere, and prove that MGD, with a small number of random initializations, is guaranteed to recover the filter with high probability in polynomial time. Our result is stated informally below.

\begin{theorem}[Informal]
Assume the sparse inputs are generated using a Bernoulli-Gaussian model, where the activation probability $\theta\in (0,1/3)$. As long as the sample size is sufficiently large, i.e. $p= O(\mathrm{poly}(n))$, manifold gradient descent, initialized from at most $O(\log n)$ independently and uniformly selected points on the sphere, recovers the filter accurately with high probability, for properly chosen $\mu$, and step size $\eta_t$.
\end{theorem}
 
Our theorem provides justifications to the empirical success of MGD with random initializations. This result is achieved through an integrated analysis of geometry and optimization. Namely, we identify a union of subsets, corresponding to neighborhoods of equivalent global minimizers, and show that this region has large gradients pointing towards the direction of minimizers. Consequently, if the iterates of MGD lie in this region, and never jump out of it during its execution, we can guarantee that MGD converges to the global minimizers. Luckily, this region is also large enough, so that the probability of a random initialization selected uniformly over the sphere has at least a constant probability falling into the region. By independently initializing a few times, it is guaranteed with high probability at least one of the initializations successfully land into the region of interest and return a faithful estimate of the filter.

\subsection{Paper Organization and Notation}\label{sec:notation}

The rest of this paper is organized as follows. Section~\ref{sec:main} presents the problem formulation and main results, with comparisons to existing approaches. Section~\ref{sec:analysis} outlines the analysis framework and sketches the proof. Section~\ref{sec:numerical} provides numerical experiments on both synthetic and real data with comparisons to existing algorithms. Section~\ref{sec:literature} further discusses the related literature and we conclude in Section~\ref{sec:discussions} with future directions. 

Throughout the paper, we use boldface letters to represent vectors and matrices. Let $\bx^\top$, $\bx^{\mathsf{H}}$ denote the transpose and conjugate transpose of $\bx$, respectively.
 Let $[n]$ denote the index set $\{1,2,\cdots,n\}$. For a vector $\bx\in \mathbb{R}^n$, let $x_{j}$ denote its $j$th element. Let $\bx_{\mathcal{D}}$, $\mathcal{D}\subseteq [n]$ denote the length-$\abs{\mathcal{D}}$ vector composed of the elements in the index set $\mathcal{D}$ of $\bx$, and let $\bx_{\backslash \mathcal{D}}$ denote the vector obtained by removing the elements of $\bx$ in the index set $\mathcal{D}$. For example, $\bx_{1:j}$ denotes the length-$j$ vector composed of the first $j$ entries of $\bx$, i.e., the vector $[x_1,x_2,\cdots, x_j]^\top$, and $\bx_{\backslash \{i\}}$ denotes the length-$(n-1)$ vector composed of all entries of $\bx$ except the $i$th one, i.e. the vector $\bx_{1:i-1,i+1:n}$. If an index $j \notin [n]$ for an $n$-dimensional vector, then the actual index is computed as in the modulo $n$ sense. $\mathcal{S}_j$ denotes a circular shift by $j$ positions, i.e., $[\mathcal{S}_j(\bx)]_{k}=x_{k-j}$ for $j,k\in [n]$. Let $\norm{\cdot}_p$, $p\in[1,\infty]$ represent the $\ell_p$ norm of a vector, and $\|\cdot\|$, $\| \cdot\|_{\mathrm{F}}$ denote the operator norm and the Frobenius norm of a matrix, respectively. Let $\sigma_i(\bA)$ be the $i$th largest eigenvalue of a matrix $\bA$.Let $\odot$ denote the Hadamard product for two vector $\bx,\by\in\mathbb{R}^n$ of the same dimension. Let $\bI$ denote an identity matrix, and $ \be_i\in \mathbb{R}^n,i\in[n]$ be the $i$th standard basis vector. If $\bA\preceq \bB$, then $\bB-\bA$ is positive semidefinite. Last, we use $c_1,c_2, C, \ldots$ to denote universal constants whose values may change from line to line.

\section{Main Results}
\label{sec:main}
 
To begin, we state a few key assumptions. In this paper, we assume that the sparse inputs are generated according to the well-known Bernoulli-Gaussian model, defined below.  

\begin{definition}[Bernoulli-Gaussian model \cite{spielman2012exact}] \label{def: BG}
The inputs $\bx_i$, $i=1,\cdots,p$, are said to satisfy the Bernoulli-Gaussian model with parameter $\theta\in (0,1)$, i.e. $\bx_i\sim_{iid} \mathrm{BG}(\theta)$, if $\bx_i = \bOmega_i \odot \bz_i$, where $\bOmega_i$ is an i.i.d. Bernoulli vector with parameter $\theta$, and $\bz_i$ is a random vector with i.i.d. random Gaussian variables drawn from $\mathcal{N}(0,1)$.
\end{definition}

Furthermore, the geometry of the loss function $f(\bh)$ turns out to be highly related to the condition number of the matrix $\mathcal{C}(\bg)$, which is defined below.

\begin{definition}[Condition number]
Let $\kappa$ be the condition number of $\mathcal{C}(\bg)$, i.e. $\kappa = \sigma_1(\mathcal{C}(\bg)) / \sigma_n(\mathcal{C}(\bg) )$.
\end{definition}

When $\mathcal{C}(\bg)$ is orthogonal, we have $\kappa=1$. Let the discrete Fourier transform (DFT) of $\bg$ be $\hat{\bg}= \bF \bg$, then $\kappa$ is equivalent to
 the ratio of the largest and the smallest absolute values of $\hat{\bg}$, i.e. $\kappa := |\hat{\bg} |_{\max} / |\hat{\bg} |_{\min}$. Therefore, $\kappa$ measures the flatness of the spectrum $\hat{\bg}$, which plays a similar role as the coherence introduced in early works of blind deconvolution with a single snapshot \cite{chi2016guaranteed,ahmed2014blind}. In addition, since $\bg_{\mathrm{inv}}$ can only be identified up to scaling and shift ambiguities, without loss of generality, we assume $\|\bg_{\mathrm{inv}}\|_2=1$.

\subsection{Geometry of the Empirical Loss} \label{sec:results_orghogonal}
 
We start by describing the geometry of $f_o(\bh)$ when $\mathcal{C}(\bg)$ is an orthonormal matrix, where pre-conditioning is not needed. Without loss of generality, we can assume $\mathcal{C}(\bg)=\bI$,\footnote{Denote $ \widetilde{\bh} =\mathcal{C}(\bg) \bh $, we have $ \|\widetilde{\bh}\|_2=\norm{\mathcal{C}(\bg) \bh}_2=1$ due to the orthonormality of $\mathcal{C}(\bg)$. Rewriting the loss function with respect to $\widetilde{\bh}$ confirms this assertion. This does not change the geometry of the objective function that is of primary interest.}  which  corresponds to the ground truth $\bg_{\inv}= \bm{e}_1$ and $\by_i = \bx_i$. Therefore, the loss function $f_o(\bh)$ in \eqref{eq:opt_problem} can be equivalently reformulated as
\begin{equation} \label{eq:equivalent_formulation}
	 \min_{\bh\in \mathbb{R}^n} \;
   f_o(\bh)=\frac{1}{p}\sum_{i=1}^p \psi_{\mu}(\mathcal{C}(\bx_i) \bh) \quad {\rm s.t.} \quad \norm{\bh}_2=1.
\end{equation}
 Our geometric theorem characterizes benign properties of the curvatures in the local neighborhood of $\{\pm \be_i\}_{i=1}^n$, shifted and sign-flipped copies of the ground truth. Inspired by \cite{bai2019subgradient,gilboa2019efficient}, we introduce $2n$ subsets,
\begin{align}\label{eq:subsets}
    \mathcal{S}_{\xi}^{(i\pm)} & =\left \{\bh:h_i \gtrless 0, \frac{h_i^2}{\norm{\bh_{\backslash\{i\}}}_{\infty}^2}\geqslant 1+\xi\right \},   \quad i\in[n],
\end{align}
where $\xi\in[0,\infty)$. Clearly, $\be_i \in  \mathcal{S}_{\xi}^{(i+ )}$ and $ - \be_i \in  \mathcal{S}_{\xi}^{(i- )}$, for all $i\in[n]$. The quantity $\xi$  captures the size of the local neighborhood --- the smaller $\xi$ is, the larger the size of $\mathcal{S}_{\xi}^{(i\pm)}$.

Due to symmetry, we focus on describing the geometry of $f_o(\bh)$ in one of such subsets, say $\mathcal{S}_{\xi}^{(n+)}$. For convenience, we introduce a reparametrization trick \cite{sun2017complete}. Define $\bw=\bh_{1:n-1} \in \mathbb{B}^{n-1}$, corresponding to the first $(n-1)$ entries of $\bh$, where $\mathbb{B}^{n-1}:=\{\bw \in \mathbb{R}^{n-1}: \|\bw\|_2\leq 1\}$ is the unit ball in $\mathbb{R}^{n-1}$. Given $\bw$, the vector $\bh$ can be written as
 \begin{equation} \label{eq:reparam}
        \bh(\bw)=\sbra{\bw,\sqrt{1-\norm{\bw}_2^2}}, \quad \forall \bw\in \mathbb{B}^{n-1}.
\end{equation}
Therefore, $\bw=\bm{0}$ is equivalent to $\bh(\bm{0})=\be_n$, which is the shifted ground truth within $\mathcal{S}_{\xi}^{(n+)}$. The loss function $f_o(\bh)$ can be rewritten with respect to $\bw$ as
\begin{equation}\label{eq:phio_w}
   \phi_o(\bw)= f_o(\bh(\bw))=\frac{1}{p}\sum_{i=1}^p \psi_{\mu}(\mathcal{C}(\bx_i) \bh(\bw)).
\end{equation} 
In addition, a short calculation reveals that,\footnote{When $\bh(\bw)\in \mathcal{S}_{\xi}^{(n+)}$, we have $h_n^2 \geq (1+\xi) \| \bh_{\backslash\{i\}} \|_{\infty}^2$, which leads to 
$1 = \|\bh\|_2^2 \leq h_n^2 + (n-1) \| \bh_{\backslash\{i\}} \|_{\infty}^2 \leq \left( 1+ \frac{n-1}{1+\xi}\right) h_n^2 = \left( 1+ \frac{n-1}{1+\xi}\right)  ( 1- \|\bw\|_2^2)$.}
\begin{equation}\label{eq:range_w} 
\|\bw\|_2^2 \leq  \frac{n-1}{n+\xi} \quad \mbox{whenever} \quad \bh(\bw)\in \mathcal{S}_{\xi}^{(n+)}. 
\end{equation}

The theorem below states the geometry of $\phi_o(\bw)$ in the neighborhood $\bh(\bw)\in\mathcal{S}_{\xi_0}^{(n+)}$ for $\xi_0 \in (0,1)$. In particular, we split the region of interest into two subregions:
\begin{equation}\label{equ:subregions}
        \mathcal{Q}_1 := \left \{\bw: \frac{\mu}{4\sqrt{2}} \leq \norm{\bw}_2\leq \sqrt{\frac{n-1}{n+\xi_0}} \right \}, \quad
        \mathcal{Q}_2 := \left \{\bw: \norm{\bw}_2\leq \frac{\mu}{4\sqrt{2}} \right \} .
\end{equation}
\begin{theorem}[Geometry in the orthogonal case] \label{thm:orthogonal_geometry}
Without loss of generality, suppose $\mathcal{C}(\bg)=\bI$. For any $\xi_0\in (0,1)$, $\theta\in (0,\frac{1}{3})$, there exist constants $c_1,c_2,c_3,c_4, c_5, C$ such that when 
$\mu< c_1 \min\{\theta,    \xi_0^{1/6} n^{-3/4} \}$ and
\begin{equation}\label{eq:sample_size_orthogonal} 
 p\geq \frac{Cn^4  }{\theta^2 \xi_0^2} \log n \log\sbra{ \frac{ n^3 \log p}{\mu\theta \xi_0} }, 
\end{equation} 
the following holds with probability at least $1-c_3 p^{-7} -\exp\sbra{-c_4 n}$ for $\bh(\bw) \in\mathcal{S}_{\xi_0}^{(n+)}$:
\begin{subequations} \label{eq:orthogonal_geometry}
	\begin{align}
	 \mbox{(large directional gradient)}\quad	\frac{\bw^\top\nabla \phi_o(\bw)}{\norm{\bw}_2 }&  \geq c_2 \xi_0 \theta, \quad \mbox{if}\; \bw \in  \mathcal{Q}_1,  \label{res:lg_geometry}\\
	 \mbox{(strong convexity)} \quad \nabla^2 \phi_o(\bw) &\succeq \frac{c_2 n\theta}{\mu} \bI, \quad \mbox{if}\;   \bw \in  \mathcal{Q}_2.\label{res:sc_geometry}
	\end{align}
\end{subequations}	
Furthermore, the function $\phi_o(\bw)$ has exactly one unique local minimizer $\bw_o^{\star}$ near $\bm{0}$, such that
	\begin{equation}\label{res:nearbasis_geometry}
		\norm{\bw_o^{\star}- \bm{0}}_2\leq  \frac{c_5 \mu}{\theta} \sqrt{\frac{\log^2 p}{p}}. 
	\end{equation}
 \end{theorem}
Theorem \ref{thm:orthogonal_geometry} has the following implications when $\bh(\bw) \in\mathcal{S}_{\xi_0}^{(n+)}$, as long as the sample size $p$ is sufficiently large and satisfies \eqref{eq:sample_size_orthogonal}:
\begin{itemize}
 \item The function $\phi_o(\bw)$ either has a large gradient when $\|\bw\|_2$ is large (cf. \eqref{res:lg_geometry}), or is strongly convex when $\|\bw\|_2$ is small (cf. \eqref{res:sc_geometry}), indicating the geometry is rather benign and suitable for optimization using first-order methods such as MGD;
 \item There are no spurious local minima, and the unique local optimizer is close to the ground truth according to \eqref{res:nearbasis_geometry} with an error decays at the rate  $O\sbra{ \frac{\mu}{\theta} \sqrt{\frac{\log^2 p }{ p}}}$ as the sample size $p$ increases.
 
 \end{itemize}
Theorem \ref{thm:orthogonal_geometry} also suggests that a larger sample size is necessary to guarantee a benign geometry when the subset $\mathcal{S}_{\xi_0}^{(i\pm )}$ gets larger -- with the decrease of $\xi_0$. By a simple union bound, we can ensure a similar geometry applies to all $2n$ subsets $\mathcal{S}_{\xi_0}^{(i\pm)}$ defined in \eqref{eq:subsets}.

\paragraph{Extension to the general case.} To extend the geometry in Theorem~\ref{thm:orthogonal_geometry} to the general case when $\mathcal{C}(\bg)$ is invertible, we adopt the trick in \cite{li2018global, sun2017complete} and introduce the pre-conditioning matrix $\bR$:
 \begin{equation}\label{eq:preconditioning}
 	\bR=\mbra{\frac{1}{\theta np}\sum_{i=1}^p \mathcal{C}(\by_i)^\top\mathcal{C}(\by_i)}^{-1/2}.
 \end{equation}
The main purpose of the pre-conditioning is to convert the loss function to one similar to the orthogonal case studied above. Recognizing that $\mbE{\frac{1}{\theta np} \sum_{i =1}^p \mathcal{C}(\bx_i)^\top \mathcal{C}(\bx_i) } = \bI$, we have $\bR \approx \sbra{\cC(\bg)^\top \cC(\bg)}^{-1/2}$ asymptotically as $p$ increases. Plugging $\mathcal{C}(\by_i)=\mathcal{C}(\bx_i)\mathcal{C}(\bg)$ and $\bR \approx \sbra{\cC(\bg)^\top \cC(\bg)}^{-1/2}$ into the loss function of \eqref{eq:opt_problem_preconditioned}, we have
\begin{equation} \label{eq:approx_loss}
f(\bh) \approx \frac{1}{p}\sum_{i=1}^p \psi_{\mu}(\mathcal{C}(\bx_i)\bU \bh)  ,
\end{equation}
where $\bU$ is a circulant orthonormal matrix given by  
\begin{equation} \label{eq:def_U}
	\bU := \cC(\bg)\sbra{\cC(\bg)^\top \cC(\bg)}^{-\frac{1}{2}}.
\end{equation}
By the rotation invariance of the loss function over the sphere with respect to the orthonormal transform by $\bU$, \eqref{eq:approx_loss} is equivalent to the one studied in the orthogonal case, thus justifying our choice of the pre-conditioning matrix. Returning to the original loss function without approximating $\bR$ by its population counterpart, we can repeat the same argument performed in \eqref{eq:equivalent_formulation} and rewrite a rotationed version of (\ref{eq:opt_problem_preconditioned}) as
\begin{equation} \label{equ:equvalent_f2}
	 \min_{\bh\in \mathbb{R}^n} \;
   f(\bh)=\frac{1}{p}\sum_{i=1}^p \psi_{\mu}(\mathcal{C}(\bx_i) \cC(\bg)\bR\bU^{\top} \bh) \quad {\rm s.t} \quad \norm{\bh}_2=1,
\end{equation}
where the shifted and sign-flipped ground truth has been rotated to almost $\{\pm \be_i\}_{i=1}^n$, which is the same as the orthogonal case. The theorem below suggests that under the same reparameterization $\bh=\bh(\bw)$ in \eqref{eq:reparam}, a similar geometry as Theorem~\ref{thm:orthogonal_geometry} can be guaranteed for $\phi(\bw) = f(\bh(\bw))$.

\begin{theorem}[Geometry in the general case] \label{theo:general_geometry}
Suppose $\mathcal{C}(\bg)$ is invertible with condition number $\kappa$. For any $\xi_0\in (0,1)$, $\theta\in (0,\frac{1}{3})$, there exist constants $c_1,c_2,c_3,c_4, C$ such that when 
$\mu< c_1 \min\{\theta,    \xi_0^{1/6} n^{-3/4} \}$ and  
\begin{equation}\label{eq:sample_size_general_1}
p\geq C \frac{\kappa^8 n^3 \log^4 p \log^2 n}{\theta^4 \mu^2 \xi_0^2},
\end{equation}  
the geometry \eqref{eq:orthogonal_geometry} holds for $\phi(\bw)$ with probability at least $1-c_3 p^{-7} -\exp\sbra{-c_4 n}$ for $\bh(\bw) \in\mathcal{S}_{\xi_0}^{(n+)}$. In addition, the function $\phi(\bw)$ has exactly one unique local minimizer $\bw^{\star}$ near $\bm{0}$, such that
 $$  \norm{\bw^{\star} - \bm{0} }_2  \leq  \frac{c_2 \kappa^4}{\theta^2} \sqrt{\frac{n\log^3 p \log^2 n}{p}} .$$
\end{theorem}
 
Theorem \ref{theo:general_geometry} demonstrates that a benign geometry similar to that in Theorem \ref{thm:orthogonal_geometry} can be guaranteed for the general case, as long as a proper pre-conditioning is applied, and the sample size is sufficiently large. In particular, the sample size in \eqref{eq:sample_size_general_1} increases with the increase of  the condition number of $\mathcal{C}(\bg)$. 
 
\subsection{Convergence Guarantees of MGD}\label{sec:mgd_guarantees}

Owing to the benign geometry in the subsets of interest $\left \{\cS_{\xi_0}^{(i\pm)},  i\in [n] \right \}$, a simple MGD algorithm is proposed to optimize \eqref{equ:equvalent_f2}, by updating
\begin{equation}
\bh^{(k+1)}= \frac{\bh^{(k)}-\eta  \partial f  (\bh^{(k)})} {\norm{\bh^{(k)}-\eta  \partial f \sbra{\bh^{(k)}}}_2}
\end{equation}
for $k=1,\ldots, T-1$, where $\partial f(\bh)= (\bI - \bh\bh^{\top}) \nabla f(\bh)$ is the Riemannian manifold gradient with respect to $\bh$, $\nabla f(\bh)$ is the Euclidean gradient of $f(\bh)$, and $\eta$ is the step size. The next theorem demonstrates that with an initialization in one of the $2n$ subsets  $\left \{\cS_{\xi_0}^{(i\pm)},  i\in [n] \right \}$, the proposed MGD algorithm, with a proper step size, will recover $\pm \bm{e}_i$ in that region in a polynomial time.

\begin{theorem}\label{thm:general_GD}
Let $ 0<\epsilon<1$ and instate the assumptions of Theorem~\ref{theo:general_geometry}. If the initialization satisfies $\bh^{(0)}\in \mathcal{S}_{\xi_0}^{(i\pm)}$ for any $i\in [n]$, then with a step size $\eta\leq \frac{c  \mu \xi_0\theta }{  n^2 \sqrt{\log(np)}  } $ for some sufficiently small constant $c$, the iterates $\bh^{(k)}, k=1,2,\cdots$ stay in $\mathcal{S}_{\xi_0}^{(i\pm)}$ and achieve 
$\| \bh^{(T)} \mp \bm{e}_i \|_2 \lesssim  \frac{\kappa^4}{\theta^2} \sqrt{\frac{n\log^3 p \log^2 n}{p}}  +\epsilon$  in  
$$T\lesssim  \frac{n}{\mu \eta \xi_0 \theta}  + \frac{ \mu}{n\theta\eta}\log\sbra{ \frac{ \mu}{ \epsilon} }$$ 
iterations. 
\end{theorem}
With Theorem~\ref{thm:general_GD} in place, one still needs to address how to find an initialization that satisfies $\bh^{(0)}\in \mathcal{S}_{\xi_0}^{(i\pm)}$. Fortunately, setting $\xi_0=1/(4\log n)$ allows a sufficiently large basin of attraction, such that a random initialization can land into it with a constant probability. A few random initializations guarantee that the MGD algorithm will succeed with high probability. This is made precise in the following corollary. 

\begin{algorithm}[t]
\KwIn{Observation $\{\by_i\}_{i=1}^p$, step size $\eta$, initialization $\bh^{(0)}$ on the sphere, the loss function $f(\bh)$ in \eqref{eq:opt_problem_preconditioned}; }
\For{ $ k = 0$ \KwTo $T-1$}{ 
    \emph{ $$\bh^{(k+1)}= \frac{\bh^{(k)}-\eta  \partial f  (\bh^{(k)})} {\norm{\bh^{(k)}-\eta  \partial f \sbra{\bh^{(k)}}}_2} ;$$  } }    
\KwOut{ Return $\widehat{\bm{g}}_{\mathrm{inv}} = \bR \bh^{(T)}$, where $\bR$ is given in \eqref{eq:preconditioning}.}

\caption{Manifold Gradient Descent for MSBD}
\label{algo:1}
\end{algorithm}

Putting everything together, Alg.~\ref{algo:1} summarizes the proposed MGD algorithm for the original loss function in  \eqref{eq:opt_problem_preconditioned}, where the pre-conditioning matrix is applied back at the end of the iterations to produce the final estimate $\widehat{\bg}_{\inv}$ of $\bm{g}_{\mathrm{inv}}$. To measure the success of recovery, we use the following distance metric that takes into account the ambiguities:
\begin{equation}\label{eq:metric_shift}
\dist(\widehat{\bg}_{\inv}, \bg_{\mathrm{inv}}) = \min_{j \in [n]} \|  \bg_{\mathrm{inv}} \pm \cS_j(\widehat{\bg}_{\inv})\|_2  .
\end{equation}
We have the following corollary.
\begin{corollary}[Putting everything together]\label{cor:main_corollary}
Suppose $\mathcal{C}(\bg)$ is invertible with condition number $\kappa$. For $\theta\in (0,\frac{1}{3})$, there exists some constant $c_1$ such that when 
$\mu< c_1 \min\{\theta,    (\log n)^{-1/6} n^{-3/4} \}$ and the sample complexity satisfies
\begin{equation}\label{eq:sample_size_general}
p\gtrsim  \frac{\kappa^8 n^3 \log^3 n \log^4 p}{\theta^4 \mu^2 },
\end{equation}  
with $O(\log n)$ random initializations selected uniformly over the sphere,  the MGD algorithm in Alg.~\ref{algo:1} with a proper step size is guaranteed to obtain a vector $\widehat{\bg}_{\mathrm{inv}}$ that satisfies
 $$  \mathrm{dist}(\widehat{\bg}_{\mathrm{inv}}, \bg_{\mathrm{inv}}) \lesssim \frac{\kappa^4}{\theta^2\sigma_n(\cC(\bg))} \sqrt{\frac{n\log^3 p \log^2 n}{p}} + \frac{\epsilon}{\sigma_n(\cC(\bg))}$$
for any $0<\epsilon<1$, in $O(\mbox{poly}(n) ) $   iterations.   
 \end{corollary}

Corollary~\ref{cor:main_corollary} provides theoretical footings to the success of MGD for solving the highly nonconvex MSBD problem. In particular, consider the interesting regime when $\theta=O(1)$ and $\kappa=O(1)$, it is sufficient to set $\mu \lesssim  (\log n)^{-1/6} n^{-3/4} $, which leads to a sample size $p=O( n^{4.5})$ up to logarithmic factors.

\subsection{Comparisons to Existing Approaches}

The identifiability of the MSBD problem is established in \cite{li2015unified} that says under 
the Bernoulli-Gaussian model (cf. Definition~\ref{def: BG}) on the sparse coefficients, the filter is identifiable with high probability, provided that $\bg$ is invertible, $\theta\in(1/n,1/4)$ and $p\gtrsim n\log n$.  Wang and Chi proposed a linear program in \cite{wang2016blind} that succeeds when $p\gtrsim n\log^4 n$. However, the success of the linear program therein imposes stringent requirements on the conditioning number of the filter $\bg$ and the sparsity level $\theta$.  

Our approach is most related to the subsequent work of Li and Bresler \cite{li2018global}, which runs perturbed MGD with a random initialization, over a spherically constrained loss function based on $\ell_4$ norm maximization. Li and Bresler showed that, when the sample complexity is large enough, the landscape of the loss function does not possess spurious local maxima, and all saddle points admit directions that strictly increase the loss function. However, their sample complexity is significantly worse. Specifically, to reach a similar accuracy as ours, \cite{li2018global} requires $O(n^9)$ samples, while we only require $O(n^{4.5})$ samples ignoring logarithmic factors, leading to an order-of-magnitude improvement. One key observation is that the large sample complexity required by \cite{li2018global} is partially due to bounding the global geometry everywhere over the sphere, through uniform concentration of the gradient and the Hessian of the empirical loss function around their population counterparts, which is sufficient but in fact not necessary to ensure the algorithmic success of MGD. Indeed, motivated by \cite{bai2019subgradient,gilboa2019efficient}, to optimize the sample complexity, we only require the uniform concentration of directional gradient over a large region near the global minimizer, which can be guaranteed at a significantly reduced sample complexity. In addition, this region is large enough so that with a logarithmic number of random initializations we are guaranteed to land into this region with high probability and recover the signal of interest via {\em vanilla} MGD. It is worth pointing out that since we focused on a region without saddle points, no perturbation is needed to ensure the success of MGD, which is another salient difference from \cite{li2018global}. As will be seen in Section~\ref{sec:numerical}, the proposed loss function not only theoretically, but also empirically, outperforms the $\ell_4$ norm used in \cite{li2018global}. 

At the time of finishing this paper, we became aware of another concurrent work \cite{qu2019nonconvex}, which optimized a different smooth surrogate of the $\ell_1$ norm over the sphere for the same problem. Their work \cite{qu2019nonconvex} requires a sample complexity on the order of $O(n^5)$, which is slightly worse than ours (i.e. $O(n^{4.5})$), to guarantee the benign geometry in a similar region near the global optimizer. In addition, a refinement procedure is proposed in \cite{qu2019nonconvex} to allow exact recovery of the filter. Their path to a better sample complexity than \cite{li2018global}  is similar to ours as described above. We expect that their method behaves similar to ours in practice.

\section{Overview of the Analysis}\label{sec:analysis}

In this section, we outline the proof of the main results, while leaving the details to the appendix. We first deal with the simpler case when $\mathcal{C}(\bg)$ is an orthonormal matrix employing the objective function $\phi_o(\bw)$ (i.e, $f_o(\bh)$) without pre-conditioning in Section~\ref{sec:proof_ortho_geometry}, and then extend the analysis to the general case where the objective function $\phi(\bw)$ (i.e, $f(\bh)$) is pre-conditioned in Section~\ref{sec:geometry_general_outline}. Finally, we discuss the convergence guarantee of MGD in Section~\ref{pipe:GD_orthogonal}.

\subsection{Proof Outline of Theorem~\ref{thm:orthogonal_geometry}}\label{sec:proof_ortho_geometry}
 
The proof of Theorem~\ref{thm:orthogonal_geometry} is divided into several steps.
\begin{enumerate}
\item First, we characterize the landscape of the population loss function $\mathbb{E}[\phi_o(\bw)]$; 
\item Second, we prove the pointwise concentration of the directional gradient and the Hessian of the empirical loss $\phi_o(\bw)$ around those of the population one $\mathbb{E}[\phi_o(\bw)]$ in the region of interest; 
\item Third, we extend such concentrations to the uniform sense, thus the benign geometric properties of $\mathbb{E}[\phi_o(\bw)]$ carry over to the empirical version $\phi_o(\bw)$.
\end{enumerate}
 
To begin, the lemma below describes the geometry of $\mathbb{E}[\phi_o(\bw)]$, whose proof is given in  Appendix~\ref{proof:population_geometry_orthogonal}.
 
\begin{lemma}[Geometry of the population loss in the orthogonal case]\label{thm:population_geometry_orthogonal}
Without loss of generality, suppose $\mathcal{C}(\bg)=\bI$. For any $\xi_0\in (0,1)$, $\theta\in (0,\frac{1}{3})$, there exists some constant $c_1$ such that when 
$\mu< c_1 \min\{\theta,    \xi_0^{1/6 } n^{-3/4}\}$, we have for $\bh(\bw) \in\mathcal{S}_{\xi_0}^{(n+)}$:
\begin{subequations}
	\begin{align}
	 \mbox{(large directional gradient)}\quad	\frac{\bw^\top \nabla\mathbb{E} \phi_o(\bw)}{\norm{\bw}_2 }  & \geq \frac{\xi_0 \theta}{480\sqrt{10\pi}}, \quad  \bw \in  \mathcal{Q}_1,  \label{eq:lg_pop_geometry}\\
	 \mbox{(strong convexity)} \quad \nabla^2 \mathbb{E}\phi_o(\bw) & \succeq \frac{n\theta}{5\sqrt{2\pi}\mu} \bI , \quad   \bw \in  \mathcal{Q}_2.\label{eq:sc_pop_geometry}
	\end{align}
\end{subequations}	 
\end{lemma}

 To extend the benign geometry to the empirical loss with a finite sample size $p$, we first need to prove the pointwise concentration of these quantities around their expectations for a fixed $\bw$, using the Bernstein's inequality. The next two propositions demonstrate the pointwise concentration results, whose proofs are provided in Appendix~\ref{proof:pointwise_large_gradient} and \ref{proof:pointwise_strong_convexity}. 
\begin{prop}\label{pro:pointwise_large_gradient}
	For any $\bw$ satisfies $\norm{\bw}_2 \leq \sqrt{ \frac{n-1}{n} }$, there exist some universal constants $C_1$ and $C_2$ such that for any $t>0$:
	\begin{align*}
\mathbb{P}\mbra{\abs{\frac{\bw^\top\nabla \phi_o(\bw)}{\norm{\bw}_2}- \frac{\bw^\top  \nabla \mathbb{E} \phi_o(\bw)}{\norm{\bw}_2 } }\geq t} 	&	\leq 2\exp\sbra{\frac{-p t^2} { C_1n^3\log n+ C_2 t \sqrt{n^3\log n} } }.
	\end{align*}
\end{prop}
 
\begin{prop}\label{pro:pointwise_strong_convexity}
	For any $\bw$ satisfies $\norm{\bw}_2 \leq 1/2$, there exist some universal constants $C_1$ and $C_2$ such that for any $t>0$,
	\begin{align*}
\mathbb{P}\mbra{\norm{\nabla^2 \phi_o(\bw)- \nabla^2\mathbb{E}\phi_o(\bw)}\geq t} 	& \leq 4n \exp\sbra{\frac{-p\mu^2 t^2}{C_1 n^2 \log^2 n + C_2\mu n t \log n }}.
	\end{align*}
\end{prop}

 The concentration of the Hessian and directional gradient between the empirical and population objective functions at a fixed point suggests that the empirical objective function may inherit the benign geometry of the population one outlined in Lemma~\ref{thm:population_geometry_orthogonal}. However, one needs to carefully extend the pointwise concentrations in Propositions~\ref{pro:pointwise_large_gradient} and \ref{pro:pointwise_strong_convexity} through a covering argument, which requires bounding the Lipschitz constants of the Hessian and directional gradients. The rest of the proof of Theorem~\ref{thm:orthogonal_geometry} is provided in Appendix~\ref{proof:orthogonal_geometry_discretization}.

\subsection{Proof Outline of Theorem~\ref{theo:general_geometry}}\label{sec:geometry_general_outline}

To extend the benign geometry to the general case, we show that through pre-conditioning, the landscape of $\phi(\bw)$ is not too far from that of $\phi_o(\bw)$. Recall that the pre-conditioned loss function \eqref{equ:equvalent_f2} is
	\begin{align} \label{equ:define_delta}
		\phi(\bw )&=\frac{1}{p}\sum_{i=1}^p \psi_{\mu}\left(\mathcal{C}(\bx_i) \cC(\bg)\bR \bU^\top \bh(\bw)\right) \nonumber \\
		&=\frac{1}{p}\sum_{i=1}^p \psi_{\mu}\Big(\mathcal{C}(\bx_i) \big[\bI + \underbrace{ \sbra{\cC(\bg)\bR\bU^{-1}-\bI} }_{\bDelta} \big] \bh(\bw) \Big)  ,
	\end{align}
where $\bDelta := \cC(\bg)\bR\bU^{-1}-\bI=(\bU' - \bU) \bU^{-1}$, with $\bU' =  \cC(\bg)\bR$ and $\bU$ in \eqref{eq:def_U}. As was discussed earlier, as $\bR$ converges to $\mbra{\cC(\bg)^\top \cC(\bg)}^{-1/2}$ when $p$ increases, it is expected that $\bU'$ converges to $\bU$. Therefore, by bounding the size of $\bDelta$, we can control the deviation between $\phi_o(\bw)$ and $\phi(\bw)$. To this end, the rest of the proof is divided into the following two steps.

First, we show that the spectral norm of $\bDelta$ is bounded when the sample size is sufficiently large in Lemma~\ref{lemma:deviation op norm}, whose proof is given in Appendix \ref{appen: deviation op norm}.
	\begin{lemma}[Spectral norm of $\bDelta$]\label{lemma:deviation op norm}
	There exist some constants $C_1, c_f$, such that when $p\geq  \frac{C_1 \kappa^4 n \log^2 n \log p }{\theta^2}$, with probability at least $1-2n p^{-8}$,  
	\begin{equation}
		\norm{\bDelta} \leq c_f \kappa^4 \sqrt{\frac{\log^2 n \log p}{\theta^2 p}}.
	\end{equation}    
\end{lemma}

 Second, we show that the  deviation between the directional gradient and the Hessian of $\phi(\bw)$ and $\phi_o(\bw)$ can be bounded by the spectral norm of $\bDelta$, as shown in Lemma~\ref{lemma: deviation gd and he}. The proof can be found in Appendix~\ref{appen: deviation gd and he}.

\begin{lemma}[Deviation between $\phi_o(\bw)$ and $\phi(\bw)$]\label{lemma: deviation gd and he}
There exist some constants $c_g, c_h, C_1$, such that when $p\geq  \frac{C_1 \kappa^8 n \log^2 n \log p}{\theta^2}$, with probability at least $1-2p^{-8}$, we have
	\begin{subequations}
	\begin{align}
		\norm{\nabla \phi_o(\bw)-\nabla  \phi(\bw)}_2 \leq c_g \frac{n^{3/2}\log(np) }{\mu} \norm{\bDelta}, \quad  \bw \in  \mathcal{Q}_1, \label{eq:perturbation_gradient}\\
		\norm{\nabla^2  \phi_o(\bw)-\nabla^2 \phi(\bw )}\leq c_h  \frac{n^{5/2}\log^{3/2}(np) }{\mu^2}\norm{\bDelta},\quad  \bw \in  \mathcal{Q}_2. \label{eq:perturbation_hessian}
	\end{align}
	\end{subequations}
\end{lemma}

To complete the proof of Theorem~\ref{theo:general_geometry}, we need to show that the perturbations of the Hessian and the gradient between $\phi_o(\bw)$ and $\phi(\bw)$ are sufficiently small, which hold as long as the sample size is sufficiently large, in view of Lemma~\ref{lemma:deviation op norm}. Consequently, we can propagate the benign geometry of $\phi_o(\bw)$ in Theorem~\ref{thm:orthogonal_geometry} to $\phi(\bw)$. The complete proof is provided in Appendix~\ref{proof:theorem_general_geometry}.

\subsection{Proof Outline of Theorem~\ref{thm:general_GD}}\label{pipe:GD_orthogonal}

To capitalize on the benign geometry established in Theorem~\ref{theo:general_geometry}, one of the key arguments is to ensure that the iterates of MGD stay in the $2n$ subsets $\left \{\cS_{\xi_0}^{(i\pm)}, i\in [n] \right \}$ implicitly. This requires bounding properties of the directional gradient of $f(\bh)$ in \eqref{equ:equvalent_f2}, supplied in the following lemma whose proof can be found in Appendix~\ref{proof:thm:implicit_uniform_concentration_general}.
\begin{lemma} [Uniform concentration of the directional gradient]\label{thm:implicit_uniform_concentration_general}
Instate the assumptions of Theorem~\ref{theo:general_geometry}. There exist some constants $c_a, c_b, C_1$, such that with probability at least $1-3(np)^{-8} -2\exp\sbra{-c_an}$, 
	\begin{align}
		 \partial f(\bh)^\top \sbra{\frac{\be_k}{h_k}-\frac{\be_n}{h_n}} & \geq \frac{c_b \xi_0 \theta}{2},  \label{res: sins geometry}
	\end{align}
for $\bh\in \mathcal{H}_k =\left\{\bh: \bh\in \mathcal{S}_{\xi_0}^{(n+)}, h_k\neq 0, h_n^2/h_k^2<4\right\}$, and
	\begin{align}
		  \norm{\partial f(\bh)}_2 & \leq \norm{\nabla f(\bh)}_2\leq C_1 n \sqrt{\log(np)} \label{eq:bound_grad}
	\end{align}	 
for $\bh\in\mathbb{S}^{n-1}$.
\end{lemma}

The following lemma, proved in Appendix \ref{proof:implicit_stay}, then shows that the iterates of MGD will always stay in one of the subsets $\left \{\cS_{\xi_0}^{(i\pm) }, i\in [n] \right \}$ that it initializes in, as long as the sample complexity $p$ is large enough and the step size is properly chosen.
\begin{lemma} [Implicitly staying in the subsets]\label{thm:implicit_stay}
Instate the assumptions of Theorem~\ref{theo:general_geometry}. 
	For the MGD algorithm in Alg.~\ref{algo:1}, if the initialization satisfies that $\bh^{(0)}\in \cS_{\xi_0}^{(i \pm)}$ for any $i\in[n]$, and the step size satisfies
		$\eta\leq \frac{c}{ n^{3/2}\sqrt{\log(np)} }$
for some small enough constant $c$,	 then the iterates $\bh^{(k)}$, $k=1,2,\cdots$ will stay in $\cS_{\xi_0}^{(i\pm)}$.
\end{lemma}
 
The proof of Theorem~\ref{thm:general_GD} then follows by analyzing the convergence in two stages, corresponding to when the iterates lie in the region with large directional gradients, and the region with strong convexity, respectively. The details are given in Appendix \ref{proof:gd_for_orthogonal_case}.

 Till this point, the only left ingredient is to make sure a valid initialization can be obtained efficiently. By setting $\xi_0$ sufficiently  small, it is known from the following lemma \cite[Lemma 3]{gilboa2019efficient} that the union of $\left \{\cS_{\xi_0}^{(i \pm)}, i\in [n] \right \}$ is large enough to ensure a random initialization will land in it with a constant probability.
\begin{lemma}[{\cite[Lemma 3]{gilboa2019efficient}}]\label{theo:random_initialization} 
    When $\xi_0=\frac{1}{4 \log n}$, an initialization selected uniformly at random on the sphere lies in one of these $2n$ subsets $\left \{\cS_{\xi_0}^{(i\pm)}, i\in [n] \right \}$ with probability at least $1/2$. 
\end{lemma}

Finally, combining Lemma~\ref{theo:random_initialization} and Theorem~\ref{thm:general_GD}, by setting $\xi_0=1/(4\log n)$, we can guarantee to recover $\bg_{\inv}$ accurately up to global ambiguity with high probability,  as long as Alg.~\ref{algo:1} is initialized uniformly at random over the sphere with $O(\log n)$ times. This leads precisely to Corollary~\ref{cor:main_corollary}.

\section{Numerical Experiments}
\label{sec:numerical}

In this section, we examine the performance of the proposed approach with comparison to \cite{li2018global}, which is also based on MGD but using a different loss function $L(\bh)=-\frac{1}{4p}\sum_{i=1}^p \norm{\mathcal{C}(\by_i)\bR \bh}_4^4 $ over the sphere, on both synthetic and real data.

\subsection{MSBD with Synthetic Data}

We first compare the success rates of the proposed approach and the approach in \cite{li2018global}, following a similar simulation setup as in \cite{li2018global}. In each experiment, the sparse inputs are generated following $\mathrm{BG}(\theta)$, and $\cC(\bg)$ with specific $\kappa$ is synthesized by generating the DFT $\widehat{\bg}$ of $\bg$ which is random with the following rules: 1) the DFT $\widehat{\bg}$ is symmetric to ensure that $\bg$ is real, i.e., $\widehat{\bg}_{j} = \widehat{\bg}^*_{n+2-j}$, where $*$ denotes the conjugate operation; 2) the gains of $\widehat{\bg}$ follow a uniform distribution on $[1, \kappa]$, and the phases of $\widehat{\bg}$ follow a uniform distribution on $[0, 2\pi)$.  

In all experiments, we run MGD (cf. Alg.~\ref{algo:1}) for no more than $T=200$ iterations with a fixed step size of $\eta=0.1$ and apply backtracking line search for both methods for computational efficiency. For our formulation, we set $\mu=\min{(10 n^{-5/4},0.05)}$. For each parameter setting, we conduct $10$ Monte Carlo simulations to compute the success rate. Recall that the desired estimate $\widehat{\bg}_{\mathrm{inv}}$ is a signed shifted version of $\bg_{\inv}$, since $\cC(\bg)  \bg_{\mathrm{inv}} = \pm \be_j$ ($j\in [n]$). Therefore, to evaluate the accuracy of the output $\widehat{\bg}_{\mathrm{inv}}$, we compute $\cC(\bg) \widehat{\bg}_{\mathrm{inv}}$ using the ground truth $\bg$, and declare that the recovery is successful if
$ \|\cC(\bg) \widehat{\bg}_{\mathrm{inv}}\|_\infty / \| \cC(\bg) \widehat{\bg}_{\mathrm{inv}}\|_2 > 0.99$.

Fig.~\ref{fig:success_rate_nonorth} (a) and (d) show the success rates of the proposed approach and the approach in \cite{li2018global} with respect to $n$ and $p$, where $\theta=0.3$ and $\kappa =8$ are fixed. It can be seen that the proposed approach succeeds at a much smaller sample size, even when $p$ is smaller than $n$. This indicates room for improvements of our theory. Fig.~\ref{fig:success_rate_nonorth} (b) and (e) shows the success rates of the two approaches with respect to $\theta$ and $p$, where $n=64$ and $\kappa=8$ are fixed. The proposed approach continues to work well even at a relatively high value of $\theta$ up to around $0.5$. Finally, Fig.~\ref{fig:success_rate_nonorth} (c) and (f) shows the success rate of the two approaches with respect to $\kappa$ and $p$, where $n=64$ and $\theta=0.3$ are fixed. Again, the performance of the proposed approach is quite insensitive to the condition number $\kappa$ as long as the sample size $p$ is large enough. On the other end, the approach in \cite{li2018global} performs significantly worse than the proposed approach under the examined parameter settings.

\begin{figure}[t] 
    \centering
    \begin{tabular}{ccc}
    \hspace{-0.1in} \includegraphics[width=0.32\linewidth]{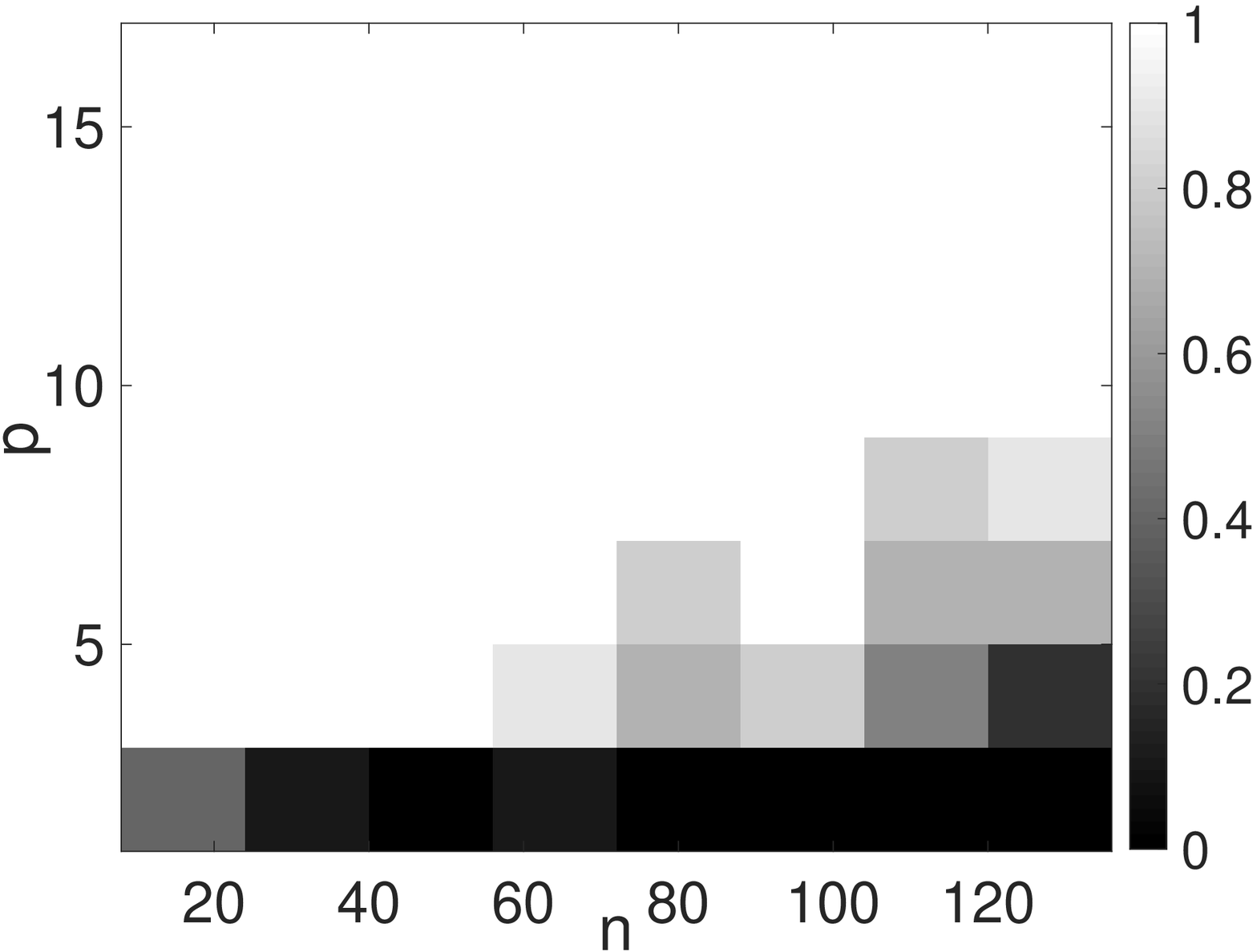} & 
    \hspace{-0.1in}    \includegraphics[width=0.32\linewidth]{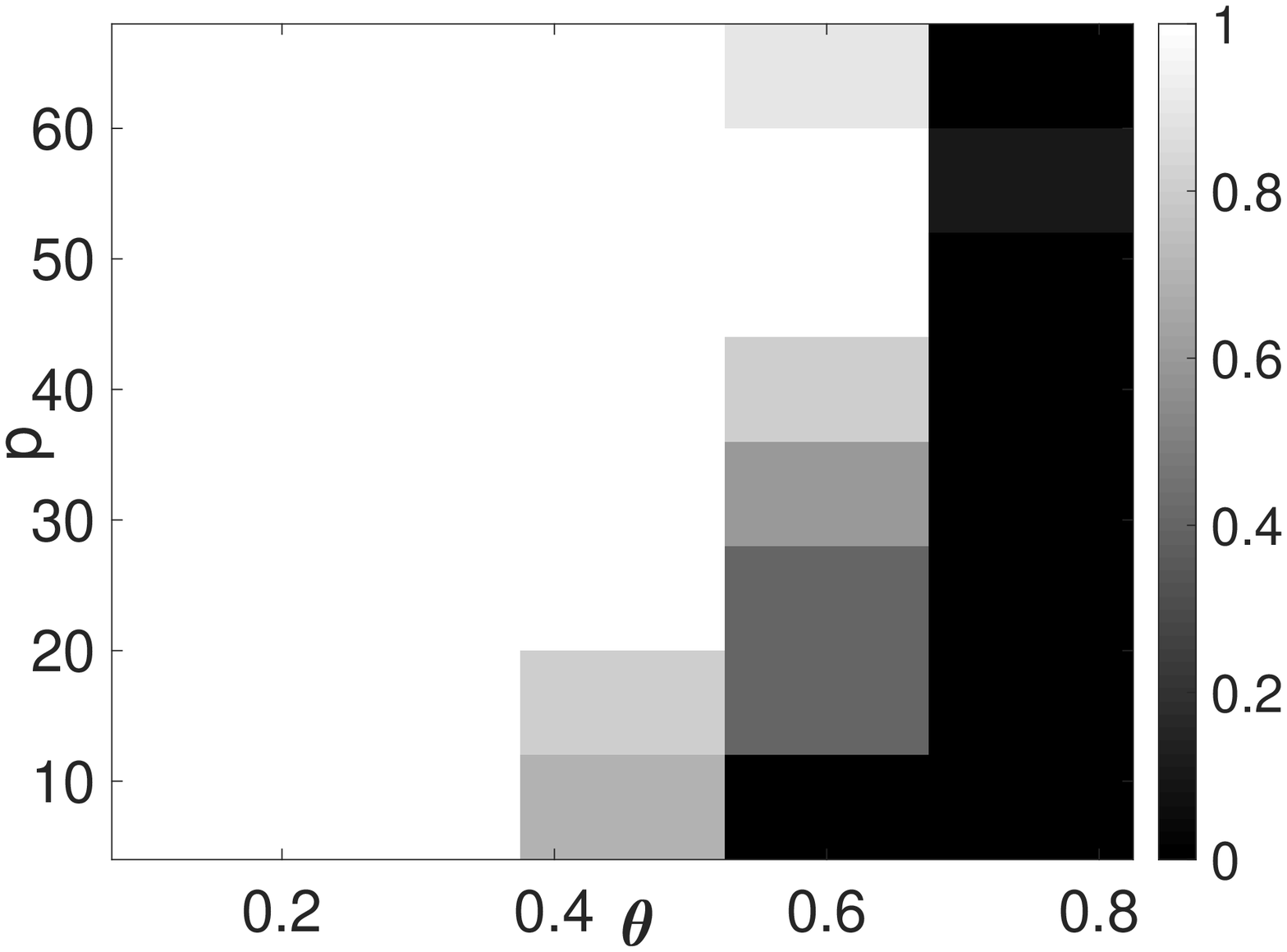} &
    \hspace{-0.1in} \includegraphics[width=0.32\linewidth]{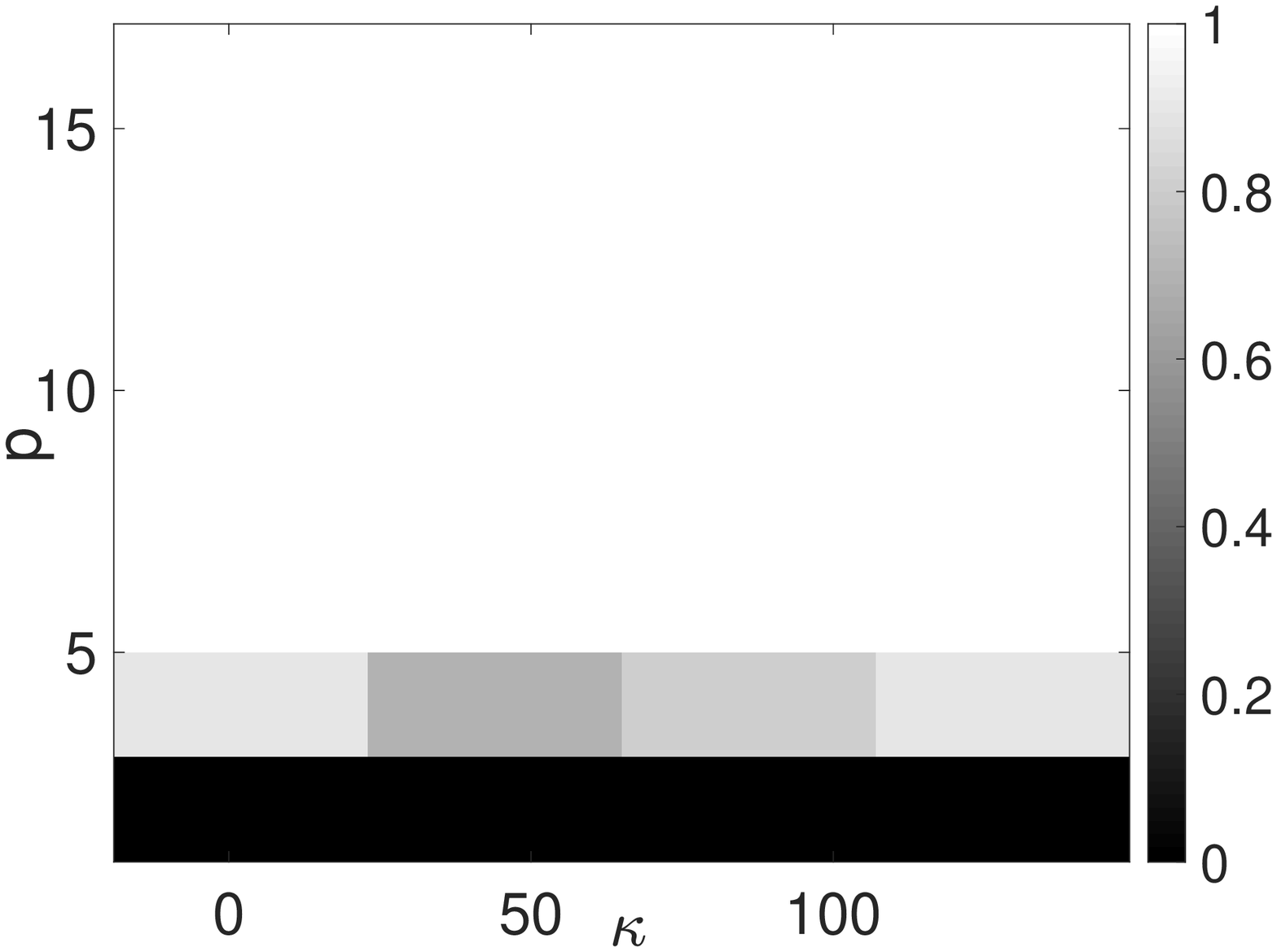}\\  
    \hspace{-0.2in} (a) ours $(n,p)$ & 
    \hspace{-0.2in}(b) ours $(\theta,p)$& 
    \hspace{-0.2in} (c) ours $(\kappa,p)$ \\ \vspace{0.01in} \\
    \hspace{-0.1in} \includegraphics[width=0.32\linewidth]{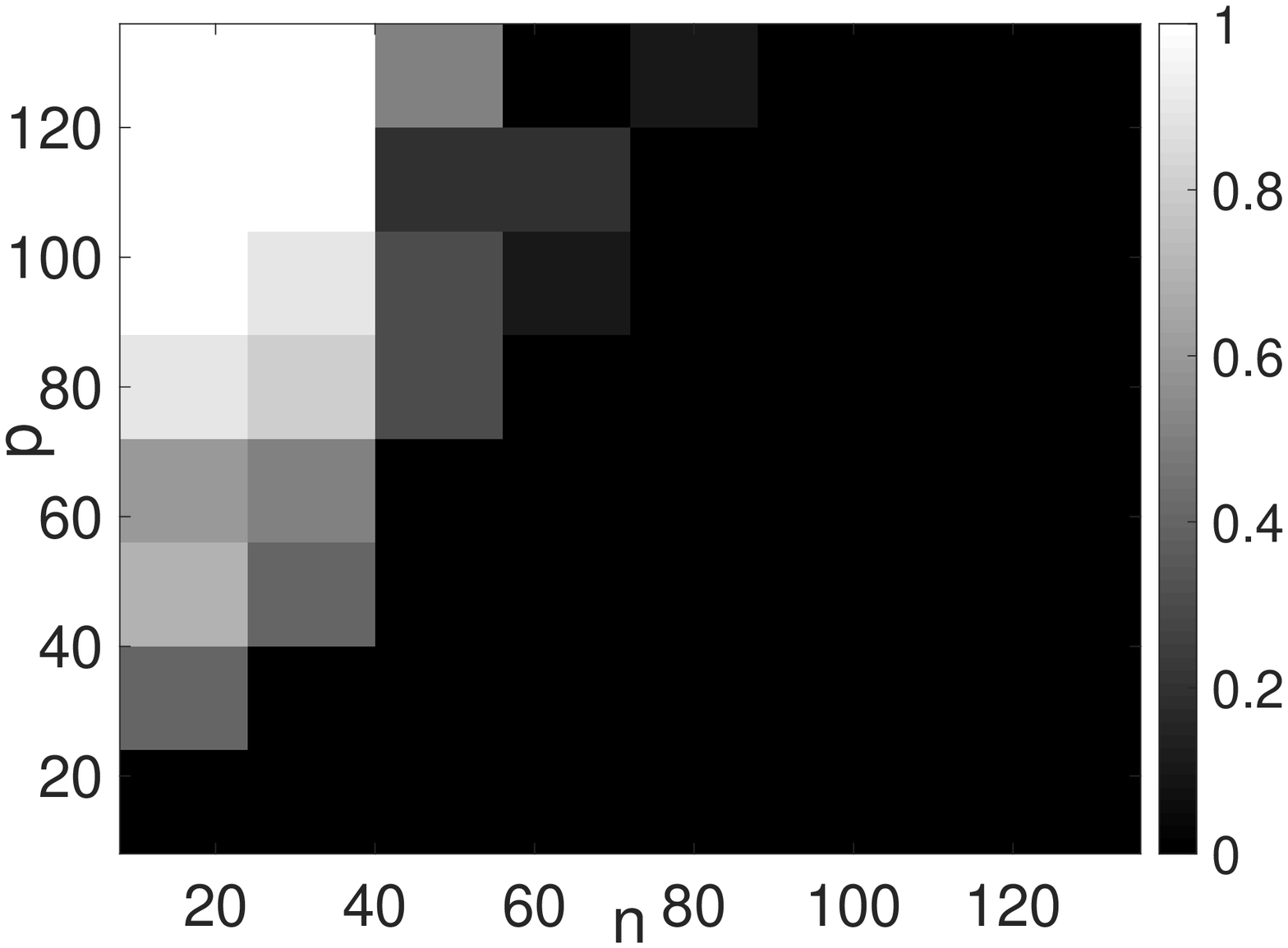} &
    \hspace{-0.1in} \includegraphics[width=0.32\linewidth]{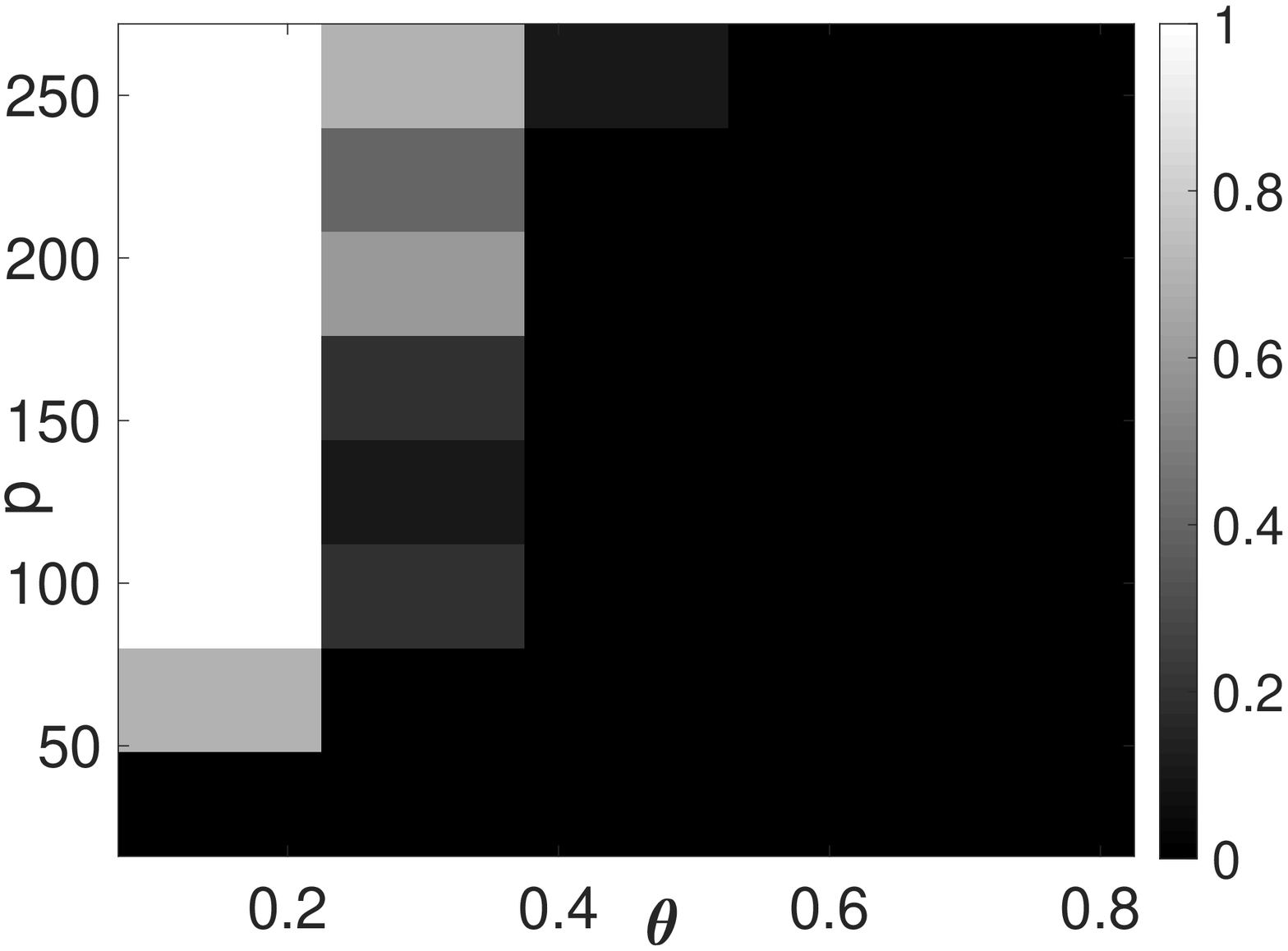} &
   \hspace{-0.1in} \includegraphics[width=0.32\linewidth]{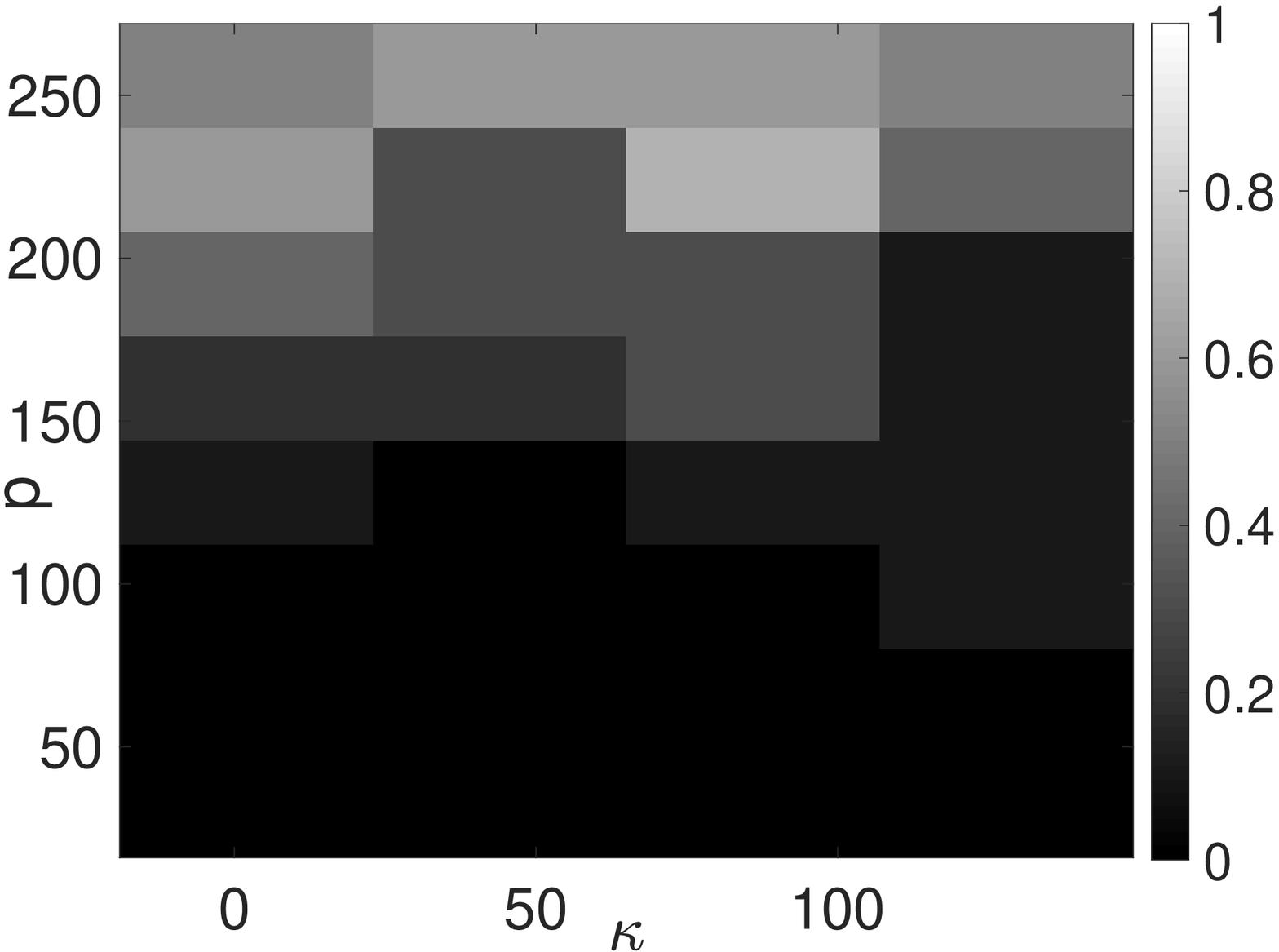} \\
    \hspace{-0.2in} (d)  \cite{li2018global} $(n,p)$ &
    \hspace{-0.2in} (e)   \cite{li2018global} $(\theta,p)$& 
    \hspace{-0.2in} (f)   \cite{li2018global} $(\kappa,p)$
    \end{tabular} 
    \caption{Success rates of the proposed approach (first row) and the approach in \cite{li2018global} (second row) under various parameter settings. }
    \label{fig:success_rate_nonorth}
\end{figure}

\subsection{Image Deconvolution and Deblurring} 

To further evaluate our method, we performance the task of blind image reconstruction and deblurring, and compare with \cite{li2018global}. Firstly, suppose multiple circulant convolutions $\{ \by_i\}_{i=1}^p$ (illustrated in Fig.~\ref{fig:real_data_experiments} (b)) of an unknown 2D image (illustrated in the ground truth figure in Fig. \ref{fig:real_data_experiments}, the Hamerschlag Hall on the campus of CMU) and multiple Bernoulli-Gaussian (BG) sparse inputs $\{ \bx_i \sim_{iid}\mathrm{BG}(\theta)\}_{i=1}^p$ (illustrated in Fig.~\ref{fig:real_data_experiments} (a)) are observed. Here, the size of the observations is $n=128\times 128$, $\theta=0.1$, and the number of observations $p=1000$, which is significantly smaller than $n$. 

We apply the proposed reconstruction method to each channel of the image, i.e. R, G, B, 
respectively using the corresponding channel of the observations $\{ \by_i\}_{i=1}^p$, and obtain the final recovery by summing up the recovered channels. For each channel, the recovered image is computed as $ \hat{\bg}=\mathcal{F}^{-1} \mbra{\mathcal{F}\sbra{\bR \hat{\bh}}^{\odot  -1}  },$ where $\hat{\bh}$ denotes the output of the algorithm, $\mathcal{F}$ is the 2D DFT operator, and $\bx^{\odot  -1}$ is the entry-wise inverse of a vector $\bx$. Fig.~\ref{fig:real_data_experiments} (c) and (d) show the final recovered images by our method and \cite{li2018global} (after aligning the shift and sign) respectively. It implies that the proposed approach obtains much better recovery than that in \cite{li2018global} when the sparse inputs $\{ \bx_i\}_{i=1}^p$ are with constant sparsity level $\theta$.

We next consider a more realistic setting and examine the performance of the proposed algorithm when the sparse coefficients do not obey the Bernoulli-Gaussian model. Using the same 2D image, we now generate multiple circulant convolutions $\{ \by_i\}_{i=1}^p$ (illustrated in Fig.~\ref{fig:real_data_experiments} (f)) using realistically-generated motion blur kernels\footnote{The nonlinear blur kernels are randomly produced using the tool in \url{https://github.com/LeviBorodenko/motionblur}. } (illustrated in Fig.~\ref{fig:real_data_experiments} (e)). Fig.~\ref{fig:real_data_experiments} (g) and (h) show the final recovered images by our method and \cite{li2018global} (after aligning the shift and sign) respectively. 
It can be seen that the proposed approach still obtains a robust recovery and removes the blurring effectively, while the recovery using \cite{li2018global} further degenerates possibly due to the model mismatch.

\begin{figure}[t]
    \centering
    \begin{tabular}{cccc}
    \centering
  &    \multicolumn{2}{c}{\includegraphics[width=0.25\linewidth]{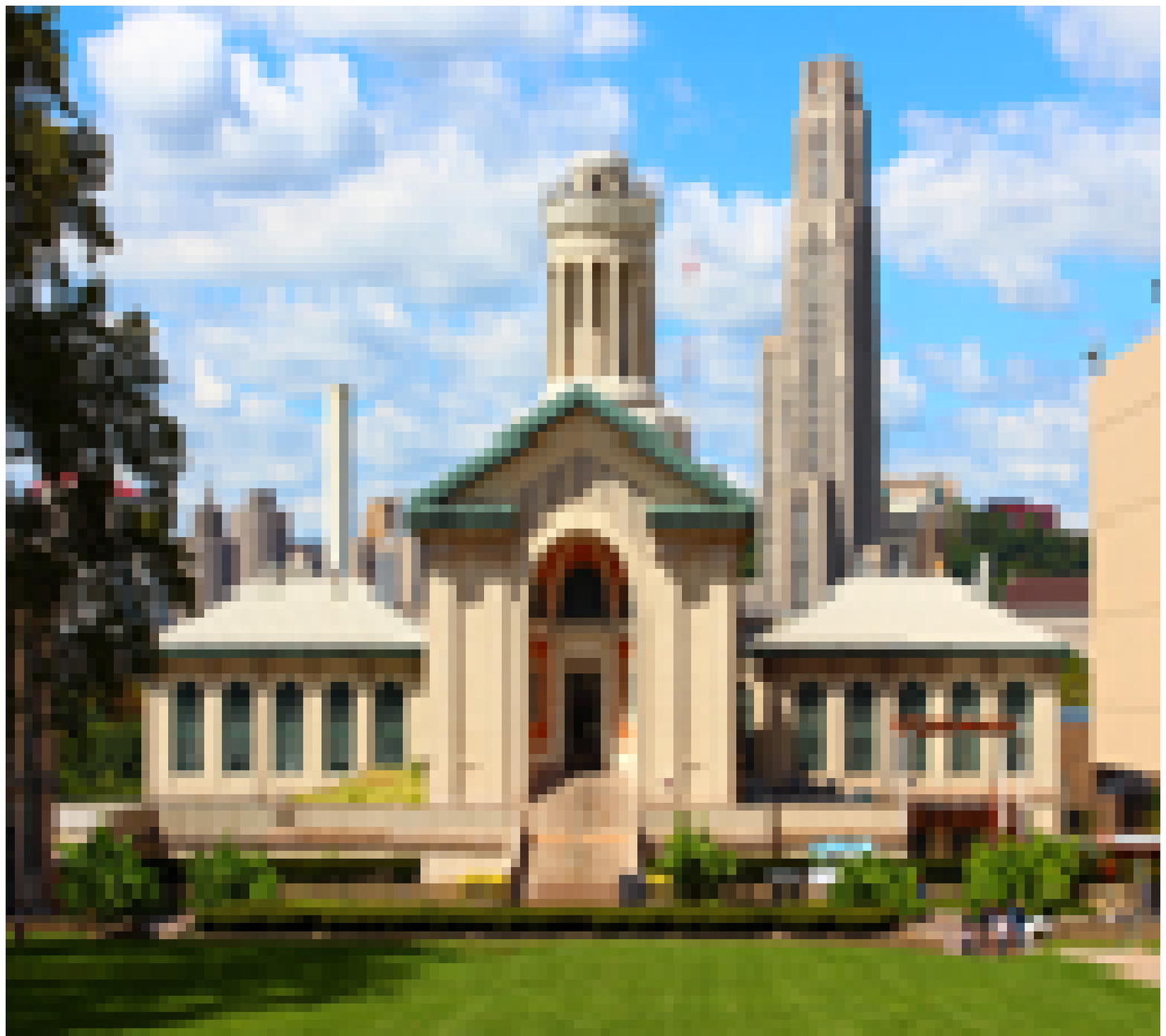}}  & \\
      \multicolumn{4}{c}{ground truth} \\ \vspace{0.01in} \\
    \hspace{-0.15in} \includegraphics[width=0.25\linewidth]{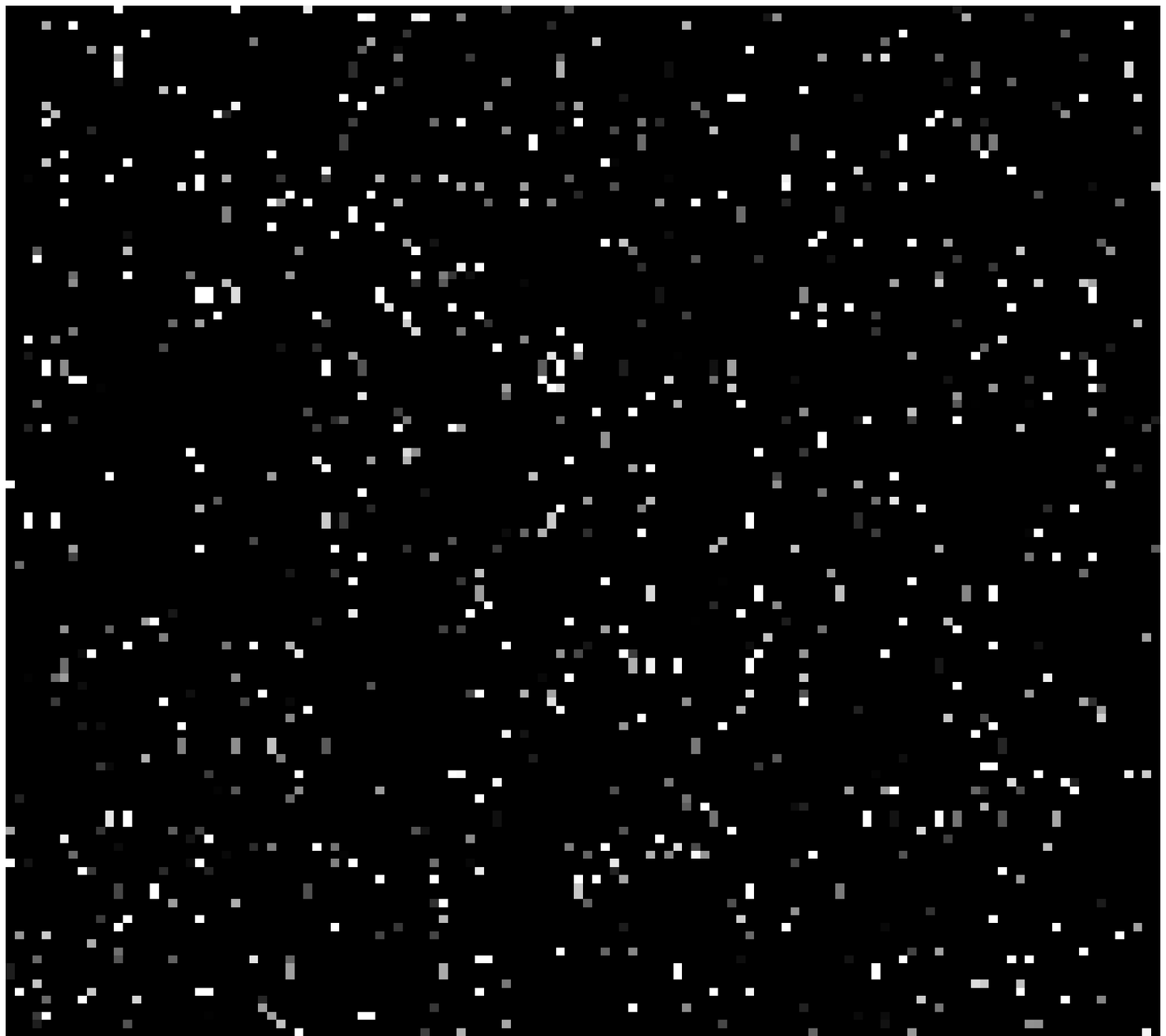} & 
    \hspace{-0.15in}  \includegraphics[width=0.25\linewidth]{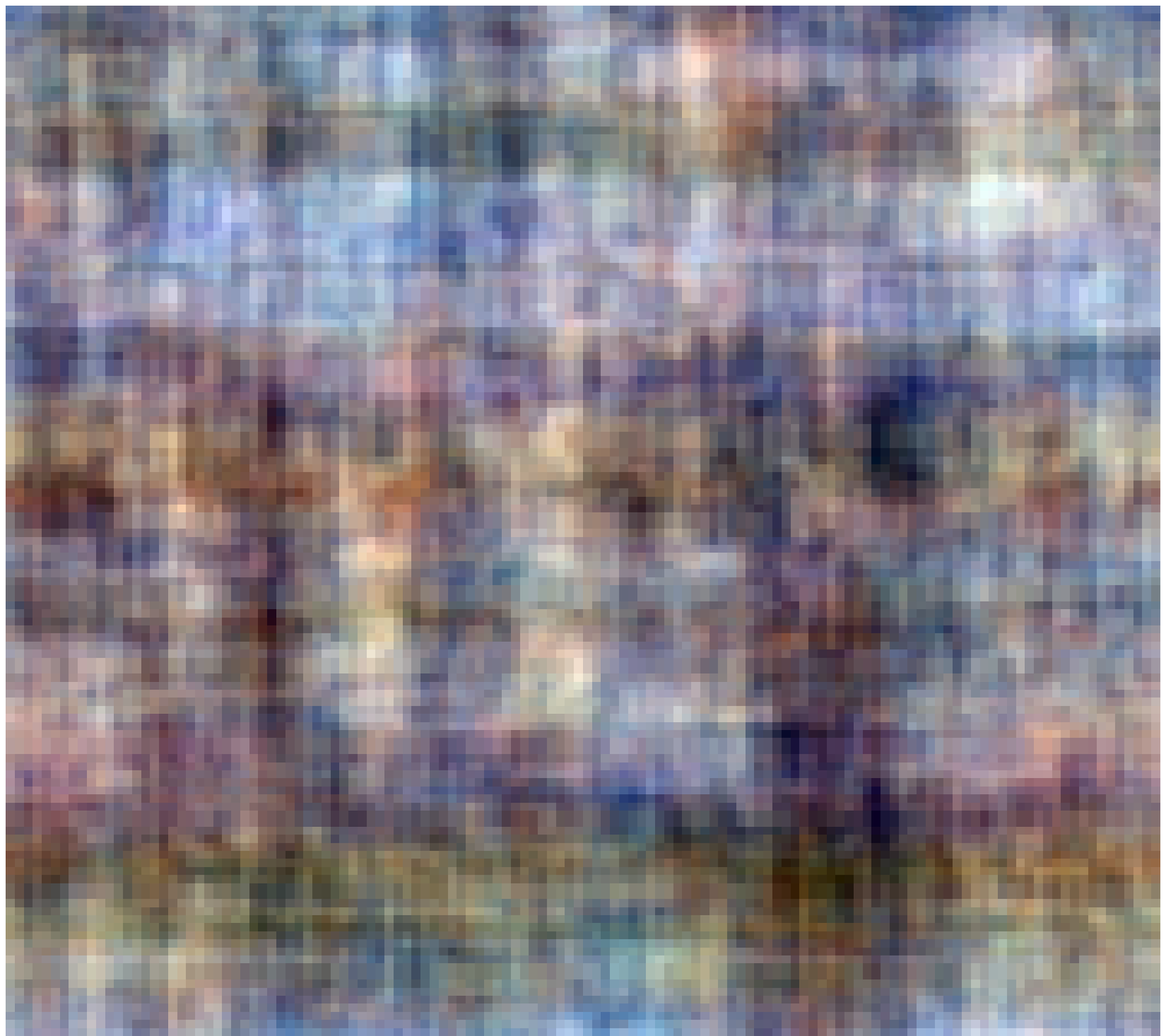}    &
    \hspace{-0.15in} \includegraphics[width=0.25\linewidth]{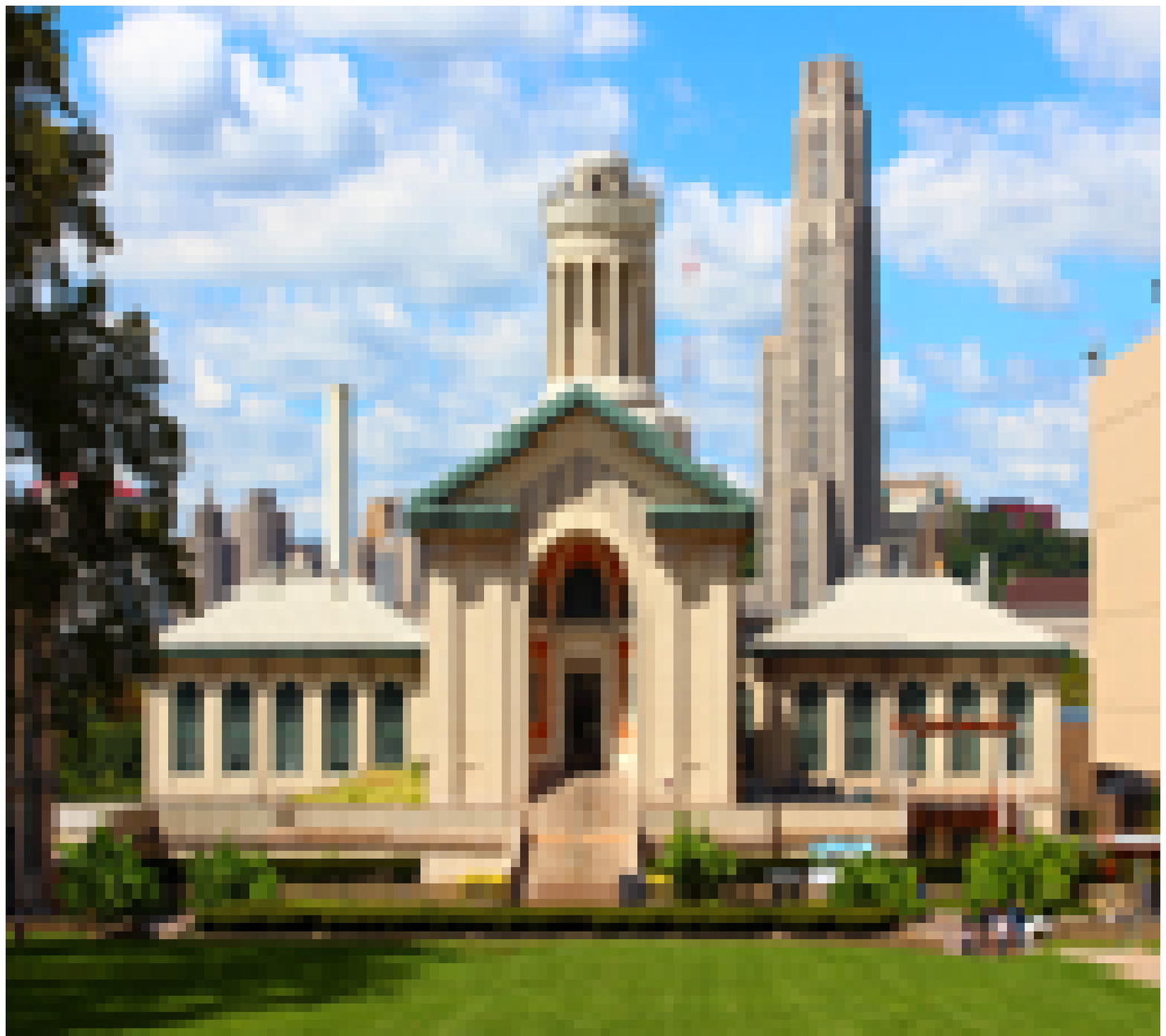} & 
    \hspace{-0.15in} \includegraphics[width=0.25\linewidth]{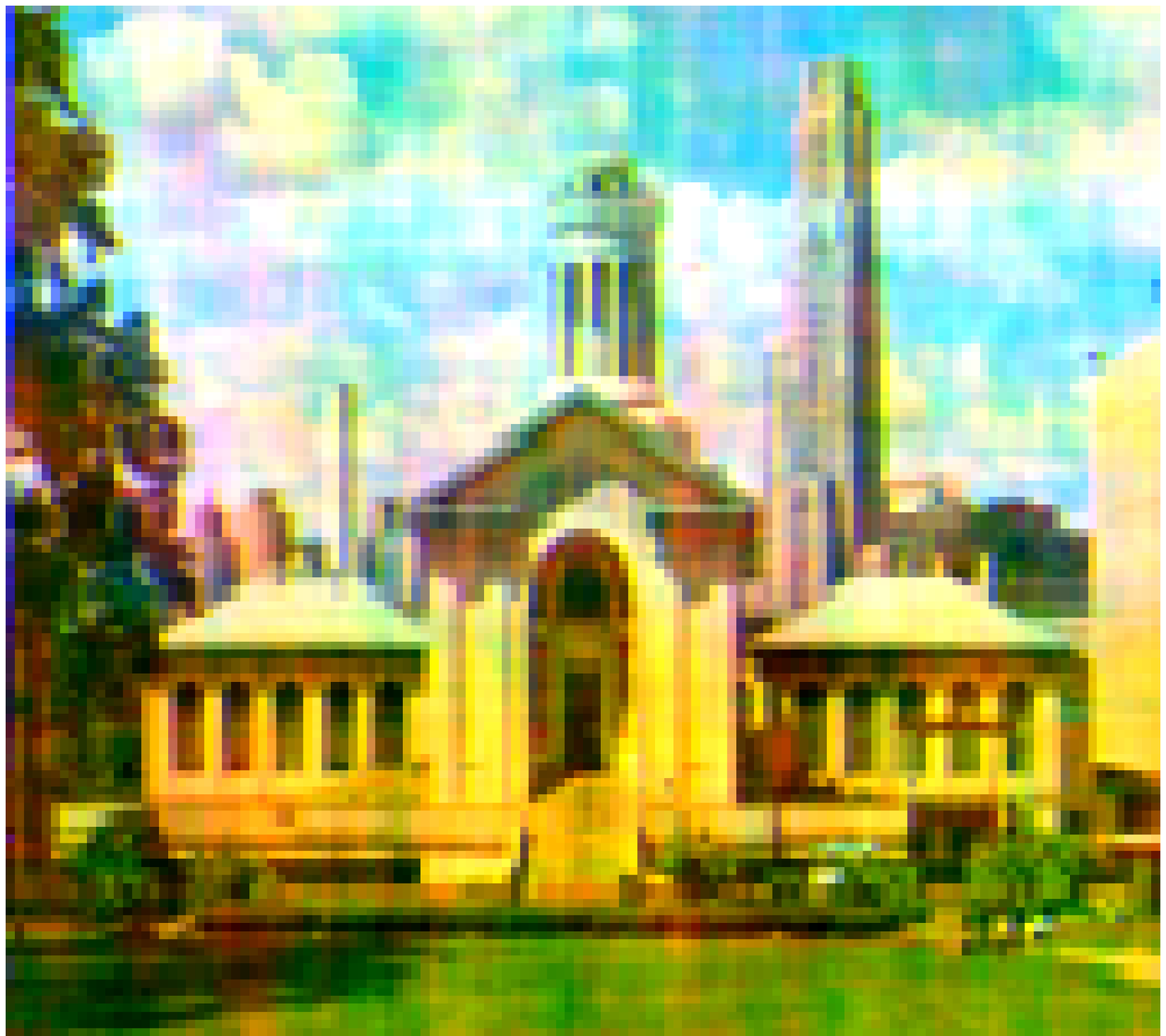} 
     \\ 
\hspace{-0.2in} (a) BG input   &
     \hspace{-0.2in} (b) observation  &
    \hspace{-0.2in} (c) recovery via ours &
    \hspace{-0.2in} (d) recovery via \cite{li2018global}    \\ \vspace{0.01in} \\
    \hspace{-0.15in} \includegraphics[width=0.25\linewidth]{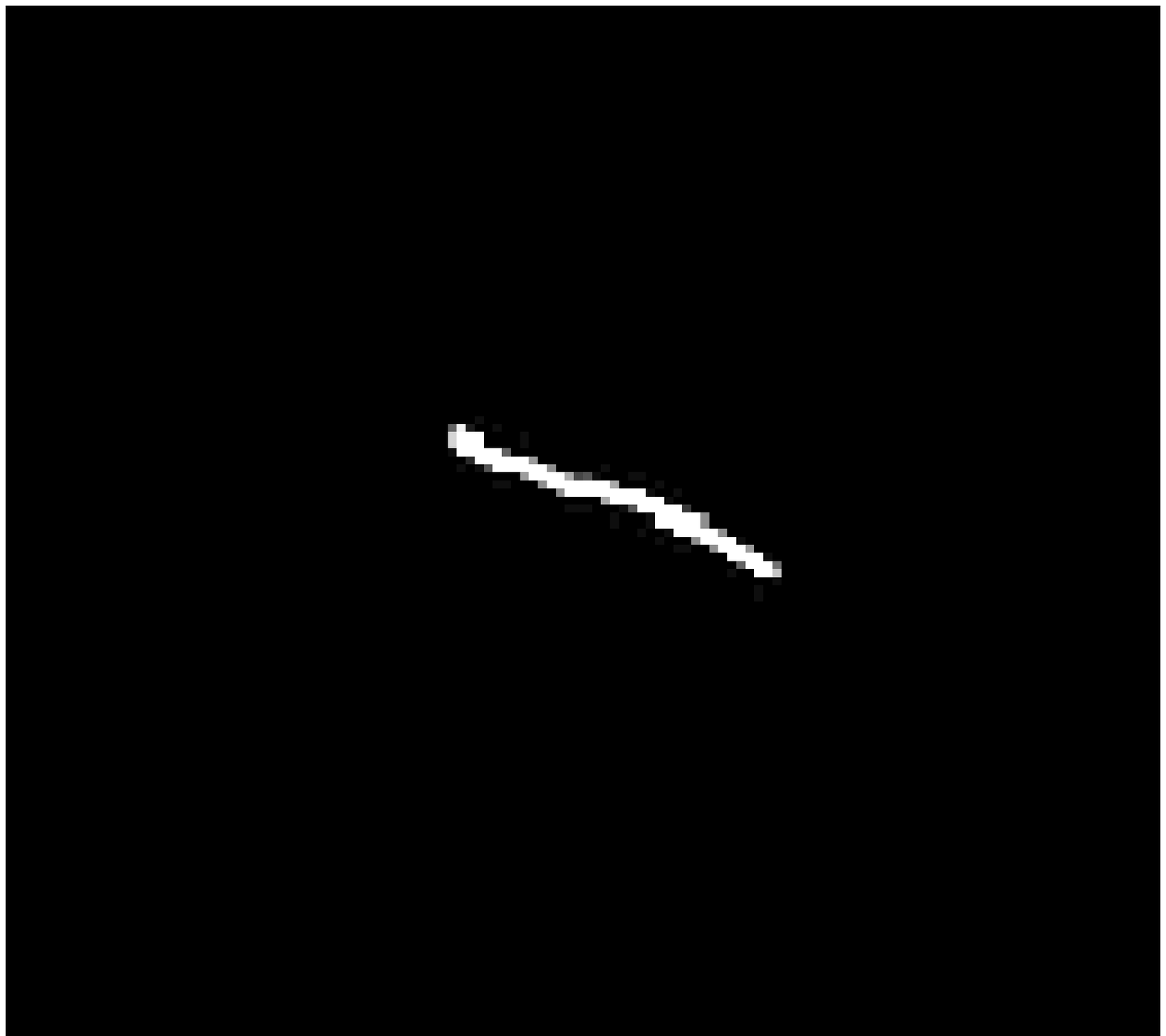} & 
    \hspace{-0.15in} \includegraphics[width=0.25\linewidth]{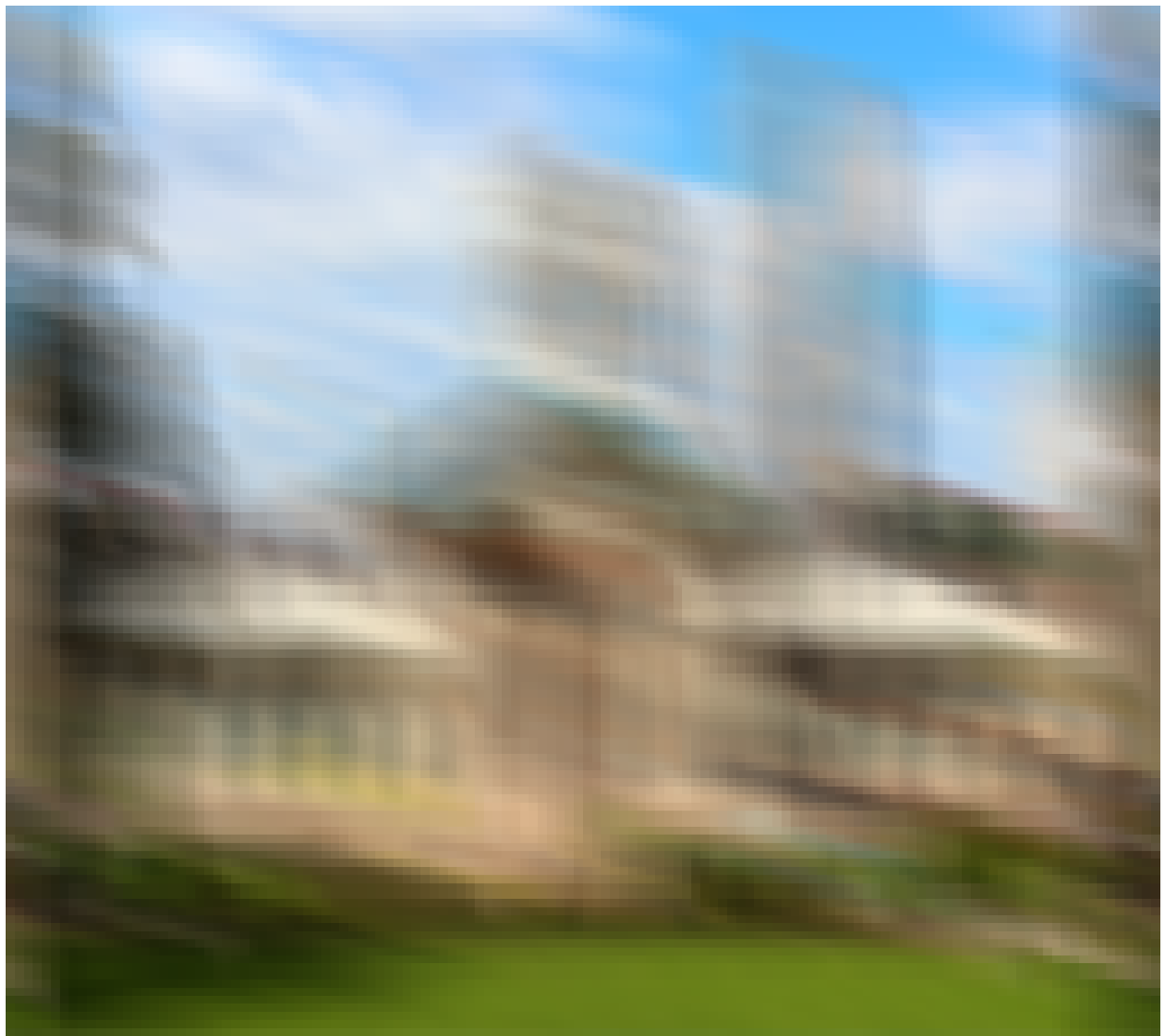}  &    
    \hspace{-0.15in} \includegraphics[width=0.25\linewidth]{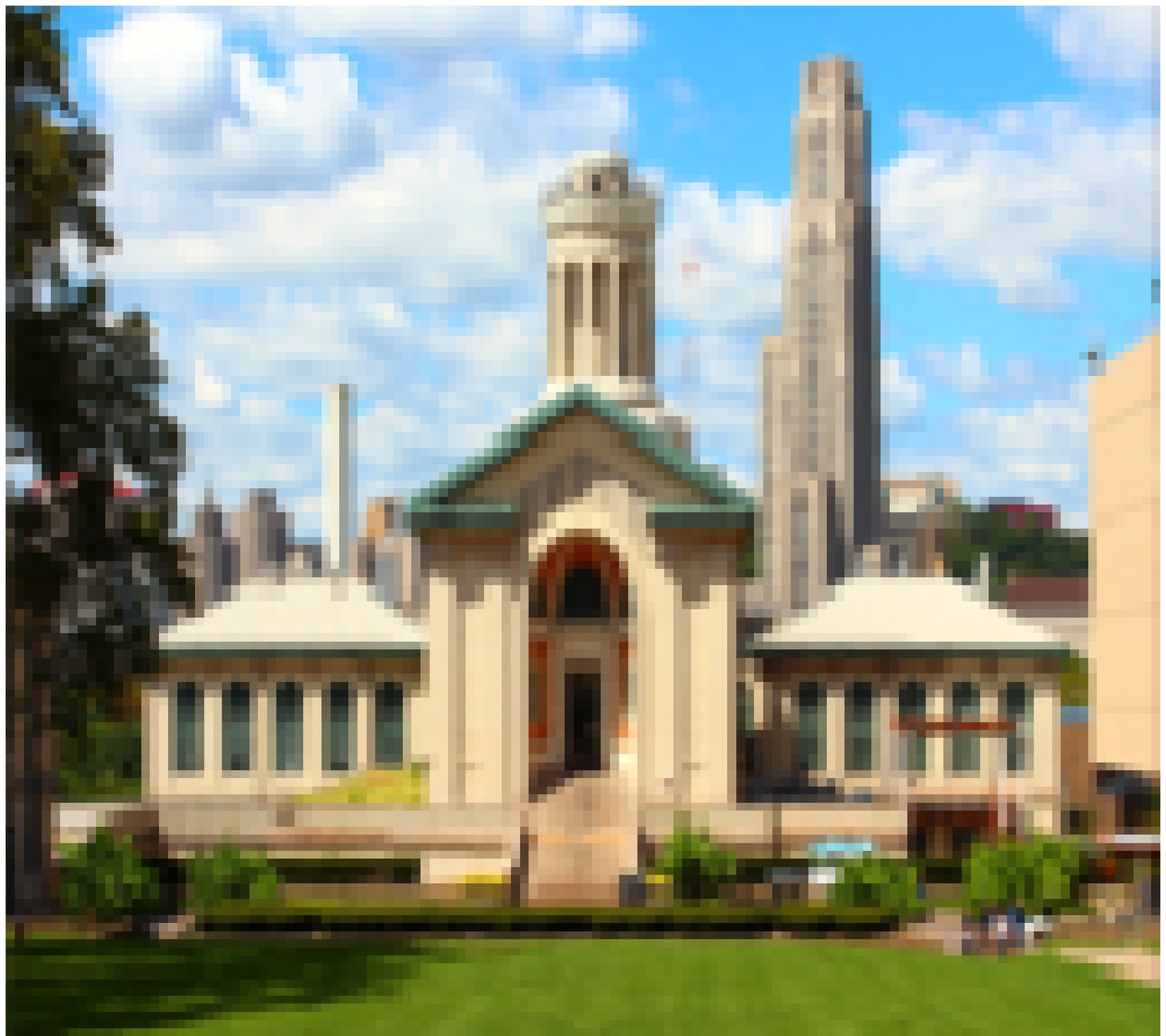} &  
        \hspace{-0.15in} \includegraphics[width=0.25\linewidth]{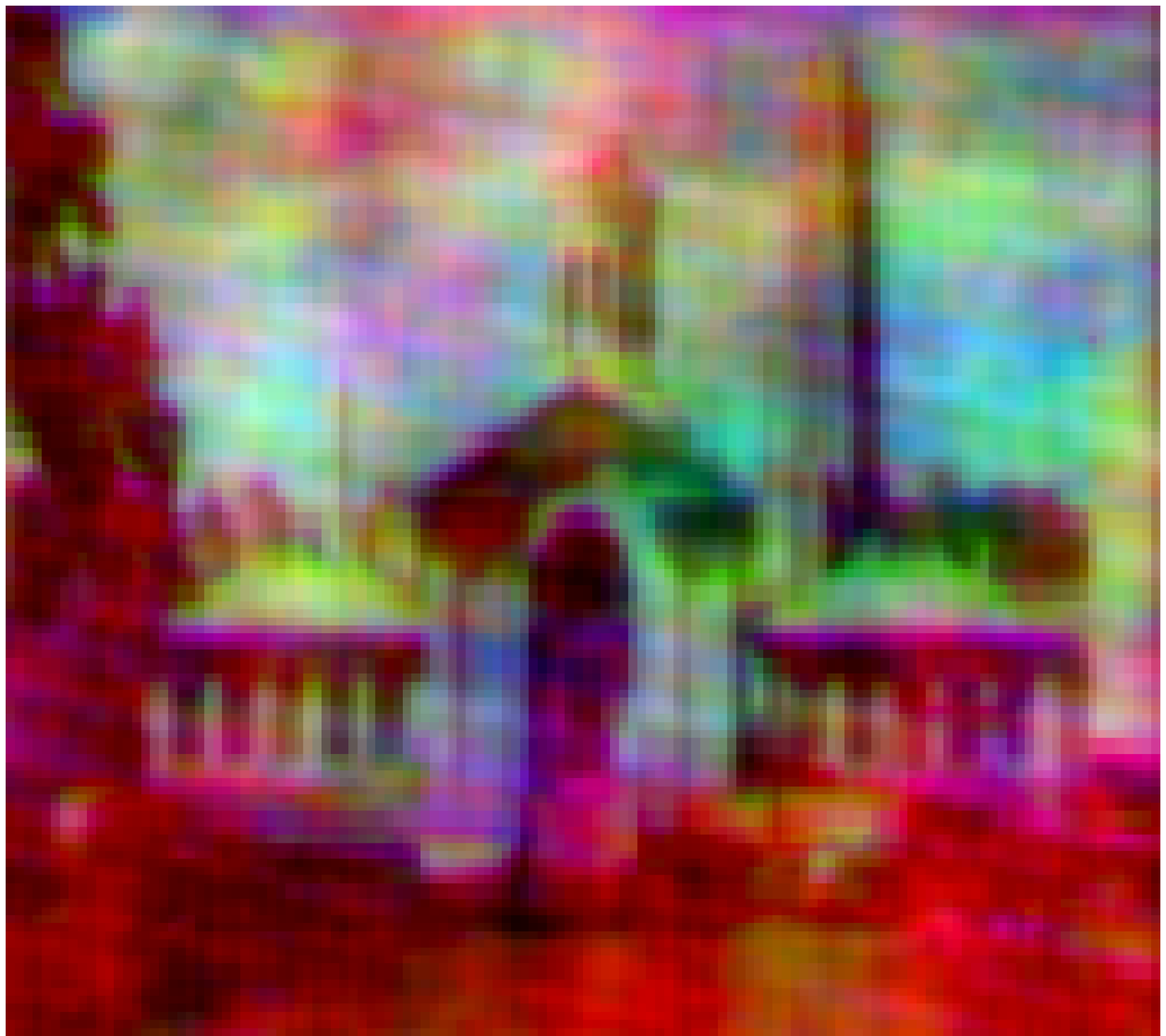}  \\
    \hspace{-0.2in} (e) blurring input  &
     \hspace{-0.2in} (f) observation  &
    \hspace{-0.2in} (g) recovery via ours &
    \hspace{-0.2in} (h) recovery via \cite{li2018global}
    \end{tabular} 
    \caption{Multi-channel sparse blind image deconvolution and deblurring. Top row: the ground truth image. Middle row: (a) the sparse input generated from the BG model; (b) the observation; (c) and (d): the recovery via our method and \cite{li2018global}. Bottom row: (e) the sparse input modeling motion blur; (f) the observation; (g) and (h): the recovery via our method and \cite{li2018global}. 
    }
    \label{fig:real_data_experiments}
\end{figure}

\section{Further Related Work}
\label{sec:literature} 

In this section, we discuss further related literature, emphasizing on algorithms with provable guarantees.

\paragraph{Provable blind deconvolution.} The problem of blind deconvolution with a single snapshot (or equivalently, channel) has been studied recently under different geometric priors such as sparsity and subspace assumptions on both the filter and the input, using both convex and nonconvex optimization formulations \cite{chi2016guaranteed,ling2015self,ahmed2014blind,li2019rapid,ma2018implicit,gribonval2012blind,bilen2014convex,li2018blind, choudhary2014sparse,chen2020convex,charisopoulos2019composite}. With the presence of multiple channels, one expects to identify the filter with fewer prior assumptions. Algorithms for multi-channel blind deconvolution include sparse spectral methods \cite{lee2018fast}, linear least squares \cite{ling2018self}, and nonconvex regularization \cite{xia2018identifiability}. A different model called ``sparse-and-short'' deconvolution is studied in \cite{kuo2019geometry, zhang2019structured}.

\paragraph{Provable dictionary learning.} Learning a sparsifying invertible transform from data has been extensively studied, e.g. in \cite{sun2018geometric, spielman2012exact,bai2019subgradient,gilboa2019efficient,luh2015dictionary,ravishankar2013learning,zhai2019complete}. In addition, provable algorithms for learning overcomplete dictionaries are also proposed in \cite{arora2014new,agarwal2014learning,barak2015dictionary,arora2015simple}. Our problem can be regarded as learning a convolutional invertible transform, where the proposed algorithm is inspired by the approach in \cite{bai2019subgradient} that characterizes a local region large enough for the success of gradient descent with random initializations. However, the approach in \cite{bai2019subgradient} is only applicable to an orthogonal dictionary, while we deal with a general invertible convolutional kernel. Compared to sample complexities required in learning complete dictionaries \cite{sun2017complete}, our result demonstrates the benefit of exploiting convolutional structures in further reducing the sample complexity.

\paragraph{Provable nonconvex statistical estimation.} Our work belongs to the recent line of activities of designing provable nonconvex procedures for high-dimensional statistical estimation, see \cite{chi2018nonconvex,chen2018harnessing,zhang2020symmetry} for recent overviews. Our approach interpolates between two popular approaches, namely, global analyses of optimization landscape (e.g. \cite{sun2018geometric,sun2017complete,ling2019landscape,ge2016matrix,park2017non,ge2017no,bhojanapalli2016global,zhu2017global,li2018global}) that are independent of algorithmic choices, and local analyses with careful initializations and local updates (e.g. \cite{candes2015phase,cai2016optimal,zhang2016provable,zhang2017reshaped,li2020non,chen2015solving,ma2018implicit,netrapalli2014non,tu2015low,tong2020accelerating,li2019nonconvex}).

\section{Discussions} \label{sec:discussions}

This paper proposes a novel nonconvex approach for multi-channel sparse blind deconvolution based on manifold gradient descent with random initializations. Under a Bernoulli-Gaussian model for sparse inputs, we demonstrate that the proposed approach succeeds as long as the sample complexity satisfies $p=O(n^{4.5}\mbox{polylog} p)$, a result significantly improving prior art in \cite{li2018global}. We conclude the paper by some discussions on future directions.
\begin{itemize} 

\item {\em Improve sample complexity.} Our numerical experiments indicate that there is still room to further improve the sample complexity of the proposed algorithm, which may require a more careful analysis of the trajectory of the gradient descent iterates, as done in \cite{chen2018gradient}.

\item {\em Efficient exploitation of negative curvature.} We remark that it is possible to characterize the global geometry over the sphere, where the remaining region contains saddle points with negative curvatures. However, a direct analysis leads to an increase of sample complexity which is undesirable and therefore not pursued in this paper. On the other end, it seems random initialization {\em without restarts} also works well in practice, which warrants further investigation.

\item {\em Super-resolution blind deconvolution}. The model studied in this paper assumes the same temporal resolution of the input and the output, while in practice the sparse activations of the input can occur at a much higher resolution. This lead to the consideration of a refined model, where the observation is given as $\by = \bm{F}_{n\times n}^{\mathsf{H}} \mbox{diag}(\hat{\bg}) \bm{F}_{n\times D} \bx$, where $\bm{F}_{n\times D}$ is the oversampled DFT matrix of size $n\times D$, $D\geq n$. The approach taken in this paper cannot be applied anymore, and new formulations are needed to address this problem, see \cite{da2019self} for a related problem.

\item {\em Convolutional dictionary learning.} Our work can be regarded as a first step towards developing sample-efficient algorithms for 
 convolutional dictionary learning \cite{kavukcuoglu2010learning} with performance guarantees. An interesting model for future investigation is when multiple filters are present, and the observation is modeled as $\by = \sum_{\ell=1}^L \mathcal{C}(\bg_\ell)\bx_\ell$, with $L$ the number of filters. The goal is thus to simultaneously learn multiple filters $\{\bg_\ell \}_{\ell=1}^L$ from a number of observations in the form of $\by$. See \cite{qu2019geometric} for some recent developments.

\end{itemize}

\section*{Acknowledgment}
The authors thank Wenhao Ding of Carnegie Mellon University and Rong Fu of Tsinghua University for useful discussions about the conducted experiments.

\bibliography{bibfileSparseBD,bibfileNonconvex_TIT}
\bibliographystyle{IEEEtran}

\appendix

\section{Prerequisites}

For convenience, let $\bX\in \mathbb{R}^{n\times p}$ be the inputs $ \bX=[\bx_1,\bx_2,\cdots,\bx_p]$. Denote the first-order derivative of $\psi_\mu(x)$ as $\psi_\mu^{\prime}(x)=\tanh(x/\mu)$ and the second-order derivative as $\psi_\mu^{\prime\prime}(x)=\sbra{1-\tanh^2\sbra{x /\mu}}/\mu$. The gradient of $ \psi_{\mu}(\mathcal{C}(\bx_i) \bh)$ with respect to $\bh$ can be written as
\begin{align}\label{eq:gradient_phi_h}
\nabla_{\bh} \psi_{\mu}(\mathcal{C}(\bx_i) \bh) & =  \mathcal{C}(\bx_i)^{\top} \tanh\sbra{\frac{\mathcal{C}(\bx_i)\bh}{\mu}} ,
\end{align}
where, with slight abuse of notation, we allow $\tanh(\cdot)$ to take a vector-value in an entry-wise manner.

Recalling the reparameterization $\bh = \bh(\bw)=\sbra{\bw,\sqrt{1-\norm{\bw}_2^2}}$, let $\bJ_{\bh}(\bw)$ be the Jacobian matrix of $\bh(\bw)$, i.e. 
\begin{equation}\label{eq:jacobian}
\bJ_{\bh}(\bw)= \mbra{\bI, -\frac{\bw}{h_n(\bw)}} \in \mathbb{R}^{(n-1)\times n},
\end{equation}  
where $h_n(\bw) = \sqrt{ 1- \|\bw\|_2^2}$ is the last entry of $\bh(\bw)$. 
By the chain rule, the gradient of $ \psi_{\mu}\sbra{\cC(\bx_i)\bh(\bw)}$ with respect to $\bw$ is given as
\begin{align}\label{equ:gradient_psi}
       \nabla_{\bw}  \psi_{\mu}\sbra{\cC(\bx_i)\bh(\bw)} 
       & = \bJ_{\bh}(\bw)  \nabla_{\bh}  \psi_{\mu}\sbra{\cC(\bx_i)\bh}  =  \bJ_{\bh}(\bw) \mathcal{C}(\bx_i)^{\top} \tanh\sbra{\frac{\mathcal{C}(\bx_i)\bh(\bw)}{\mu}}  .
\end{align}
Moreover, the Hessian of $ \psi_{\mu}\sbra{\cC(\bx_i)\bh(\bw)}$ is given as
\begin{align}\label{equ:hessian_psi} 
        \nabla_{\bw}^2  \psi_{\mu}\sbra{\cC(\bx_i)\bh(\bw)} &= \frac{1}{\mu} \bJ_{\bh}(\bw)   \mathcal{C}(\bx_i)^{\top}  \mbra{\bI -\mbox{diag}\left(\tanh^2\sbra{\frac{\mathcal{C}(\bx_i)\bh(\bw)}{\mu}}\right) }  \mathcal{C}(\bx_i)  \bJ_{\bh}(\bw)^{\top} \nonumber \\ 
&\quad\quad\quad\quad  -  \frac{1}{h_n}\mathcal{S}_{n-1}(\bx_i)^{\top} \tanh\sbra{\frac{\mathcal{C}(\bx_i)\bh(\bw)}{\mu}}  \bJ_{\bh}(\bw) \bJ_{\bh}(\bw)^{\top} . 
   \end{align}

\subsection{Useful Concentration Inequalities}

We first introduce some notation and properties of sub-Gaussian variables. A random variable $X$ is called sub-Gaussian if its sub-Gaussian norm satisfies $\norm{X}_{\psi_2}<\infty$ \cite{vershynin2018high}. Similarly, we have $\norm{\bx}_{\psi_2}<\infty$ for a sub-Gaussian random vector $\bx$ \cite[Definition 3.4.1]{vershynin2018high}. For Bernoulli-Gaussian random variables~/~vectors, we have the following two facts, which imply that they are also sub-Gaussian.

\begin{fact}[{\cite[Lemma F.1]{bai2019subgradient}}] \label{fact:BG to SG} 
A random variable $X\in \mathrm{BG}(\theta)$ is sub-Gaussian, i.e. $\norm{X}_{\psi_2}\leq C_a$ for some constant $C_a$. Similarly, for a random vector $\bx\sim_{iid} \mathrm{BG}(\theta)$ and any deterministic vector $\bv\in\mathbb{R}^n$, we have $\norm{\bv^\top \bx}_{\psi_2}\leq C_b \norm{\bv}_2$ for some constant $C_b$.  
\end{fact}

\begin{fact}[{\cite[Lemma 21]{sun2017complete}}]\label{fact:BG to gaussian}
Assume $\bx$, $\by\in\mathbb{R}^n$ satisfy $\bx\sim_{iid} \mathrm{BG}(\theta)$ and $\by\sim_{iid} \mathcal{N}(\boldsymbol{0},\bI)$. Then for any deterministic vector $\bu\in\mathbb{R}^n$, we have $\mathbb{E}(\abs{\bu^\top\bx}^m)\leq \mathbb{E}(\abs{\bu^\top\by}^m)$, and $ \mathbb{E}(\norm{\bx}_2^m) \leq \mathbb{E}(\norm{\by}_2^m)$  
for all integers $m\geq 1.$ 
\end{fact}
The second fact allows us to bound the moments of a Bernoulli-Gaussian vector via the moments of a Gaussian vector, which are given below.
\begin{fact}[{\cite[Lemma 35]{sun2017complete}}] \label{fact: Moments bound for chi} 
Let $\by\in\mathbb{R}^n$ be $\by\sim_{iid}  \mathcal{N}(\boldsymbol{0},\bI)$, we have for any $m\geq 1$,
$      \mathbb{E} \sbra{\|\by\|_2^m}\leq m! n^{m/2}$.
\end{fact}

In addition, let us list a few more useful facts about sub-Gaussian random variables.
\begin{fact}[{\cite[Lemma 2.6.8]{vershynin2018high}}] \label{fact:centering}    If $X$ is sub-Gaussian, then $X-\mathbb{E}X$ is also sub-Gaussian with $\norm{X-\mathbb{E}X}_{\psi_2}\leq C\norm{X}_{\psi_2}$ for some constant $C$. 
\end{fact}

\begin{fact}[{\cite[Proposition 2.6.1]{vershynin2018high}}]  \label{fact:sum of BG}
If $X_1, X_2,\cdots, X_n$ are zero-mean independent sub-Gaussian random variables, then there exists some constant $C$ such that $\norm{\sum_{i=1}^n X_i}_{\psi_2}^2\leq C \sum_{i=1}^n \norm{X_i}_{\psi_2}^2.$  
\end{fact}

\begin{fact}[{\cite[Eq. (2.14-2.15)]{vershynin2018high}}] \label{fact: tail bound of BG}
If $X$ is sub-Gaussian, it satisfies the following bounds:
\begin{equation*} 
        \mbP{\abs{X}\geq t}\leq 2 \exp\sbra{- c t^2/\snorm{X}^2} \quad \forall t\geq 0, \quad
        \sbra{\mathbb{E}{\abs{X}^m}}^{1/m}\leq C\sqrt{m}\snorm{X}\quad \forall m\geq 1,
\end{equation*}
where $c$, $C$ are some universal constants.
\end{fact}

Combining standard tail bounds with the union bound, we have the following facts.
	\begin{fact} \label{fact:norm of BG}
    Let $ \{\bx_i \}_{i=1}^p \in\mathbb{R}^n$ be independent sub-Gaussian vectors with $\|\bx_i \|_{\psi_2}\leq B$ for some constant $B$. Then there exists some universal constant $C$ such that with probability at least $1- p^{-8}$, we have
    \begin{equation*} 
        \max_{i\in{[p]}} \norm{\bx_i}_2 \leq CB\sqrt{n \log p}.
    \end{equation*}
\end{fact}

\begin{fact}\label{fact:maximum of X}
    Let $\bX\in \mathbb{R}^{n\times p}$ be $\bX\sim_{iid} \mathrm{BG}(\theta)$, where $\theta\in (0,1/2)$. With probability at least $1-\theta (np)^{-7}$, we have
    \begin{equation*}
        \norm{\bX}_\infty = \max_{i,j} |X_{ij}| \leq 4\sqrt{\log(np)}.
    \end{equation*}
\end{fact}

Finally, let us record the useful Bernstein's inequality for random vectors and matrices, which does not require the quantities of interest to be centered.  This is a direct consequence of Fact~\ref{fact:centering} on centering and \cite[Theorem 6.2]{tropp2012user}. 
\begin{lemma}[Moment-controlled Bernstein's inequality]\label{lemma:matrix_bernstein}
    Let  $\{\bX_k \in\mathbb{R}^{n\times n}\}_{k=1}^{p}$ be a set of independent random matrices. Assume there exist $\sigma, R$ such that for all $m\geq 2$, $\mathbb{E} \sbra{\|\bX_k\|^m} \leq \frac{m!}{2} \sigma^2 R^{m-2}$. Denote $\bS=\frac{1}{p}\sum_{k=1}^p \bX_k$, then we have for any $t>0$,
 $$\mbP{\norm{\bS-\mathbb{E}\bS}> t}\leq 2n \exp\sbra{\frac{-pt^2}{2\sigma^2+2 Rt}}. $$ 
 Let $\{\bx_k \in\mathbb{R}^{n}\}_{k=1}^p$ be a set of independent random vectors. Assume there exist $\sigma$, $R$ such that $ \mathbb{E} \sbra{\norm{\bx_k}_2^m} \leq \frac{m!}{2} \sigma^2 R^{m-2}$. Denote $\bs=\frac{1}{p}\sum_{k=1}^p \bx_k$, then we have for any $t>0$,
$$    \mbP{\norm{\bs-\mathbb{E}\bs}_2> t}\leq (n+1) \exp\sbra{\frac{-pt^2}{2\sigma^2+2 Rt}}. $$ 
 \end{lemma}

\subsection{Technical Lemmas}

In this section, we provide some technical lemmas that are used throughout the proof. We start with some useful properties about the $\tanh(\cdot)$ function since it appears frequently in our derivation.

\begin{lemma}\label{lemma:loss_expectation}
    Let $X\sim \gd{\sigma_x^2}$, $Y\sim \gd{\sigma_y^2}$, then we have
    \begin{align*}
            \mbE{\tanh(aX)X}&=a \sigma_x^2 \mbE{1-\tanh^2 (aX)} , \\
            \mbE{\tanh\sbra{a(X+Y)}X}&=a \sigma_x^2 \mbE{1-\tanh^2\sbra{a(X+Y)}}. 
    \end{align*}
\end{lemma}
\begin{proof}
    Using integration by parts, we have
    \begin{align*} 
      &  \mbE{\tanh(aX)X}= \frac{1}{\sqrt{2\pi} \sigma_x} \int_{-\infty}^{\infty} \tanh(aX)X \exp\sbra{\frac{-X^2}{2\sigma_x^2}} dX\\
        &=- \frac{2\sigma_x^2}{\sqrt{2\pi} \sigma_x}  \tanh(aX) \exp\sbra{\frac{-X^2}{2\sigma_x^2}}\Big|_{0}^{\infty}+ \frac{1}{\sqrt{2\pi} \sigma_x} \int_{-\infty}^{\infty} a\sigma_x^2 \sbra{1-\tanh^2(aX)} \exp\sbra{\frac{-X^2}{2\sigma_x^2}} dX\\
        &=a \sigma_x^2 \mbE{1-\tanh^2 (aX)},
        \end{align*}
and        
    \begin{align*}         
        \mbE{\tanh \sbra{a(X+Y)}X}&= \frac{1}{2\pi \sigma_x \sigma_y} \int_{-\infty}^{\infty}X \exp\sbra{\frac{-X^2}{2\sigma_x^2}} \int_{-\infty}^{\infty} \tanh(a(X+Y)) \exp\sbra{\frac{-Y^2}{2\sigma_y^2}} dY dX\\
        &=- \frac{1}{2\pi \sigma_x \sigma_y} \cdot \sigma_x^2 \mbra{\int_{-\infty}^{\infty} \tanh(a(X+Y)) \exp\sbra{\frac{-Y^2}{2\sigma_y^2}} dY} \exp\sbra{\frac{-X^2}{2\sigma_x^2}}\Big|_{-\infty}^{\infty}\\
        &\quad +\frac{1}{2\pi \sigma_x \sigma_y}\iint_{-\infty}^{\infty} a\sigma_x^2 \sbra{1-\tanh^2(a(X+Y))} \exp\sbra{\frac{-X^2}{2\sigma_x^2}}\exp\sbra{\frac{-Y^2}{2\sigma_y^2}} dY dX\\
        &=a \sigma_x^2 \mbE{1-\tanh^2 (a(X+Y))},
        \end{align*}  
      where we used the fact that the first term in the second line is $0$.
    \end{proof}

\begin{lemma}\label{lemma:loss_lipschitz}
The functions ${\psi}_\mu^{\prime}(x)=\tanh(x/\mu)$ and $\psi_\mu^{\prime\prime}(x)=\sbra{1-\tanh^2\sbra{x /\mu}}/\mu$ are Lipschitz continuous with Lipschitz constants $1/\mu$ and $2/\mu^2$, respectively. 
\end{lemma}

\begin{proof}
    Since $\psi_\mu(x)$ is continuous and third-order differentiable, we have for any $x$ and $x'$,
     \begin{align*}
        \abs{{\psi}_\mu^{\prime} (x)- \psi_\mu^{\prime}(x')}&\leq \abs{\int_x^{x'} {\psi}_\mu^{\prime\prime}(z) dz}\leq \abs{x-x'} \max_{z} | \psi_\mu^{\prime\prime}(z)| \leq \frac{\abs{x-x'}}{\mu} , 
      \end{align*}  
and
\begin{align*}
        \abs{{\psi}_\mu^{\prime\prime}(x)- {\psi}_\mu^{\prime\prime}(x')}&\leq \abs{\int_x^{x'} \psi_{\mu}^{\prime\prime\prime}(z)dz } \leq \int_x^{x'} \abs{-\frac{2}{\mu^2}\tanh\sbra{\frac{z}{\mu}}  \sbra{1-\tanh^2\sbra{\frac{z}{\mu} }} } dz\leq \frac{2\abs{x-x'}}{\mu^2}, 
        \end{align*} 
     where we use the fact that $\abs{\tanh(x)}\leq 1$ and $1-\tanh^2\sbra{x } \leq 1$ for all $x\in\mathbb{R}$.
\end{proof}

\begin{lemma}\label{lemma:Operator norm of Circulant Matrix}
    Let $\bx\sim_{iid} \mathrm{BG}(\theta)$ for $\theta\in(0,1]$. There exist some constants $c_1$ and $c_2$ such that
    \begin{equation*}
        \mbP{\norm{\cC(\bx)} \geq t} \leq 2n \exp\sbra{\frac{-t^2}{c_1 n}}, \quad \mbox{and}\quad        \mathbb{E}\norm{\cC(\bx)}^{2m} \leq \frac{m!}{2} \sbra{c_2 n \log n}^m
    \end{equation*}
for all $m\geq 1$.
\end{lemma}

\begin{proof}
   Since a circulant matrix is diagonalizable by the DFT matrix, the spectral norm of $\cC(\bx)$ is the maximum magnitude of the DFT coefficients of $\bx$, where the $i$th coefficient is given as $\hat{x}_i = \boldsymbol{f}_i^{\mathsf{H}}\bx$, where $\boldsymbol{f}_i=[1, e^{j2\pi i/n}, \cdots, e^{j2\pi i(n-1)/n}]^{\top}$ is the $i$th column of the DFT matrix. Since $\bx\sim_{iid} \mathrm{BG}(\theta)$ is sub-Gaussian, by Fact~\ref{fact:BG to SG}, $\hat{x}_i$ is also sub-Gaussian with $\|\hat{x}_i\|_{\psi_2}\leq C\|\boldsymbol{f}_i\|_2 = C\sqrt{n}$. Therefore, by the union bound, together with Fact \ref{fact: tail bound of BG}, we have    
   \begin{equation*} 
    \mbP{\norm{\cC(\bx)}\geq t}=\mbP{\max_{i\in[n]} \abs{\hat{x}_i }\geq t} \leq 2n \exp\sbra{\frac{-t^2}{c_1n}},
\end{equation*}
for some constant $c_1$. Equipped with the above bound, we can bound the moments of $\norm{\cC(\bx)}^2$ by  
\begin{align} 
\mathbb{E}\norm{\cC(\bx)}^{2m} &=\int_0^\infty \mathbb{P}(\norm{\cC(\bx)}^{2m} >u) du =  \int_0^\infty   \mathbb{P}(\norm{\cC(\bx)} >t)\cdot 2m t^{2m-1} dt,
\end{align}
where the second equality follows by a change of variable $t=u^{1/2m}$. To continue, we break the bound as
\begin{align*}    
\mathbb{E}\norm{\cC(\bx)}^{2m}    &\leq \int_0^{2\sqrt{c_1 n\log n}} 1\cdot 2m t^{2m-1} dt +\int_{2\sqrt{c_1 n\log n }}^\infty 2n\exp\sbra{\frac{-t^2}{c_1 n}}2m t^{2m-1} dt   \\
    &\leq \sbra{4c_1 n \log n}^m + \int_0^\infty \exp\sbra{\frac{-t^2}{2 c_1 n}}2m t^{2m-1} dt   \\
    &=\sbra{4c_1n \log n}^m + \sbra{2c_1 n}^m m! \leq \frac{m!}{2} \sbra{c_2 n \log n}^m ,
    \end{align*}
where the second line used the fact $\exp\sbra{\frac{-t^2}{2c_1 n}}>2n\exp\sbra{\frac{-t^2}{c_1n}}$ when $t\geq 2\sqrt{c_1n\log n}$, and the third line used the definition of the Gamma function. The proof is completed.
\end{proof}

\begin{lemma}\label{lemma: other1}
Let $\{\bx_i \}_{i=1}^p\in\mathbb{R}^n$ be drawn according to $\bx_i\sim_{iid} \mathrm{BG}(\theta)$, $\theta\in (0,1/2)$. There exists some constant $C$ such that
    \begin{equation*}
 \norm{\frac{1}{\theta np} \sum_{i=1}^p \cC(\bx_i)^\top \cC(\bx_i)-\bI}\leq C\sqrt{\frac{\log^2 n \log p}{\theta^2 p}}
    \end{equation*}
 holds with probability at least $1- 2n p^{-8}$.
\end{lemma}
\begin{proof}
By assumption, it is easy to check
$$\mbE{\frac{1}{\theta np} \sum_{i=1}^p \cC(\bx_i)^\top \cC(\bx_i)}= \mbE{\frac{1}{\theta n} \cC(\bx_1)^\top \cC(\bx_1)}=\bI. $$
The remaining of the proof is to verify the quantities needed to apply Lemma~\ref{lemma:matrix_bernstein}. Specifically, we bound the $m$th moment of $\frac{1}{\theta n}\cC(\bx_i)^\top \cC(\bx_i)$ as
\begin{equation*} 
        \mathbb{E}\left\| \frac{1}{\theta n}  \cC(\bx_i)^\top \cC(\bx_i) \right\|^m =
     \frac{1}{\theta^m n^m}\mathbb{E} \norm{\cC(\bx_i)}^{2m} \leq  \frac{m!}{2} \sbra{\frac{c \log n}{\theta}}^m ,
\end{equation*}
where the last line comes from Lemma \ref{lemma:Operator norm of Circulant Matrix}. Let $\sigma^2= \frac{c^2 \log^2 n}{\theta^2}$, $R=\frac{c \log n}{\theta}$ in Lemma~\ref{lemma:matrix_bernstein}, we have  
\begin{equation}\label{equ:F_i}
    \mathbb{P}\sbra{\norm{\frac{1}{\theta np} \sum_{i=1}^p \cC(\bx_i)^\top \cC(\bx_i)-\bI}\geq t} \leq 
    2n \exp\sbra{\frac{-p\theta^2 t^2}{2 c^2 \log^2 n+ 2c \theta t\log n }}.
\end{equation}
Setting $t= C \sqrt{\frac{\log^2n\log p}{\theta^2 p}}$ for some sufficiently large $C$, we complete the proof.
\end{proof}

\section{Proofs for Section~\ref{sec:proof_ortho_geometry}} \label{proof:orthogonal_geometry}

\subsection{Proof of Lemma~\ref{thm:population_geometry_orthogonal}} \label{proof:population_geometry_orthogonal}

Recall the two regions introduced in (\ref{equ:subregions}):
\begin{equation*}
        \mathcal{Q}_1 := \left \{\bw: \frac{\mu}{4\sqrt{2}} \leq \norm{\bw}_2\leq \sqrt{\frac{n-1}{n+\xi_0}} \right \}, \quad
        \mathcal{Q}_2 := \left \{\bw: \norm{\bw}_2\leq \frac{\mu}{4\sqrt{2}} \right \} .
\end{equation*}
We further divide $Q_1$ into two subregions, 
$$\mathcal{R}_0 =\left\{\bw: \frac{\mu}{4\sqrt{2}}\leq \norm{\bw}_2\leq \frac{1}{20\sqrt{5}}\right\} , \quad \quad \mathcal{R}_1=\left\{\bw: \frac{1}{20\sqrt{5}}\leq \norm{\bw}_2\leq \sqrt{\frac{n-1}{n+\xi_0}}\right\}, $$ 
which we will prove the desired bound separately.

Note that
\begin{equation} \label{eq:reduction}
\mathbb{E}  \phi_o(\bw) =n\cdot \mathbb{E}  \psi_{\mu}\sbra{\bx^\top \bh(\bw) },
\end{equation}
since every row of $\mathcal{C}(\bx)$ has the same distribution as $\bx\sim_{iid} \mathrm{BG}(\theta)$. Therefore, the strong convexity bound \eqref{eq:sc_pop_geometry} in $\mathcal{Q}_2$ follows directly from the following lemma from \cite[Proposition 8]{sun2017complete} by a multiplicative factor of $n$.
\begin{lemma}[{\cite[Proposition 8]{sun2017complete}}]\label{lemma:related_strong_convexity}
 For any $\theta\in(0,1/2)$, if $\mu\leq \frac{1}{20\sqrt{n}}$, it holds for all $\bw$ with $\norm{\bw}_2\leq \frac{\mu}{4\sqrt{2}}$ that $\nabla_{\bw}^2 \mathbb{E}\psi_\mu\sbra{\bx^\top \bh(\bw)}\succeq \frac{\theta}{5\sqrt{2\pi}\mu} \bI$.
\end{lemma}

Similarly, by the following lemma from \cite[Proposition 7]{sun2017complete}, we have the desired bound \eqref{eq:lg_pop_geometry} in $\mathcal{R}_0$.
\begin{lemma}[{\cite[Proposition 7]{sun2017complete}}]  \label{lemma: related_expectation_gradient} 
 For any $\theta\in (0,1/3)$, if $\mu\leq 9/50$, it holds for all $\bw\in\mathcal{R}_0$ such that
 \begin{equation*}
    \frac{\bw^\top \nabla_{\bw} \mathbb{E}\psi_\mu (\bx^\top \bh(\bw))}{\norm{\bw}_2}\geq \frac{\theta}{20\sqrt{2\pi}}.
 \end{equation*}
\end{lemma}

Therefore, the remainder of the proof is to show that \eqref{eq:lg_pop_geometry} also applies to $\mathcal{R}_1$. To ease presentation, we introduce a few short-hand notations. For $\bx = \bOmega \odot \bz \sim_{iid} \mathrm{BG}(\theta)\in \mathbb{R}^n$, we denote the first $n-1$ dimension of $\bx$, $\bz$ and $\bOmega$ as $\bar{\bx}$, $\bar{\bz}$ and $\bar{\bOmega}$, respectively. Denote $\mathcal{I}$ as the support of $\bm{\Omega}$ and $\mathcal{J}$ as the support of $\bar{\bOmega}$.

Note that it is easy to confirm the exchangeability of the expectation and derivatives \cite[Lemma 31]{sun2017complete} as
\begin{subequations}
\begin{align}
    \frac{\bw^\top \nabla_{\bw} \mathbb{E} \psi_\mu (\bx^\top \bh(\bw))}{\norm{\bw}_2} &= \mathbb{E}  \mbra{\frac{\bw^\top\nabla_{\bw} \psi_{\mu}\sbra{\bx^\top \bh(\bw)}}{\norm{\bw}_2} }, \label{equ:communication of gradient} \\
    \nabla_{\bw}^2 \mathbb{E}\psi_\mu\sbra{\bx^\top \bh(\bw)} & = \mathbb{E}\nabla_{\bw}^2 \psi_\mu\sbra{\bx^\top \bh(\bw)}. \label{equ:communication of hessian}
\end{align}
\end{subequations}

Thus, plugging in (\ref{equ:gradient_psi}), we rewrite the expectation of the directional gradient as following:
    \begin{align}
    \mathbb{E} &  \mbra{\frac{\bw^\top\nabla_{\bw} \psi_{\mu}\sbra{\bx^\top \bh(\bw)}}{\norm{\bw}_2} } 
    =\frac{1}{\norm{\bw}_2}\mathbb{E}\mbra{\tanh\sbra{\frac{\bx^\top \cdot \bh(\bw)}{\mu}}\cdot \sbra{\bw^\top\bar{\bx}-\frac{x_n \norm{\bw}_2^2}{h_n}}} \nonumber\\
      &=\frac{(1-\theta)}{\norm{\bw}_2} \underbrace{\mathbb{E}_{\bar{\bx}} \mbra{\tanh \left(\frac{\bw^\top \bar{\bx}}{\mu}\right) \bw^\top \bar{\bx}}}_{I_1}+\frac{\theta}{\norm{\bw}_2} \underbrace{\mathbb{E}_{\bar{\bx},z_n} \mbra{\tanh \sbra{\frac{\bw^\top \bar{\bx}+h_n z_n}{\mu}}\sbra{\bw^\top \bar{\bx}-\frac{\norm{\bw}_2^2}{h_n}z_n}}}_{I_2} , \label{eq:exp_direct_gradient}
    \end{align}
where the second line is expanded over the distribution of $\Omega_n\sim\mathrm{Bernoulli}(\theta)$. Conditioned on the support of $\bar{\bOmega}$,  we have $X= \bw^\top \bar{\bx} | \bar{\bOmega} \sim \mathcal{N}(0,\norm{\bw_{\cJ}}_2^2)$. Moreover, denote $Y= h_n z_n \sim \mathcal{N}(0,h_n^2)$. Therefore, invoking Lemma \ref{lemma:loss_expectation}, we can express $I_1$ and $I_2$ respectively as
\begin{align*}
I_1 &= \mathbb{E}_{\bar{\bOmega}} \mbra{ \mathbb{E}_{X} \sbra{\tanh \left(\frac{X}{\mu}\right) X} } =\frac{1 }{\mu }\mathbb{E}_{\bar{\bOmega}} \mbra{\norm{\bw_\cJ}_2^2 \mathbb{E}_{X} \sbra{1-\tanh^2\sbra{\frac{X}{\mu}}}} , \\
I_2  &= \mathbb{E}_{\bar{\bOmega}} \mbra{ \mathbb{E}_{X,Y} \sbra{ \tanh \sbra{\frac{X + Y}{\mu}}\sbra{X-\frac{\norm{\bw}_2^2}{h_n^2}Y}}} = \frac{1}{\mu}\mathbb{E}_{\bar{\bOmega}} \mbra{ \left( \norm{\bw_\cJ}_2^2 -  \norm{\bw}_2^2\right)  \mathbb{E}_{X,Y} \sbra{1-\tanh^2\sbra{\frac{X+Y}{\mu}}}}.
\end{align*}
Plugging the above quantities back into \eqref{eq:exp_direct_gradient}, and using $\norm{\bw_{\cJ}}_2^2=\sum_{i=1}^{n-1} w_i^2 \mathbbm{1}\{\Omega_i = 1 \}$, $\norm{\bw_{\cJ^c}}_2^2=\sum_{i=1}^{n-1} w_i^2 \mathbbm{1}\{\Omega_i = 0 \}$, we arrive at
\begin{align}
     \mathbb{E} & \mbra{\frac{\bw^\top\nabla_{\bw} \psi_{\mu}\sbra{\bx^\top \bh(\bw)}}{\norm{\bw}_2} }    = \frac{(1-\theta)}{\mu\norm{\bw}_2} \mathbb{E}_{\bar{\bOmega}} \mbra{\sum_{i=1}^{n-1} w_i^2 \mathbbm{1}\{\Omega_i = 1\}  \cdot \mathbb{E}_{\bz} \sbra{1-\tanh^2\sbra{\frac{\bw_{\setminus \{i\}}^\top\bar{\bx}_{\backslash \{i\}}+w_i z_i \mathbbm{1}\{\Omega_i=1\}}{\mu}}}} \nonumber \\
    &\quad -  \frac{\theta}{\mu\norm{\bw}_2} \mathbb{E}_{\bar{\bOmega}} \mbra{ \sum_{i=1}^{n-1} w_i^2 \mathbbm{1}\{\Omega_i = 0\}  \cdot\mathbb{E}_{\bz} \sbra{1-\tanh^2\sbra{\frac{\bw_{\setminus \{i\}}^\top\bar{\bx}_{ \backslash\{i\}}+w_i z_i \mathbbm{1}\{\Omega_i = 1 \}+h_n z_n}{\mu}}}} \nonumber \\
    & = \frac{1}{\mu\|\bw\|_2}\sum_{i=1}^{n-1} w_i^2 Q_i, \label{equ:expectation_to_Q}
\end{align}
where $Q_i$ is written as
\begin{align*}
Q_i & = (1-\theta)\mathbb{E}_{\bar{\bOmega}} \mbra{\mathbbm{1}\{\Omega_i = 1\}  \cdot \mathbb{E}_{\bz} \sbra{1-\tanh^2\sbra{\frac{\bw_{\setminus \{i\}}^\top\bar{\bx}_{\backslash \{i\}}+w_i z_i \mathbbm{1}\{\Omega_i=1\}}{\mu}}} } \\
&\quad\quad\quad  -  \theta \mathbb{E}_{\bar{\bOmega}} \mbra{ \mathbbm{1}\{\Omega_i = 0\}  \cdot\mathbb{E}_{\bz} \sbra{1-\tanh^2\sbra{\frac{\bw_{\setminus \{i\}}^\top\bar{\bx}_{ \backslash\{i\}}+w_i z_i \mathbbm{1}\{\Omega_i = 1 \}+h_n z_n}{\mu}}}} .
\end{align*}
Evaluating $\mathbb{E}_{\bar{\bOmega}}$ over $\bar{\bOmega}\setminus \{i\}$ and $\Omega_i$ sequentially, and combining terms, we can rewrite $Q_i$ as,
\begin{align}
Q_i & = (1-\theta)\theta  \mathbb{E}_{\bar{\bOmega}\setminus \{i\}}  \mbra{\mathbb{E}_{\bz} \sbra{1-\tanh^2\sbra{\frac{\bw_{\setminus \{i\}}^\top\bar{\bx}_{\backslash \{i\}}+w_i z_i }{\mu}}} -   \mathbb{E}_{\bz} \sbra{1-\tanh^2\sbra{\frac{\bw_{\setminus \{i\}}^\top\bar{\bx}_{ \backslash\{i\}} +h_n z_n}{\mu}}}}  \nonumber \\
& =  (1-\theta)\theta  \mathbb{E}_{\bar{\bOmega}\setminus \{i\}}  \bigg[ \mathbb{E}_{\bz}  \underbrace{ \sbra{ \tanh^2\sbra{\frac{\bw_{\setminus \{i\}}^\top\bar{\bx}_{ \backslash\{i\}} +h_n z_n}{\mu}} -\tanh^2\sbra{\frac{\bw_{\setminus \{i\}}^\top\bar{\bx}_{\backslash \{i\}}+w_i z_i }{\mu}}} }_{=:K_i}    \bigg]. \label{eq:expression_Q}
\end{align}
Our goal is to lower bound $Q_i$ for all $i\in[n-1]$. Without loss of generality, we denote the index of the largest entry of $\bw$ in magnitude as $i_0$, i.e, $\abs{w_{i_0}}\geq \abs{w_j}$, $\forall j\in[n-1]$.  We first claim 
\begin{equation}\label{eq:intermediate_claim}
Q_j \geq Q_{i_0}, \quad \forall j\in[n-1],
\end{equation}
whose proof is given at the end of this subsection. With this claim,  we only need to lower bound $Q_{i_0}$.
We proceed to lower bound $\mathbb{E}_{\bz}[K_{i_0}]$. Let $X:= \bw_{\setminus \{i_0\}}^\top\bar{\bx}_{\backslash \{i_0\}}+w_{i_0} z_{i_0} | \bar{\bOmega} \sim \mathcal{N}(0, \norm{\bw_{\cJ \setminus \{i_0\}}}_2^2+w_{i_0}^2 ) :=\mathcal{N}(0,\sigma_X^2)$ and
$Y:=\bw_{\setminus \{i_0\}}^\top\bar{\bx}_{\backslash \{i_0\}}+ h_n z_n | \bar{\bOmega} \sim \mathcal{N}(0, \norm{\bw_{\cJ \setminus \{i_0\}}}_2^2+h_n^2):=\mathcal{N}(0,\sigma_Y^2)$.
 By the fundamental theorem of calculus, we have
\begin{align}
  K_{i_0}&=\tanh^2\sbra{\frac{Y}{\mu}}-\tanh^2\sbra{\frac{X}{\mu}}\nonumber \\
    &=\frac{2}{\mu} \int_{ \abs{X}} ^{\abs{Y}} \tanh\sbra{\frac{x}{\mu}} \cdot \sbra{1- \tanh^2\sbra{\frac{x}{\mu}}} dx  \nonumber\\
    &\geq \frac{2}{\mu} \int_{ \abs{X}} ^{\abs{Y}} \mbra{ 2\exp\sbra{\frac{-2x}{\mu}} -\exp\sbra{\frac{-4x}{\mu}}   } \mbra{ 1- 2\exp\sbra{\frac{-2x}{\mu}}   } dx   \nonumber\\
    &\geq \frac{2}{\mu} \int_{ \abs{X}} ^{\abs{Y}} \mbra{ 2\exp\sbra{\frac{-2x}{\mu}} - 5\exp\sbra{\frac{-4x}{\mu}}  } dx  \nonumber\\
    &=\underbrace{2\mbra{ \exp\sbra{\frac{-2\abs{X}}{\mu}}-\exp\sbra{\frac{-2\abs{Y}}{\mu}} } }_{K_1}
     - \underbrace{\frac{5}{2}\mbra{\exp\sbra{\frac{-4\abs{X}}{\mu}} - \exp\sbra{\frac{-4\abs{Y}}{\mu}} }}_{K_2}  , \label{equ:K_1_K_2_K_3}
\end{align}
where the third line follows from the bounds $2 \exp(-2x/\mu)- \exp(-4x/\mu)  \leq 1-\tanh^2\sbra{{x}/{\mu}}$ and $\tanh(x/\mu) \leq 1-\exp(-2x/ \mu)$ in \cite[Lemma 29]{sun2017complete}. To continue, we record the lemma rephrased from \cite[Lemma 32, 40]{sun2017complete} and obtain the following lemma by directly repeating integration by parts.
\begin{lemma}[{\cite[Lemma 32, 40]{sun2017complete}}]\label{lemma:exponential_function}
    Let $X\sim \mathcal{N} (0, \sigma_X^2)$. For any $a>0$, we have
    \begin{align*}
      \frac{1 }{\sqrt{2\pi}} \sbra{ \frac{1}{a\sigma_X}-\frac{1}{a^3\sigma_X^3} + \frac{3}{a^5 \sigma_X^5} - \frac{15}{a^7 \sigma_X^7} } \leq    \mathbb{E}  \mbra{\exp(-aX) \mathbbm{1} \{X>0\} }   \leq \frac{1}{\sqrt{2\pi} } \sbra{ \frac{1}{a\sigma_X}-\frac{1}{a^3\sigma_X^3} + \frac{3}{a^5 \sigma_X^5} }   .
    \end{align*}
\end{lemma}

Therefore, $K_1$ can be bounded as
\begin{align*}
K_1  &= 2 \mathbb{E} \mbra{ \exp\sbra{\frac{-2 \abs{X}}{\mu} } - \exp\sbra{\frac{-2 \abs{Y}}{\mu} } }  \\
&= 4 \mathbb{E}\mbra{ \exp\sbra{\frac{-2X}{\mu}} \mathbbm{1} \{X>0\}  -  \exp\sbra{\frac{-2Y}{\mu}} \mathbbm{1} \{Y>0\}  } \\
& \geq \frac{4}{\sqrt{2\pi}}\sbra{\frac{\mu}{2\sigma_X} - \frac{\mu^3}{8\sigma_X^3} + \frac{3\mu^5}{32\sigma_X^5} - \frac{15 \mu^7}{ 2^7\sigma_X^7} } - \frac{4}{\sqrt{2\pi}}\sbra{\frac{\mu}{2\sigma_Y}-\frac{\mu^3}{8\sigma_Y^3}+\frac{3\mu^5}{32\sigma_Y^5} } \\
& = \frac{2}{\sqrt{2\pi}} \mbra{ \sbra{\frac{\mu}{\sigma_X} - \frac{\mu}{\sigma_Y}} - \sbra{ \frac{\mu^3}{ 4\sigma_X^3}-\frac{\mu^3}{ 4\sigma_Y^3}  } +\sbra{ \frac{3\mu^5}{ 16\sigma_X^5}- \frac{3\mu^5}{ 16\sigma_Y^5}} -\frac{15 \mu^7}{ 2^6\sigma_X^7}   }.
\end{align*}
Similarly, we have
\begin{equation*}
    K_2 \leq \frac{5}{4 \sqrt{2\pi}} \mbra{ \sbra{ \frac{\mu}{\sigma_X} - \frac{\mu}{\sigma_Y} } - \sbra{ \frac{\mu^3}{ 16\sigma_X^3}-\frac{\mu^3}{ 16\sigma_Y^3}  } + \sbra{ \frac{3\mu^5}{ 4^4\sigma_X^5}- \frac{3\mu^5}{ 4^4\sigma_Y^5}} -\frac{15 \mu^7}{ 4^6\sigma_Y^7}  }. 
\end{equation*}
Plugging the above bounds back into \eqref{equ:K_1_K_2_K_3}, we have
\begin{align}
    & \mathbb{E}_{\bz}\mbra{K_{i_0}}    \geq \mathbb{E}_{X,Y} \mbra{K_1-K_2} \nonumber\\
    & \geq \frac{2}{\sqrt{2\pi}} \mbra{ \sbra{\frac{\mu}{\sigma_X} - \frac{\mu}{\sigma_Y}} - \sbra{ \frac{\mu^3}{ 4\sigma_X^3}-\frac{\mu^3}{ 4\sigma_Y^3}  } +\sbra{ \frac{3\mu^5}{ 16\sigma_X^5}- \frac{3\mu^5}{ 16\sigma_Y^5}} -\frac{15 \mu^7}{ 2^6\sigma_X^7}   }\\
    &    -     \frac{5}{4 \sqrt{2\pi}} \mbra{ \sbra{ \frac{\mu}{\sigma_X} - \frac{\mu}{\sigma_Y} } - \sbra{ \frac{\mu^3}{ 16\sigma_X^3}-\frac{\mu^3}{ 16\sigma_Y^3}  } + \sbra{ \frac{3\mu^5}{ 4^4\sigma_X^5}- \frac{3\mu^5}{ 4^4\sigma_Y^5}} - \frac{15 \mu^7}{ 4^6\sigma_Y^7}  } \nonumber\\
    &=\frac{1}{\sqrt{2\pi}}\mbra{ \frac{3\mu}{4} \sbra{ \frac{1}{\sigma_X}-\frac{1}{\sigma_Y} }
    -  \frac{27\mu^3}{64} \sbra{ \frac{1}{\sigma_X^3} -\frac{1}{\sigma_Y^3} } + \frac{113 \mu^5}{4^5}  \sbra{ \frac{1}{\sigma_X^5} - \frac{1}{\sigma_Y^5} }  
    - \frac{15 \mu^7}{ 2^5\sigma_X^7} - \frac{75 \mu^7}{ 4^7\sigma_Y^7}       } \nonumber \\
    &= \frac{1}{\sqrt{2\pi}} \mbra{ \sbra{ \frac{1}{\sigma_X}-\frac{1}{\sigma_Y} } \sbra{ \frac{3\mu}{4} 
    -  \frac{27\mu^3}{64} \sbra{ \frac{1}{\sigma_X^2} + \frac{1}{\sigma_Y^2} +\frac{1}{\sigma_X \sigma_Y} } }
   + \frac{113 \mu^5}{4^5}  \sbra{ \frac{1}{\sigma_X^5} - \frac{1}{\sigma_Y^5} }  
    - \frac{15 \mu^7}{ 2^5\sigma_X^7} - \frac{75 \mu^7}{ 4^7\sigma_Y^7}         } \nonumber \\
    &\geq \frac{1}{\sqrt{2\pi}} \mbra{ \sbra{ \frac{1}{\sigma_X}-\frac{1}{\sigma_Y} } \sbra{ \frac{3\mu}{4} 
    -  \frac{27\mu^3}{64} \sbra{ \frac{1}{\sigma_X^2} + \frac{1}{\sigma_Y^2} +\frac{1}{\sigma_X \sigma_Y} } } 
    - \frac{\mu^7}{2\sigma_X^7} } \label{equ:all_bound}
\end{align}
where the last line follows from the fact $\sigma_X <\sigma_Y$ and $\frac{113 \mu^5}{4^5}  \sbra{ \frac{1}{\sigma_X^5} - \frac{1}{\sigma_Y^5} } >0$. 

To continue, since $\sigma_X= \sqrt{\norm{\bw_{\cJ \backslash \{i_0\}}}_2^2+h_{i_0}^2} <1 $ and $\sigma_Y= \sqrt{\norm{\bw_{\cJ \backslash \{i_0\}}}_2^2+h_n^2} <1$,
\begin{align}
\frac{1}{\sigma_X} -\frac{1}{  \sigma_Y }   =\frac{\sigma_Y^2-\sigma_X^2}{\sigma_X\sigma_Y(\sigma_X+\sigma_Y)}& \geq \frac{\sigma_Y^2-\sigma_X^2}{2}  
=\frac{1}{2}\sbra{h_n^2-h_{i_0}^2}\geq \frac{1}{2}\mbra{h_n^2- \frac{1}{1+\xi_0}h_n^2} \geq \frac{\xi_0}{4n}, \label{equ:first_order_bound}
\end{align}
where the second inequality uses the fact ${h_n^2}/{h_{i_0}^2}\geq 1+\xi_0$, $h_n\geq 1/\sqrt{n}$ and $\xi_0\in(0,1)$. In addition, as $\abs{h_{i_0}} =\max_i w_i$ and $\norm{\bw}_2> \frac{1}{20\sqrt{5}}$, we have $\abs{h_{i_0}} \geq \frac{1}{20\sqrt{5n}}$. So we have $ \frac{1}{\sigma_X}\leq \frac{1}{\abs{h_{i_0}}}\leq 20 \sqrt{5n},$ such that
\begin{equation}\label{equ:mu_bound_1}
    \frac{1}{\sigma_X^2} + \frac{1}{\sigma_Y^2} +\frac{1}{\sigma_X \sigma_Y} \leq 2061n \quad \Longrightarrow \quad \frac{27\mu^3}{64} \sbra{ \frac{1}{\sigma_X^2} + \frac{1}{\sigma_Y^2} +\frac{1}{\sigma_X \sigma_Y} } \leq \frac{\mu}{4},
\end{equation}
provided $\mu \leq cn^{-1/2}$ for a sufficiently small $c>0$. Plugging \eqref{equ:first_order_bound} and \eqref{equ:mu_bound_1} back into \eqref{equ:all_bound}, we have
\begin{equation}\label{equ:bound_for_K}
    \mathbb{E}_{\bz}\mbra{K_{i_0}}  \geq \frac{ \mu \xi_0}{8 \sqrt{2\pi}n } -\frac{\mu^7}{ 2\sqrt{2\pi}\sigma_X^7 } \geq \frac{ \mu \xi_0}{16 \sqrt{2\pi}n } ,
\end{equation}
conditioned on the support $\bar{\bOmega}_{\backslash \{i_0\} }$, provided that $\frac{1}{\sigma_X^7} \leq (20\sqrt{5})^7 n^{7/2}$ and $\mu\leq c \xi_0^{1/6} n^{-3/4}$ for a sufficiently small $c>0$.

Plugging \eqref{equ:bound_for_K} back into \eqref{eq:expression_Q}, then into \eqref{equ:expectation_to_Q} with the help of \eqref{eq:intermediate_claim}, finally, by the assumption $\norm{\bw}_2 \geq \frac{1}{20\sqrt{5} }$ and \eqref{equ:communication of gradient}, we have
\begin{equation}
    \frac{\bw^\top \mathbb{E} \nabla \phi_o(\bw)}{\norm{\bw}_2 } =\mathbb{E} \mbra{n\frac{\bw^\top\nabla_{\bw} \psi_{\mu}\sbra{\bx^\top \bh(\bw)}}{\norm{\bw}_2} }  \geq n \frac{\norm{\bw}_2 \theta (1-\theta)}{\mu} \frac{ \mu\xi_0}{16 \sqrt{2\pi}n} \geq \frac{\theta \xi_0}{480\sqrt{10\pi}},
\end{equation}
where the final bound follows from the constraint $\theta\in(0,1/3)$.

\paragraph{Proof of \eqref{eq:intermediate_claim}:}
For any $j\in[n-1]$ and $j\neq i_0$, by evaluating $\bar{\bOmega}\setminus \{i_0\}$ over $\bar{\bOmega}\setminus \{i_0,j \}$ and $\Omega_j$ sequentially, we can rewrite $Q_{i_0}$ as
\small
\begin{align}
   Q_{i_0} &  = (1-\theta)\theta^2 \mathbb{E}_{\bar{\bOmega}\setminus \{i_0, j\}}  \bigg[ \mathbb{E}_{\bz}   \sbra{  \tanh^2\sbra{\frac{\bw_{\setminus \{i_0,j\}}^\top\bar{\bx}_{ \backslash\{i_0, j\}} + h_n z_n + w_j z_j}{\mu}} -\tanh^2\sbra{\frac{\bw_{\setminus \{i_0, j\}}^\top\bar{\bx}_{\backslash \{i_0, j\}}+w_{i_0} z_{i_0} + w_j z_j }{\mu}} }     \bigg] \nonumber \\
    &  + (1-\theta)^2\theta  \mathbb{E}_{\bar{\bOmega}\setminus \{i_0, j\}}  \bigg[ \mathbb{E}_{\bz}  \sbra{  \tanh^2\sbra{\frac{\bw_{\setminus \{i_0,j\}}^\top\bar{\bx}_{ \backslash\{i_0,j\}} +h_n z_n}{\mu}} -\tanh^2\sbra{\frac{\bw_{\setminus \{i_0,j\}}^\top\bar{\bx}_{\backslash \{i_0,j\}}+w_{i_0} z_{i_0} }{\mu}} }    \bigg].\label{eq:expression_Q_i0}.
\end{align}
\normalsize
Similarly, we can write $Q_j$ as 
\small
\begin{align}
    Q_{j} & = (1-\theta)\theta^2 \mathbb{E}_{\bar{\bOmega}\setminus \{i_0, j\}}  \bigg[ \mathbb{E}_{\bz}   \sbra{  \tanh^2\sbra{\frac{\bw_{\setminus \{i_0,j\}}^\top\bar{\bx}_{ \backslash\{i_0, j\}} + h_n z_n + w_{i_0} z_{i_0}}{\mu}} -\tanh^2\sbra{\frac{\bw_{\setminus \{i_0, j\}}^\top\bar{\bx}_{\backslash \{i_0, j\}}+w_{i_0} z_{i_0} + w_j z_j }{\mu}} }    \bigg] \nonumber \\
    & \quad + (1-\theta)^2\theta  \mathbb{E}_{\bar{\bOmega}\setminus \{i_0, j\}}  \bigg[ \mathbb{E}_{\bz}   \sbra{  \tanh^2\sbra{\frac{\bw_{\setminus \{i_0,j\}}^\top\bar{\bx}_{ \backslash\{i_0,j\}} +h_n z_n}{\mu}} -\tanh^2\sbra{\frac{\bw_{\setminus \{i_0,j\}}^\top\bar{\bx}_{\backslash \{i_0,j\}}+w_j z_j }{\mu}} }    \bigg] \label{eq:expression_Q_j}.
\end{align}
\normalsize
Combining \eqref{eq:expression_Q_i0} and \eqref{eq:expression_Q_j}, we have 
\small
\begin{align}\label{equ:Q_jQ_i}
    &\quad Q_j - Q_{i_0} \nonumber \\
     & = (1-\theta)\theta^2 \mathbb{E}_{\bar{\bOmega}\setminus \{i_0, j\}}  \bigg[  \underbrace{ \mathbb{E}_{\bz}  \sbra{  \tanh^2\sbra{\frac{\bw_{\setminus \{i_0,j\}}^\top\bar{\bx}_{ \backslash\{i_0, j\}} + h_n z_n + w_{i_0} z_{i_0}}{\mu}} - \tanh^2\sbra{\frac{\bw_{\setminus \{i_0,j\}}^\top\bar{\bx}_{ \backslash\{i_0, j\}} + h_n z_n + w_j z_j}{\mu}}  } }_{I_3}   \bigg] \nonumber \\
    & \quad + (1-\theta)^2\theta \mathbb{E}_{\bar{\bOmega}\setminus \{i_0, j\}}  \bigg[ \underbrace{ \mathbb{E}_{\bz}   \sbra{ \tanh^2\sbra{\frac{\bw_{\setminus \{i_0,j\}}^\top\bar{\bx}_{\backslash \{i_0,j\}}+w_{i_0} z_{i_0} }{\mu}} - \tanh^2\sbra{\frac{\bw_{\setminus \{i_0,j\}}^\top\bar{\bx}_{\backslash \{i_0,j\}}+w_j z_j }{\mu}} } }_{I_4}    \bigg] .
\end{align}
\normalsize
To show that $I_3 \geq 0$, let $X_1:= \bw_{\setminus \{i_0,j\}}^\top\bar{\bx}_{ \backslash\{i_0, j\}} + h_n z_n + w_{i_0} z_{i_0}| \bar{\bOmega} \sim \mathcal{N}(0, \norm{\bw_{\cJ \setminus \{i_0, j\}}}_2^2+ h_n^2 + w_{i_0}^2  =: \sigma_{X_1}^2)$ and $Y_1:= \bw_{\setminus \{i_0,j\}}^\top\bar{\bx}_{ \backslash\{i_0, j\}} + h_n z_n + w_j z_j |\bar{\bOmega} \sim \mathcal{N}(0, \norm{\bw_{\cJ \setminus \{i_0, j\}}}_2^2+ h_n^2 + w_j^2  := \sigma_{Y_1}^2)$. Plugging $X_1, Y_1$ into the term $I_3$, we have
\begin{align}
    I_3 = \mathbb{E}_{\bz}\sbra{ \tanh^2\sbra{\frac{X_1}{\mu} } - \tanh^2\sbra{\frac{Y_1}{\mu} }  } \geq 0
\end{align}
conditioned on any support $\bar{\bOmega}\setminus \{i_0, j\}$, since $\sigma_{X_1}^2 \geq \sigma_{Y_1}^2$ and the function $\tanh^2(x)$ is monotonically increasing with respect to $\abs{x}$.
Similarly, we have $I_4 \geq 0$ as well. In view of $I_3 , I_4 \geq 0$ and \eqref{equ:Q_jQ_i}, we have \eqref{eq:intermediate_claim}.

\subsection{Proof of Proposition~\ref{pro:pointwise_large_gradient}}\label{proof:pointwise_large_gradient}
The directional gradient can be written as a sum of $p$ i.i.d. random variables as following:  
\begin{equation*} 
    \frac{\bw^\top\nabla_{\bw} \phi_o(\bw)}{\norm{\bw}_2} :=\frac{1}{p}\sum_{i=1}^p X_i, \quad \mbox{where}\quad  X_i=    \frac{\bw^\top\nabla_{\bw}\psi_{\mu}\sbra{\cC(\bx_i)\bh(\bw)}}{\norm{\bw}_2}.
\end{equation*}
In order to apply the Bernstein's inequality in Lemma \ref{lemma:matrix_bernstein}, we turn to bound the moments of $X_i$. Plugging in \eqref{equ:gradient_psi}, we have
\begin{align} 
        X_i   & =  \frac{\bw^{\top} \bJ_{\bh}(\bw)}{\|\bw\|_2 }  \mathcal{C}(\bx_i)^{\top} \tanh\sbra{\frac{\mathcal{C}(\bx_i)\bh(\bw)}{\mu}} \nonumber  \\
    & =  \begin{bmatrix}
        \frac{\bw^{\top}}{\norm{\bw}_2}     & -\frac{\norm{\bw}_2}{h_n(\bw)}
        \end{bmatrix}   \cC(\bx_i)^{\top} \tanh\sbra{\frac{\cC(\bx_i) \bh(\bw)}{\mu}} \leq \sqrt{2} n \norm{\cC(\bx_i) }  ,
\end{align}
where the last inequality follows from $\abs{\tanh(\cdot)}\leq 1$ and $\norm{\begin{bmatrix}
        \frac{\bw}{\norm{\bw}_2}   & -\frac{\norm{\bw}_2}{h_n(\bw)}
        \end{bmatrix}}_2 =  \sqrt{1+\frac{\|\bw\|_2^2}{h_n^2}}\leq \sqrt{ 1+n}\leq \sqrt{2n}$, since $\|\bw\|_2^2\leq \frac{n-1}{n}$ and $h_n = \sqrt{1- \|\bw\|_2^2}\geq \frac{1}{\sqrt{n}}$. Invoking Lemma \ref{lemma:Operator norm of Circulant Matrix}, we have for any $m\geq2$,
\begin{equation}
    \mathbb{E}\abs{X_i}^m \leq (\sqrt{2} n)^m \mathbb{E} \norm{\cC(\bx_i)}^m \leq  \frac{m!}{2}\cdot \sbra{C n^3 \log n}^{m/2}
\end{equation}
for some constant $C$. Finally, using \eqref{equ:communication of gradient}, we complete the proof by setting $\sigma^2=C n^3 \log n$, $R=\sqrt{C n^3 \log n}$ and applying the Bernstein's inequality in Lemma \ref{lemma:matrix_bernstein}.

\subsection{Proof of Proposition~\ref{pro:pointwise_strong_convexity}}\label{proof:pointwise_strong_convexity}

The Hessian of $\phi_o(\bw)$ can be written as a sum of $p$ i.i.d. random matrices as following:  
    \begin{equation*}
        \nabla_{\bw}^2 \phi_o(\bw):=\frac{1}{p}\sum_{i=1}^p \bY_i, \quad \mbox{where},\quad \bY_i=  \nabla^2_{\bw}\psi_{\mu} \sbra{\cC(\bx_i) \bh(\bw)}.
    \end{equation*}
Plugging in (\ref{equ:hessian_psi}), we divide $\bY_i$ into two parts as:
    \begin{align*}
         \bY_i&= \underbrace{\frac{1}{\mu} \bJ_{\bh}(\bw)   \mathcal{C}(\bx_i)^{\top}  \mbra{\bI -\mbox{diag}\left(\tanh^2\sbra{\frac{\mathcal{C}(\bx_i)\bh(\bw)}{\mu}}\right) }  \mathcal{C}(\bx_i)  \bJ_{\bh}(\bw)^{\top}}_{\bD_i} \nonumber \\ 
&\quad\quad\quad\quad  - \underbrace{\frac{1}{h_n}\mathcal{S}_{n-1}(\bx_i)^{\top} \tanh\sbra{\frac{\mathcal{C}(\bx_i)\bh(\bw)}{\mu}}  \bJ_{\bh}(\bw) \bJ_{\bh}(\bw)^{\top} }_{\bE_i}. 
\end{align*}
Therefore, we bound the sums of $\bD_i$ and $\bE_i$ respectively, using the Bernstein's inequality in Lemma \ref{lemma:matrix_bernstein}. 

\paragraph{Bound the concentration of $\bE_i$:} We start by bounding the moments of $\bE_i$. 
Recalling the Jacobian matrix $\bJ_{\bh}(\bw)$ in \eqref{eq:jacobian}, we have 
\begin{equation*} 
\bJ_{\bh}(\bw) \bJ_{\bh}(\bw)^{\top} = \bI + \frac{\bw\bw^{\top}}{h_n^2},
\end{equation*}
and therefore, $\norm{\bJ_{\bh}(\bw) \bJ_{\bh}(\bw)^{\top}} = 1+ \|\bw\|_2^2/h_n^2 \leq 5$ since
 for $\|\bw\|_2 \leq 1/2$, we have  $h_n(\bw)\geq 1/2$. Consequently, by the triangle inequality, 
 $$\| \bE_i \| \leq \frac{1}{h_n} \| \bx_i\|_2 \norm{\tanh\sbra{\frac{\mathcal{C}(\bx_i)\bh(\bw)}{\mu}}}_2  \norm{\bJ_{\bh}(\bw) \bJ_{\bh}(\bw)^{\top}} \leq 10 \sqrt{n} \|\bx_i\|_2.$$
We can bound the moments of $\bE_i$ as
\begin{align*} 
        \mathbb{E} \|\bE_i\|^m & \leq 10^m n^{m/2}\mathbb{E}\|\bx_i\|_2^m  \leq 10^m n^{m/2} \cdot m! n^{m/2} \leq \frac{m!}{2}(20n)^2\cdot (20n)^{m-2} ,
\end{align*}
where the second line follows from Fact \ref{fact:BG to gaussian} and Fact \ref{fact: Moments bound for chi} that bound the moments of $\|\bx_i\|_2$. 

Setting $\sigma^2= 400n^2$, $R=20n$, we apply the Bernstein's inequality in Lemma \ref{lemma:matrix_bernstein} and obtain:
\begin{equation}\label{equ:E_i}
    \mathbb{P}\sbra{\norm{\frac{1}{p}\sum_{i=1}^p \bE_i-\mathbb{E}\sbra{\frac{1}{p}\sum_{i=1}^p \bE_i}}\geq \frac{t}{2}}\leq 
    2n \exp\sbra{\frac{-pt^2}{c_1n^2+c_2nt}}
\end{equation}
for some large enough constants $c_1$ and $c_2$.

\paragraph{Bound the concentration of $\bD_i$:}  
Using the fact that $1-\tanh^2\sbra{\cdot} \leq 1$, the spectral norm of $\bD_i$ can be bounded as
\begin{align*}
 \norm{\bD_i} & \leq \frac{1}{\mu} \norm{\mathcal{C}(\bx_i)}^2 \norm{\bJ_{\bh}(\bw)}^2  \leq \frac{5}{\mu} \norm{\mathcal{C}(\bx_i)}^2,
\end{align*}
where we have used again $\norm{\bJ_{\bh}(\bw)}^2 = \norm{\bJ_{\bh}(\bw) \bJ_{\bh}(\bw)^{\top}} \leq 5$ derived above. Invoking Lemma \ref{lemma:Operator norm of Circulant Matrix}, we obtain
    \begin{equation}
        \mathbb{E} \norm{\bD_i}^m \leq \sbra{\frac{5}{\mu}}^m \mathbb{E}\norm{\mathcal{C}(\bx_i)}^{2m}  \leq \frac{m!}{2} \sbra{\frac{Cn\log n}{\mu}}^m,
    \end{equation}
for some constant $C$. Let $\sigma^2= \frac{C^2n^2 \log^2 n}{\mu^2}$, $R=\frac{Cn \log n}{\mu}$, by the Bernstein's inequality in Lemma \ref{lemma:matrix_bernstein}, we have:
\begin{equation}\label{equ:D_i}
    \mathbb{P}\sbra{\norm{\frac{1}{p}\sum_{i=1}^p \bD_i-\mathbb{E}\sbra{\frac{1}{p}\sum_{i=1}^p \bD_i} }\geq \frac{t}{2}}\leq 
    2n \exp\sbra{\frac{-p\mu^2 t^2}{c_3 n^2 \log^2 n+ c_4\mu nt  \log n }}.
\end{equation}
for some constants $c_3$ and  $c_4$.

Recall the Hessian of interest is written as:
\begin{equation}
    \nabla_{\bw}^2 \phi_o(\bw)=\frac{1}{p}\sum_{i=1}^p \bY_i=\frac{1}{p}\sum_{i=1}^p \bD_i -  \frac{1}{p}\sum_{i=1}^p \bE_i.
\end{equation}
Combining the bounds for $\bD_i$ (cf. \eqref{equ:D_i}) and $\bE_i$ (cf. \eqref{equ:E_i}), and observing $ \nabla_{\bw}^2  \mathbb{E}\phi_o(\bw) = \mathbb{E}\nabla_{\bw}^2  \phi_o(\bw) $ from \eqref{equ:communication of hessian}, we obtain the final bound as advertised:
\begin{equation*}
    \mathbb{P}\sbra{\norm{\nabla_{\bw}^2 \phi_o(\bw)- \nabla_{\bw}^2  \mathbb{E}\phi_o(\bw)}\geq t} \leq 4n \exp\sbra{\frac{-p\mu^2 t^2}{ c_5 n^2 \log^2 n + c_6 \mu n \log (n) t}}.
\end{equation*}

\subsection{Proof of Theorem \ref{thm:orthogonal_geometry}}\label{proof:orthogonal_geometry_discretization}

We start by introducing the event 
$$\mathcal{A}_0 := \left \{ \norm{\bX}_\infty \leq 4\sqrt{\log(np)}\right \},$$ 
which holds with probability at least
$ 1- \theta (np)^{-7}$
by Fact \ref{fact:maximum of X}.

\subsubsection{Proof of \eqref{res:lg_geometry}}\label{proof:uniform_gradient}

To show that $\frac{\bw^\top\nabla_{\bw} \phi_o(\bw)}{\norm{\bw}_2}$ is lower bounded uniformly in the region $\mathcal{Q}_1$, we will apply a standard covering argument. Let $\mathcal{N}_1$ be an $\epsilon$-net of $\mathcal{Q}_1$, such that for any $\bw\in \mathcal{Q}_1$, there exists $\bw_1\in \mathcal{N}_1$ with $\norm{\bw-\bw_1}_2\leq \epsilon$. By standard results \cite[Lemma 5.7]{wainwright2019high}, the size of $\mathcal{N}_1$ is at most  $\lceil 3/\epsilon \rceil^n$, where the value of $\epsilon$ will be determined later. We have 
\begin{align*}
    \frac{\bw^\top\nabla_{\bw} \phi_o(\bw)}{\norm{\bw}_2} &= \mbra{\frac{\bw^\top\nabla_{\bw} \phi_o(\bw)}{\norm{\bw}_2}-\frac{\bw_1^\top\nabla_{\bw} \phi_o(\bw_1)}{\norm{\bw_1}_2}} 
    +  \mbra{\frac{\bw_1^\top\nabla_{\bw} \phi_o(\bw_1)}{\norm{\bw_1}_2}- \frac{\bw_1^\top \mathbb{E}\nabla_{\bw}\phi_o(\bw_1) }{\norm{\bw_1}_2} }  
    +  \frac{\bw_1^\top \mathbb{E}\nabla_{\bw}\phi_o(\bw_1) }{\norm{\bw_1}_2}    \\
    &\geq \underbrace{\frac{\bw_1^\top \mathbb{E}\nabla_{\bw}\phi_o(\bw_1) }{\norm{\bw_1}_2}}_{\text{I}}  - \underbrace{\abs{\frac{\bw^\top\nabla_{\bw} \phi_o(\bw)}{\norm{\bw}_2}-\frac{\bw_1^\top\nabla_{\bw} \phi_o(\bw_1)}{\norm{\bw_1}_2}} }_{\text{II}}
    - \underbrace{\abs{\frac{\bw_1^\top\nabla_{\bw} \phi_o(\bw_1)}{\norm{\bw_1}_2}- \frac{\bw_1^\top \mathbb{E}\nabla_{\bw}\phi_o(\bw_1) }{\norm{\bw_1}_2}} }_{\text{III}}.
\end{align*}

In the sequel, we derive bounds for the terms I, II, III respectively.
\begin{itemize}
\item For term I, as $\bw_1\in\mathcal{N}_1 \subseteq \mathcal{Q}_1$, by Lemma \ref{thm:population_geometry_orthogonal}, we have
$$\text{I} = \frac{\bw_1^\top \mathbb{E}\nabla_{\bw}\phi_o(\bw_1) }{\norm{\bw_1}_2}\geq \frac{\theta}{480\sqrt{10\pi}}\xi_0 :=c_1 \theta \xi_0. $$

\item To bound term II, by the additivity of Lipschitz constants and \cite[Proposition 13]{sun2017complete}, we have $\frac{\bw^\top\nabla_{\bw} \phi_o(\bw)}{\norm{\bw}_2}$ is $L_1$-Lipschitz with 
\begin{equation*}
		L_1\leq \sbra{\frac{8\sqrt{2} n^{3/2}}{\mu}+8 n^{5/2}}\norm{\bX}_\infty+ \frac{4n^3}{\mu} \norm{\bX}_\infty^2.
\end{equation*}
Therefore, under the event $\mathcal{A}_0$, we have
$L_1\leq \frac{c_2 n^3}{\mu}\log(np)$
for some constant $c_2$. Setting $\epsilon=\frac{c_1\theta \xi_0}{3L_1}$, we obtain that  
$$\text{II} = \abs{\frac{\bw^\top\nabla_{\bw} \phi_o(\bw)}{\norm{\bw}_2}-\frac{\bw_1^\top\nabla_{\bw} \phi_o(\bw_1)}{\norm{\bw_1}_2}}
 \leq L_1 \norm{\bw-\bw_1}_2\leq L_1 \epsilon \leq \frac{c_1\theta \xi_0}{3}.$$ 
Along the way, we determine the size of $\mathcal{N}_1$ is upper bounded by
$$ \abs{\mathcal{N}_1} \leq \lceil 3/\epsilon \rceil^n\leq  \exp \left\{ n \log\sbra{\frac{c_3 n^3 \log(np)}{\mu \theta \xi_0}  } \right\}.$$

\item For term III, by setting $t=\frac{c_1\theta \xi_0}{3}$  in Proposition~\ref{pro:pointwise_large_gradient} and the union bound, we have the event
\begin{equation*}
 \mathcal{A}_1 : =   \left\{ \max_{\bw_1\in \mathcal{N}_1}  \abs{\frac{\bw_1^\top\nabla_{\bw} \phi_o(\bw_1)}{\norm{\bw_1}_2}-\frac{\bw_1^\top \mathbb{E}\nabla_{\bw}\phi_o(\bw_1) }{\norm{\bw_1}_2} }\leq  \frac{c_1\theta \xi_0}{3}         \right \}
\end{equation*}
holds with probability at least
\begin{align*}
1 -  \abs{\mathcal{N}_1}\cdot 2\exp\sbra{\frac{-pt^2}{C_1 n^3\log n +C_2 \sqrt{n^3 \log (n)}t } }&\geq 1-  2\exp\sbra{\frac{-c_4p \theta^2 \xi_0^2}{n^3\log n }+n \log\sbra{\frac{c_3 n^3 \log(np)}{\mu \theta \xi_0}  }}\\
& \geq 1- 2\exp\sbra{-c_5n  } ,
\end{align*}
provided $p\geq \frac{Cn^4  }{\theta^2 \xi_0^2} \log n \log\sbra{ \frac{ n^3 \log p}{\mu\theta \xi_0} }$ for some sufficiently large $C$ and $n$ is sufficiently large.
\end{itemize}
Combining terms, conditioned on $\mathcal{A}_0 \bigcap \mathcal{A}_1$, which holds with probability at least $1-\theta(np)^{-7} -2\exp\sbra{-c_5 n}$, we have that for all $\bw\in\mathcal{Q}_1$, \eqref{res:lg_geometry} holds since,
\begin{align*} 
\frac{\bw^\top\nabla_{\bw} \phi_o(\bw)}{\norm{\bw}_2}&\geq \text{I} - \text{II} -\text{III} 
    \geq -\frac{c_1\theta \xi_0}{3}-\frac{c_1\theta \xi_0}{3}+ c_1\theta \xi_0 =\frac{c_1\theta \xi_0}{3}. 
\end{align*}

\subsubsection{Proof of (\ref{res:sc_geometry}) } 
 
The proof is similar to the above proof of \eqref{res:lg_geometry} in Appendix \ref{proof:uniform_gradient}.  Let $\mathcal{N}_2$ be an $\epsilon$-net of $\mathcal{Q}_2$, such that for any $\bw\in \mathcal{Q}_2$, there exists $\bw_2 \in \mathcal{N}_2$ with $\norm{\bw-\bw_2}_2\leq \epsilon$. By standard results \cite[Lemma 5.7]{wainwright2019high}, the size of $\mathcal{N}_2$ is at most  $\lceil 3\mu/(4\sqrt{2}\epsilon) \rceil^n $, where the value of $\epsilon$ will be determined later. By the triangle inequality, we have for all $\bw\in\mathcal{Q}_2$,
 \begin{align*}
 \nabla_{\bw}^2\phi_o(\bw)& \succeq \underbrace{\inf_{\bw_2 \in \mathcal{N}_2} \nabla_{\bw}^2\mathbb{E}\sbra{ \phi_o(\bw_2)} }_{\bH_1} - \underbrace{\norm{\nabla_{\bw}^2\phi_o(\bw_2)-\nabla_{\bw}^2\phi_o(\bw)}\bI }_{\bH_2}-\underbrace{\norm{\nabla_{\bw}^2\phi_o(\bw_2)-\nabla_{\bw}^2\mathbb{E}\sbra{ \phi_o(\bw_2)}} \bI }_{\bH_3}.
   \end{align*}

In the sequel, we derive bounds for the terms $\bH_1$, $\bH_2$, $\bH_3$ respectively.

\begin{itemize}
\item For $\bH_1$, by Theorem \ref{thm:population_geometry_orthogonal}, we have   
$$\bH_1 =\inf_{\bw_2 \in \mathcal{N}_2} \nabla_{\bw}^2\mathbb{E}\sbra{ \phi_o(\bw_2)}\succeq \frac{n\theta}{5\sqrt{2\pi}\mu} \bI :=\frac{c_5 n\theta}{\mu} \bI. $$

\item To bound $\bH_2$, by the additivity of Lipschitz constants and \cite[Proposition 14]{sun2017complete}, we have $\nabla_{\bw}^2 \phi_o(\bw)$ is $L_2$-Lipschitz with 
\begin{equation*}
		L_2\leq \frac{4n^3}{\mu^2}\norm{\bX}^3_\infty+\sbra{\frac{4n^2}{\mu}+\frac{8\sqrt{2}n^{3/2}}{\mu}}\norm{\bX}_\infty^2+ 8n \norm{\bX}_\infty.
	\end{equation*}
Under the event $\mathcal{A}_0$, we have $L_2\leq  \frac{c_6 n^3}{\mu^2}\log^{3/2}(np)$ for some constant $c_6$. Setting $\epsilon=\frac{c_5 n\theta }{3 \mu L_2}$, we obtain 
$$\norm{\nabla_{\bw}^2\phi_o(\bw_2)-\nabla_{\bw}^2\phi_o(\bw)}\leq \frac{c_5 n\theta}{3\mu}, \quad \mbox{and}\quad \bH_2 \preceq \frac{c_5 n\theta}{3\mu} \bI. $$
Along the way, we determine the size of $\mathcal{N}_2$ is upper bounded by
$$ \abs{\mathcal{N}_2} \leq \lceil 3\mu/(4\sqrt{2}\epsilon) \rceil^n \leq \exp \mbra{ n \log\sbra{\frac{c_7 n^2 \log^{3/2}(np)}{\theta}  }}.$$

\item To bound $\bH_3$, by setting $t=\frac{c_5 n\theta }{3\mu}$ in Proposition \ref{pro:pointwise_strong_convexity} (as $\|\bw \|_2\leq 1/2$ for $\bw\in\mathcal{Q}_2$ when $\mu<1$) and the union bound, we have $\bH_3 \preceq \frac{c_5 n\theta}{3\mu} \bI $ under the event
\begin{equation*}
  \mathcal{A}_2:=  \left\{ \max_{\bw_2\in \mathcal{N}_2} \norm{\nabla_{\bw}^2\phi_o(\bw_2)- \nabla_{\bw}^2\mathbb{E}\sbra{ \phi_o(\bw_2)}}\leq  \frac{c_5 n\theta }{3\mu}         \right \}
\end{equation*}
holds with probability at least
\begin{align*} 
1- \abs{\mathcal{N}_2}\cdot 4n \exp\sbra{\frac{-p\mu^2 t^2}{C_3 n^2 \log^2 n + C_4 \mu t n \log n }} & \geq 1- 4n \exp\sbra{- \frac{c_8 p \theta^2}{ \log^2n}+n \log\sbra{\frac{c_7 n^2 \log^{3/2}(np)}{ \theta}  }}  \\
&\geq 1-\exp(-c_9n), 
\end{align*}
provided $p \geq \frac{Cn}{\theta^2} \log^2 n \log \sbra{\frac{ n^2 \log^{3/2}p}  {\theta } }$ for some sufficiently large $C$ and $n$ is sufficiently large.
\end{itemize}

Combining terms, conditioned on $\mathcal{A}_0 \bigcap \mathcal{A}_2$, which holds with probability at least $1-\theta(np)^{-7} -\exp(-c_9n)$, we have 
 (\ref{res:sc_geometry}) holds since,
 \begin{align*}
 \nabla_{\bw}^2\phi_o(\bw)& \succeq \bH_1 -  \bH_2 -\bH_3  \succeq \frac{c_5 n\theta }{3\mu} \bI.
    \end{align*}

\subsubsection{Proof of (\ref{res:nearbasis_geometry})}\label{proof:orthogonal_near0}

The characterized geometry of $\phi_o(\bw)$ implies that it has at most one local minimum in $\mathcal{Q}_2$ due to strong convexity, which is denoted as $\bw_o^{\star}$. We are going to show that $\bw_o^{\star}$ is close to $\bm{0}$ in $\mathcal{Q}_2$. By the optimality of $\bw_o^{\star}$ and the mean value theorem, we have for some $t\in(0,1)$:
\begin{align*}
   \phi_o(\bm{0})\geq  \phi_o(\bw_o^{\star}) & \geq \phi_o(\bm{0}) + \left \langle \nabla_{\bw}\phi_o(\bm{0}), \bw_o^{\star} \right \rangle +\bw_o^{\star\top}\nabla^2 \phi_o(t \bw_o^{\star})\bw_o^{\star} \\
   &  \geq  \phi_o(\bm{0}) - \norm{\bw_o^{\star}}_2 \norm{\nabla_{\bw}\phi_o(\bm{0})}_2+\frac{c_5 n\theta}{2\mu} \norm{\bw_o^{\star}}_2^2 ,
\end{align*}
where the second line follows from (\ref{res:sc_geometry}) and the Cauchy-Schwartz inequality.  
Therefore, we have
\begin{equation}\label{eq:bound_w2}
   \norm{\bw_o^{\star}}_2\leq \frac{2\mu}{c_5 n\theta} \norm{\nabla_{\bw}\phi_o(\bm{0})}_2.
\end{equation}
It remains to bound $\norm{\nabla_{\bw}\phi_o(\bm{0})}_2$, which we resort to the Bernstein's inequality in Lemma~\ref{lemma:matrix_bernstein}. As
$\nabla_{\bw}\phi_o(\bm{0})=\frac{1}{p}\sum_{i=1}^p \nabla_{\bw} \psi_{\mu}\sbra{\cC(\bx_i)\bh(\bm{0})}$, where it is straightforward to check $ \mathbb{E} \nabla_{\bw} \psi_{\mu}\sbra{\cC(\bx_i)\bh(\bm{0})} =\bm{0}$ due to symmetry. We turn to bound the moments of $\norm{\nabla_{\bw} \psi_{\mu}\sbra{\cC(\bx_i)\bh(\bm{0})}}_2$ as follows,  
    \begin{align*}
    \norm{ \nabla_{\bw} \psi_{\mu}\sbra{\cC(\bx_i)\bh(\bm{0})}}_2
    &= \norm{ \bJ_{\bh}(\bm{0})  \mathcal{C}(\bx_i)^{\top} \tanh\sbra{\frac{\mathcal{C}(\bx)\bh(\bm{0})}{\mu}}  }_2 \nonumber\\
    &\leq \norm{\bJ_{\bh}(\bm{0})} \norm{\cC(\bx_i) }\norm{\tanh\sbra{\frac{\mathcal{C}(\bx)\bh(\bm{0})}{\mu}} }_2\\
    &\leq \sqrt{n} \norm{\cC(\bx_i) },
    \end{align*}
where the last inequality follows from $\norm{ \bJ_{\bh}(\bm{0}) }=\norm{ \begin{bmatrix}
\bI_{n-1} & \bm{0} \end{bmatrix} }=1$ and $\abs{\tanh\sbra{\cdot}}\leq 1$. Invoking Lemma \ref{lemma:Operator norm of Circulant Matrix}, we have for all $m \geq 2$,
\begin{equation*}
    \mbE{\norm{\nabla_{\bw} \psi_{\mu}\sbra{\cC(\bx)\bh(\bm{0})}}_2^m}\leq \sbra{\sqrt{n}}^m \mathbb{E} \norm{\cC(\bx_i)}^m \leq  \frac{m!}{2}\cdot \sbra{C n^2 \log(n)}^{m/2}
\end{equation*}
for some constant $C$. Setting $\sigma^2=C n^2 \log(n)$, $R=\sqrt{C n^2 \log(n)}$ in the Bernstein's inequality in Lemma~\ref{lemma:matrix_bernstein}, we have  
\begin{equation*}
        \mbP{ \norm{ \nabla_{\bw}\phi_o(\bm{0})}_2 \geq t} \leq 2(n+1)\exp \sbra{\frac{-pt^2}{2Cn^2 \log(n)+2\sqrt{C n^2 \log(n)} t}}.
    \end{equation*}
Let $t=c_9\sqrt{\frac{n^2 \log (n) \log(p)}{p}}$, we have 
\begin{equation} \label{eq:bound_nabla_oh}
    \norm{\nabla_{\bw}\phi_o(\bm{0})}_2\leq c_9\sqrt{\frac{n^2 \log n\log p}{p}}
\end{equation}
with probability at least $1-4n p^{-7}$ when $p\geq c_{10}n\log(n)$. Under the sample size requirement on $p$, we have  
\begin{equation*}
    \norm{\bw_o^{\star}-\bm{0} }_2\leq  \frac{c_6 \mu}{\theta} \sqrt{\frac{\log n \log p}{p}} \leq \frac{\mu}{10},
\end{equation*}
for some constant $c_6$, which ensures $\bw_o^{\star} \in\mathcal{Q}_2$.

\section{Proofs for Section~\ref{sec:geometry_general_outline}}

\subsection{Proof of Lemma \ref{lemma:deviation op norm} }\label{appen: deviation op norm}
Recalling $\bDelta=\sbra{\bU'-\bU} \bU^{-1}$, we have
\begin{equation}
    \norm{\bDelta} =\norm{\sbra{\bU'-\bU}  \bU^{-1}} =\norm{\bU'-\bU},
\end{equation}
since $\bU$ is an orthonormal matrix, i.e., $\norm{\bU^{-1}}=1$. Therefore, it is sufficient to bound $\norm{\bU'-\bU}$ instead. Plugging in the definition of $\bU'$ and $\bU$, we have 
\begin{align}\label{equ:deviation_norm}
        \norm{\bU'-\bU}&=\norm{\cC(\bg)\bR- \cC(\bg)\sbra{\cC(\bg)^\top \cC(\bg)}^{-1/2}} \nonumber \\
        &\leq \norm{\cC(\bg)} \norm{\bR-\sbra{\cC(\bg)^\top \cC(\bg)}^{-1/2}} \nonumber \\
        & \leq   \norm{\cC(\bg)} \frac{\norm{\bR^{2} - \sbra{\cC(\bg)^\top \cC(\bg)}^{-1}}} {\sigma_{\min}\sbra{\sbra{\cC(\bg)^\top \cC(\bg)}^{-1/2}} } \nonumber \\
        & \leq\norm{\cC(\bg)}^2  \| \sbra{\cC(\bg)^\top \cC(\bg)}^{-1}\| \norm{ {\cC(\bg)^\top \cC(\bg)}\bR^{2} -\bI}  \nonumber \\
        &= \kappa^2 \norm{ {\cC(\bg)^\top \cC(\bg)}\bR^{2} -\bI} ,
\end{align}
where the second inequality follows from the fact \cite[Theorem 6.2]{higham2008functions} that for two positive matrices $\bU,\bV$, we have
$\norm{\bU^{-1/2}-\bV^{-1/2}} \leq \frac{\norm{\bU^{-1}-\bV^{-1}}}{\sigma_{\min}(\bV^{-1/2})}$. We continue by plugging in the definition of $\bR$, 
\begin{align}        
\norm{ {\cC(\bg)^\top \cC(\bg)}\bR^{2} -\bI}     &=  \norm{\sbra{\cC(\bg)^\top \cC(\bg)}\cdot \sbra{\frac{1}{\theta np}\sum_{i=1}^p \sbra{\mathcal{C}(\bg)^\top\mathcal{C}(\bx_i)^\top \mathcal{C}(\bx_i) \mathcal{C}(\bg)}}^{-1} - \bI} \nonumber \\
        &=  \norm{\left[\bI + \sbra{\cC(\bg)^\top \mbra{\frac{1}{\theta np}\sum_{i=1}^p \cC(\bx_i)^\top \cC(\bx_i)-\bI} \cC(\bg)} \cdot    \sbra{\cC(\bg)^\top \cC(\bg)}^{-1} \right]^{-1}  -\bI} \nonumber\\
  &  : = \norm{\sbra{\bI+\bA}^{-1}-\bI}. 
\end{align}        
where $\bA = \sbra{\cC(\bg)^\top \mbra{\frac{1}{\theta np}\sum_{i=1}^p \cC(\bx_i)^\top \cC(\bx_i)-\bI} \cC(\bg)} \cdot    \sbra{\cC(\bg)^\top \cC(\bg)}^{-1}$. 

By Lemma~\ref{lemma: other1}, we have when $p\geq C n \log(n)$, 
$$\norm{\frac{1}{\theta np}\sum_{i=1}^p \cC(\bx_i)^\top \cC(\bx_i)-\bI}\leq C\sqrt{\frac{\log^2n\log p}{\theta^2 p}}  $$ 
with probability at least $1-2n p^{-8}$, and $\|\bA\| \leq C \kappa^2 \sqrt{\frac{\log^2(n)\log(p)}{\theta^2 p}} $. Then as long as $\|\bA\| \leq 1/2$, which holds when $p\geq  \frac{C_2 \kappa^4  \log^2 (n) \log p }{\theta^2}$ for some large enough constant $C_2$, we have
$$\norm{\sbra{\bI+\bA}^{-1}-\bI}\leq \norm{\sbra{\bI+\bA}^{-1}}\norm{\bA}\leq \frac{\norm{\bA}}{1-\norm{\bA}} \leq 2\|\bA\|.$$ 
 Plugging this back into \eqref{equ:deviation_norm}, we have
\begin{equation*}
    \norm{\bU'-\bU}\leq C_3 \kappa^4 \sqrt{\frac{\log^2n \log p}{\theta^2 p}}.
\end{equation*}

\subsection{Proof of Lemma \ref{lemma: deviation gd and he} }\label{appen: deviation gd and he}
We first record some useful facts. For any $\bh\in\mathcal{S}_0^{(n+)}$, we have the Jacobian matrix $\bJ_{\bh}(\bw)= \mbra{\bI, -\frac{\bw}{h_n}} \in \mathbb{R}^{(n-1)\times n}$ satisfies
\begin{equation}\label{eq:jacobian_bound}
    \norm{\bJ_{\bh}(\bw) }\leq \norm{\bJ_{\bh}(\bw)}_{\mathrm{F}}\leq \sqrt{n-1+\frac{\norm{\bw}_2^2}{h_n^2}}\leq \sqrt{2n},
    \end{equation}
since $\norm{\bw}_2\leq 1$ and $h_n\geq \frac{1}{\sqrt{n}}$. In addition, by the union bound and Lemma~\ref{lemma:Operator norm of Circulant Matrix}, we have with probability at least $1- (np)^{-8}$,
\begin{equation}\label{eq:circulant_x}
    \max_{i\in[p]}\norm{\cC(\bx_i)}\leq C \sqrt{n\log(np)},
\end{equation}
for some constant $C$.

\subsubsection{Proof of \eqref{eq:perturbation_gradient}}\label{proof:perturbation_gradient}

Similar to (\ref{equ:gradient_psi}), we can write the gradient $\nabla_{\bw} \phi(\bw)$ as 
\begin{align*} 
    \nabla_{\bw} \phi(\bw)    &=\frac{1}{p}\sum_{i=1}^p  \bJ_{\bh}(\bw) \sbra{\bI+\bDelta}^\top    \cC(\bx_i)^\top  \tanh\sbra{\frac{\mathcal{C}(\bx_i) \sbra{\bI+\bDelta} \bh(\bw)}{\mu}}. 
\end{align*}
Recalling the expression of $\nabla_{\bw} \phi_o(\bw)$ in (\ref{equ:gradient_psi}), we write 
\begin{align*}
     &\quad \nabla_{\bw} \phi_o(\bw)-\nabla_{\bw} \phi(\bw) \\
&= \frac{1}{p}\sum_{i=1}^p \bJ_{\bh}(\bw) \cC(\bx_i)^\top   \tanh\sbra{\frac{\mathcal{C}(\bx_i)\bh(\bw)}{\mu}}-\frac{1}{p}\sum_{i=1}^p  \bJ_{\bh}(\bw)  \sbra{\bI+\bDelta}^\top    \cC(\bx_i)^\top   \tanh\sbra{\frac{\mathcal{C}(\bx_i) \sbra{ \bI+\bDelta} \bh(\bw)}{\mu}}  \nonumber\\
&=\underbrace{ \frac{1}{p}\sum_{i=1}^p \bJ_{\bh}(\bw) \cC(\bx_i)^\top  \mbra{\tanh\sbra{\frac{\mathcal{C}(\bx_i)\bh(\bw)}{\mu}}-\tanh\sbra{\frac{\mathcal{C}(\bx_i) \sbra{ \bI+\bDelta} \bh(\bw)}{\mu}} }  }_{\bg_1}   \nonumber\\
        &\quad\quad\quad\quad\quad - \underbrace{ \frac{1}{p}\sum_{i=1}^p \bJ_{\bh}(\bw)  \bDelta^\top \cC(\bx_i)^\top  \tanh\sbra{\frac{\mathcal{C}(\bx_i) \sbra{\bI+\bDelta} \bh(\bw)}{\mu}}  }_{\bg_2}.
\end{align*}

Therefore, we continue to bound $\|\bg_1\|_2$ and $\|\bg_2\|_2$.
\begin{itemize}
\item To bound $\|\bg_1\|_2$, we have
\begin{align}
\| \bg_1 \|_2 & \leq \norm{\bJ_{\bh}(\bw)} \cdot \max_{i\in[p]} \norm{\cC(\bx_i)}\cdot \max_{i\in[p]}\norm{\tanh\sbra{\frac{\mathcal{C}(\bx_i) \bh(\bw)}{\mu}} -\tanh\sbra{\frac{\mathcal{C}(\bx_i) \sbra{\bI+\bDelta} \bh(\bw)}{\mu}}      }_2  \nonumber \\
& \leq  \frac{1}{\mu}\norm{\bJ_{\bh}(\bw)} \cdot \max_{i\in[p]} \norm{\cC(\bx_i)}^2 \cdot \norm{\bDelta}. \label{eq:bound_g1}
\end{align}
Here, the second line follows from for any $i\in [p]$, 
\begin{align}
&\quad \norm{\tanh\sbra{\frac{\mathcal{C}(\bx_i) \bh(\bw)}{\mu}} -\tanh\sbra{\frac{\mathcal{C}(\bx_i) \sbra{\bI+\bDelta} \bh(\bw)}{\mu}}      }_2 \nonumber \\
&\leq\norm{\sbra{ \frac{\cC(\bx_i) \bh}{\mu} }-\sbra{ \frac{\cC(\bx_i) (\bI+\bDelta)\bh}{\mu} } }_2  \nonumber \\
& = \frac{1}{\mu} \norm{\cC(\bx_i) \bDelta \bh}_2 \leq \frac{1}{\mu}  \norm{\cC(\bx_i) } \norm{\bDelta} \norm{ \bh}_2 = \frac{1}{\mu}  \norm{\cC(\bx_i) } \norm{\bDelta} ,  \label{eq:lipschtiz_tanh_vector}
\end{align}
where the second line follows from Lemma \ref{lemma:loss_lipschitz}, and the last equality is due to $\norm{ \bh}_2 = 1$. 

\item To bound $\|\bg_2\|_2$, we have 
\begin{align}
\| \bg_2 \|_2 & \leq  \norm{\bJ_{\bh}(\bw)} \cdot \max_{i\in[p]}\norm{\cC(\bx_i)}\cdot \max_{i\in[p]}\norm{\tanh\sbra{\frac{\mathcal{C}(\bx_i) \sbra{\bI+\bDelta} \bh(\bw)}{\mu}}    }_2 \cdot \norm{\bDelta} \nonumber \\
& \leq \sqrt{n}  \norm{\bJ_{\bh}(\bw)} \cdot \max_{i\in[p]}\norm{\cC(\bx_i)}\cdot \norm{\bDelta}, \label{eq:bound_g2}
\end{align}
where the second line uses $|\tanh(\cdot)|\leq 1$, and $\norm{\tanh\sbra{\frac{\mathcal{C}(\bx_i) \sbra{\bI+\bDelta} \bh(\bw)}{\mu}}    }_2 \leq \sqrt{n}$.
\end{itemize}

Combining \eqref{eq:bound_g1} and \eqref{eq:bound_g2}, we have
\begin{align*}
\norm{\nabla_{\bw} \phi_o(\bw)-\nabla_{\bw} \phi(\bw)}_2 \leq \| \bg_1\|_2 + \|\bg_2\|_2 & \leq \norm{\bJ_{\bh}(\bw)} \cdot \max_{i\in[p]}\norm{\cC(\bx_i)}\cdot \norm{\bDelta} \sbra{ \sqrt{n}   +\frac{1}{\mu} \max_{i\in[p]}\norm{\cC(\bx_i )}  }  \\
& \leq C \frac{n^{3/2}\log(np) }{\mu} \norm{\bDelta},
\end{align*}
for some constant $C$, where the last line follows from \eqref{eq:jacobian_bound} and \eqref{eq:circulant_x}, which holds with probability at least $1-(np)^{-8}$.

\subsubsection{Proof of \eqref{eq:perturbation_hessian}}

First, under the sample size $p\geq  \frac{C_2 \kappa^8 n \log^2 (n) \log p }{\theta^2}$, from Lemma~\ref{lemma:deviation op norm}, we can ensure $\|\bDelta\| \leq 1$. Note that
\begin{align}
\norm{\nabla^2_{\bw} \phi_o(\bw)-\nabla^2_{\bw} \phi(\bw)}      &=\norm{\frac{1}{p}\sum_{i=1}^p \nabla^2_{\bw}\psi_{\mu}(\mathcal{C}(\bx_i)\bh(\bw))-\frac{1}{p}\sum_{i=1}^p \nabla^2_{\bw}\psi_{\mu}(\mathcal{C}(\bx_i) \sbra{\bI+\bDelta} \bh(\bw))      } \nonumber\\
& \leq \frac{1}{p}\sum_{i=1}^p \norm{\nabla^2_{\bw}\psi_{\mu}(\mathcal{C}(\bx_i)\bh(\bw))- \nabla^2_{\bw}\psi_{\mu}(\mathcal{C}(\bx_i) \sbra{\bI+\bDelta} \bh(\bw))      } . \label{eq:desired}
\end{align}         
       
Similar to (\ref{equ:hessian_psi}), we can write the Hessian $\nabla^2_{\bw}\psi_{\mu}(\mathcal{C}(\bx_i) \sbra{\bI+\bDelta} \bh(\bw))  $ as
\begin{align}
&\quad \nabla^2_{\bw} \psi_{\mu}(\mathcal{C}(\bx_i) \sbra{\bI+\bDelta} \bh(\bw)) \nonumber \\
& =  \frac{1}{\mu} \bJ_{\bh}(\bw) \sbra{\bI+\bDelta}  \mathcal{C}(\bx_i)^{\top}  \mbra{\bI -\mbox{diag}\left(\tanh^2\sbra{\frac{\mathcal{C}(\bx) \sbra{\bI+\bDelta} \bh(\bw)}{\mu}}\right) }  \mathcal{C}(\bx_i)  \sbra{\bI+\bDelta} \bJ_{\bh}(\bw)^{\top} \nonumber \\ 
&\quad\quad\quad\quad -  \frac{1}{h_n}\mathcal{S}_{n-1}(\bx)^{\top} \sbra{\bI+\bDelta}\tanh\sbra{\frac{\mathcal{C}(\bx)\sbra{\bI+\bDelta}\bh(\bw)}{\mu}}  \bJ_{\bh}(\bw) \bJ_{\bh}(\bw)^{\top} . 
\end{align}
Subtracting $\nabla^2_{\bw}\psi_{\mu}(\mathcal{C}(\bx_i)\bh(\bw))$ in (\ref{equ:hessian_psi}) from the above equation, we have
\begin{align*}
&\quad \nabla^2_{\bw}\psi_{\mu}(\mathcal{C}(\bx_i)\bh(\bw))- \nabla^2_{\bw}\psi_{\mu}(\mathcal{C}(\bx_i) \sbra{\bI+\bDelta} \bh(\bw))    \\
& = \underbrace{\frac{1}{\mu} \bJ_{\bh}(\bw)  \mathcal{C}(\bx_i)^{\top}  \mbra{\mbox{diag}\left(\tanh^2\sbra{\frac{\mathcal{C}(\bx_i) \sbra{\bI+\bDelta} \bh(\bw)}{\mu}} -\tanh^2\sbra{\frac{\mathcal{C}(\bx_i)  \bh(\bw)}{\mu}}\right) }  \mathcal{C}(\bx_i)   \bJ_{\bh}(\bw)^{\top}}_{\bQ_1} \nonumber \\ 
&\quad  \underbrace{ - \frac{1}{\mu} \bJ_{\bh}(\bw)  \bDelta  \mathcal{C}(\bx_i)^{\top}  \mbra{\bI -\mbox{diag}\left(\tanh^2\sbra{\frac{\mathcal{C}(\bx_i) \sbra{\bI+\bDelta} \bh(\bw)}{\mu}}\right) }  \mathcal{C}(\bx_i)  \sbra{\bI+\bDelta} \bJ_{\bh}(\bw)^{\top} }_{\bQ_2} \\
&\quad  \underbrace{ -\frac{1}{\mu} \bJ_{\bh}(\bw)   \mathcal{C}(\bx_i)^{\top}  \mbra{\bI -\mbox{diag}\left(\tanh^2\sbra{\frac{\mathcal{C}(\bx_i) \sbra{\bI+ \bDelta} \bh(\bw)}{\mu}}\right) }  \mathcal{C}(\bx_i)   \bDelta \bJ_{\bh}(\bw)^{\top} }_{\bQ_3} \\
&\quad  \underbrace{ +   \frac{1}{h_n}\mathcal{S}_{n-1}(\bx_i)^{\top} \mbra{ \tanh\sbra{\frac{\mathcal{C}(\bx_i)\sbra{\bI+\bDelta}\bh(\bw)}{\mu}} - \tanh\sbra{\frac{\mathcal{C}(\bx_i) \bh(\bw)}{\mu}} } \bJ_{\bh}(\bw) \bJ_{\bh}(\bw)^{\top}}_{\bQ_4} \\
& \quad \underbrace{ +  \frac{1}{h_n}\mathcal{S}_{n-1}(\bx_i)^{\top}  \bDelta \tanh\sbra{\frac{\mathcal{C}(\bx_i)\sbra{\bI+\bDelta}\bh(\bw)}{\mu}}  \bJ_{\bh}(\bw) \bJ_{\bh}(\bw)^{\top} }_{\bQ_5},
\end{align*}
where in the sequel we'll bound these terms respectively.
\begin{itemize}
\item $\bQ_1$ can be bounded as 
\begin{align*}
\norm{\bQ_1} & \leq \frac{1}{\mu} \norm{\bJ_{\bh}(\bw) }^2 \norm{\mathcal{C}(\bx_i)}^2 \norm{\tanh^2\sbra{\frac{\mathcal{C}(\bx_i) \sbra{\bI+\bDelta} \bh(\bw)}{\mu}} -\tanh^2\sbra{\frac{\mathcal{C}(\bx_i)  \bh(\bw)}{\mu}}}_{\infty} \\
& \leq  \frac{2}{\mu^2} \norm{\bJ_{\bh}(\bw) }^2 \norm{\mathcal{C}(\bx_i)}^2 \norm{\mathcal{C}(\bx_i)  \bDelta \bh(\bw) }_{\infty} \\
&\leq   \frac{2}{\mu^2} \norm{\bJ_{\bh}(\bw) }^2   \norm{\mathcal{C}(\bx_i)}^2 \norm{ \bx_i}_2   \norm{ \bDelta},
\end{align*}
where the second line follows from Lemma~\ref{lemma:loss_lipschitz}, where the last line uses $\|\bh\|_2=1$.
\item $\bQ_2$ can be bounded as 
\begin{align*}
\norm{\bQ_2} & \leq \frac{1}{\mu} \norm{\bJ_{\bh}(\bw) }^2 \norm{\mathcal{C}(\bx_i)}^2  \norm{\bDelta} (1+ \norm{\bDelta} ) \leq \frac{2}{\mu} \norm{\bJ_{\bh}(\bw) }^2 \norm{\mathcal{C}(\bx_i)}^2  \norm{\bDelta}   ,
\end{align*}
where we have used $1-\tanh^2(\cdot)\leq 1$, $\norm{\bDelta}\leq 1$ respectively.
\item Similar to $\bQ_2$, $\bQ_3$ can be bounded as
\begin{align*}
\norm{\bQ_3} & \leq \frac{1}{\mu} \norm{\bJ_{\bh}(\bw) }^2 \norm{\mathcal{C}(\bx_i)}^2  \norm{\bDelta}   .
\end{align*}
\item $\bQ_4$ can be bounded as
\begin{align*}
\norm{\bQ_4} & \leq \frac{1}{h_n}\norm{\bx_i}_2  \norm{\bJ_{\bh}(\bw) }^2 \norm{\tanh\sbra{\frac{\mathcal{C}(\bx_i)\sbra{\bI+\bDelta}\bh(\bw)}{\mu}} - \tanh\sbra{\frac{\mathcal{C}(\bx_i) \bh(\bw)}{\mu}} }_2 \\
& \leq  \frac{\sqrt{n}}{\mu}  \norm{\bJ_{\bh}(\bw) }^2 \norm{\cC(\bx_i) } \norm{\bx_i}_2  \norm{\bDelta},
\end{align*}
where the second line follows from \eqref{eq:lipschtiz_tanh_vector} and $h_n\geq 1/\sqrt{n}$.
\item $\bQ_5$ can be bounded as
\begin{align*}
\norm{\bQ_5} & \leq \frac{1}{h_n}\norm{\bx_i}_2  \norm{\bJ_{\bh}(\bw) }^2 \norm{ \bDelta }\norm{ \tanh\sbra{\frac{\mathcal{C}(\bx_i)\sbra{\bI+\bDelta}\bh(\bw)}{\mu}} }_2 \\
& \leq  n \norm{\bx_i}_2  \norm{\bJ_{\bh}(\bw) }^2\norm{ \bDelta } ,
\end{align*}
where the second line uses $|\tanh(\cdot)|\leq 1$ and $h_n\geq 1/\sqrt{n}$.
\end{itemize}
 Combining the above bounds back into \eqref{eq:desired}, we have
 \begin{align*}
&\quad \norm{\nabla^2_{\bw} \phi_o(\bw)-\nabla^2_{\bw} \phi(\bw)} \\
& \leq  \norm{\bJ_{\bh}(\bw) }^2   \norm{ \bDelta} \max_{i\in [p]} \sbra{    \frac{2}{\mu^2} \norm{\mathcal{C}(\bx_i)}^2 \norm{ \bx_i}_2  +\frac{3}{\mu}   \norm{\mathcal{C}(\bx_i)}^2 + \frac{\sqrt{n}}{\mu}  \norm{\cC(\bx_i) } \norm{\bx_i}_2 +n \norm{\bx_i}_2   }.
 \end{align*}
 Plugging in \eqref{eq:jacobian_bound}, \eqref{eq:circulant_x}, and Fact~\ref{fact:norm of BG}, where with probability at least $1-2p^{-8}$,  
 $$ \max_{i\in[p]}\norm{\cC(\bx_i)}\leq C \sqrt{n\log(np)}, \quad \max_{i\in [p]} \norm{\bx_i}_2 \leq C\sqrt{n\log p},$$
we have   
\begin{align*} 
 \norm{\nabla^2_{\bw} \phi_o(\bw)-\nabla^2_{\bw} \phi(\bw)}
    &\leq  C_9 \frac{n^{5/2}}{\mu^2}\log^{3/2}(np) \norm{\bDelta}.
\end{align*}

\subsection{Proof of Theorem \ref{theo:general_geometry}}
 \label{proof:theorem_general_geometry}

To begin, by Lemma \ref{lemma: deviation gd and he}, we have
\begin{subequations}\label{eq:concentration_delta}
\begin{align} 
        \abs{\frac{\bw^\top\nabla_{\bw} \phi_o(\bw)}{\norm{\bw}_2 }-\frac{\bw^\top\nabla_{\bw} \phi(\bw)}{\norm{\bw}_2} }&\leq \norm{\nabla_{\bw} \phi_o(\bw)-\nabla_{\bw} \phi(\bw)}_2 \leq c_g  \frac{n^{3/2}\log(np) }{\mu} \norm{\bDelta}\leq\frac{c_2 \xi_0 \theta}{2},\\
        \norm{\nabla^2_{\bw} \phi_o(\bw)-\nabla^2_{\bw} \phi(\bw)} &\leq c_h \frac{n^{5/2}\log^{3/2}(np)}{\mu^2} \norm{\bDelta}\leq  \frac{c_2 n\theta}{2\mu} ,
\end{align}
\end{subequations}
as long as the sample size satisfies 
    $$\norm{\bDelta}\leq c_f \kappa^4 \sqrt{\frac{\log^2n \log p}{\theta^2 p}} \leq C   \frac{\xi_0\theta \mu}{n^{3/2} \log^{3/2} (np)} ,$$
for some constant $C$  in view of Lemma \ref{lemma:deviation op norm}. Translating this into the sample size requirement, it means
 $$p\geq C \frac{\kappa^8 n^3 \log^4 p \log^2 n}{\theta^4 \mu^2 \xi_0^2}.  $$

Under the assumption of Theorem~\ref{thm:orthogonal_geometry}, and 
in view of (\ref{res:lg_geometry}) and (\ref{res:sc_geometry}), we have
\begin{align*} 
    \frac{\bw^\top\nabla_{\bw} \phi(\bw)}{\norm{\bw}}&\geq \frac{\bw^\top\nabla_{\bw} \phi_o(\bw)}{\norm{\bw}}-\abs{\frac{\bw^\top\nabla_{\bw} \phi_o(\bw)}{\norm{\bw}}-\frac{\bw^\top\nabla_{\bw} \phi(\bw)}{\norm{\bw}}}\geq \frac{c_2 \xi_0 \theta}{2}, \\
    \nabla^2_{\bw} \phi(\bw)&\succeq \nabla^2_{\bw} \phi_o(\bw)-\norm{\nabla^2_{\bw} \phi_o(\bw)-\nabla^2_{\bw} \phi(\bw)}\bI\succeq \frac{c_2 n\theta}{2\mu} \bI .
    \end{align*}

Now let $\bw^{\star}$ be the local minimizer of $\phi(\bw)$ in the region of interest. Similar to the proof in Appendix~\ref{proof:orthogonal_near0}, we have
\begin{align*}
   \norm{\bw^{\star}}_2 & \leq \frac{4\mu}{c_2 n\theta} \norm{\nabla_{\bw}\phi(\bm{0})}_2 \\
   &  \leq \frac{4\mu}{c_2 n\theta}\norm{\nabla_{\bw}\phi_o(\bm{0})}_2 + \frac{4\mu}{c_2 n\theta}\norm{\nabla_{\bw}\phi(\bm{0})-\nabla_{\bw}\phi_o(\bm{0})}_2 \\
   & \leq  \frac{c_9\mu}{n\theta} \sbra{ \sqrt{\frac{n^2 \log (n) \log(p)}{p}} +  \frac{n^{3/2}\log(np) }{\mu}  \kappa^4 \sqrt{\frac{\log^2(n) \log(p)}{\theta^2 p}} }.
\end{align*}
where the first term is bounded by \eqref{eq:bound_nabla_oh} and the second term is bounded by Lemma~\ref{lemma: deviation gd and he}. Under the sample size requirement, the latter term dominates and therefore we have
\begin{align*}
   \norm{\bw^{\star}-\bm{0} }_2 & \leq \frac{c \kappa^4}{\theta^2} \sqrt{\frac{n\log^3 p \log^2 n}{p}}.
\end{align*}

\section{Proofs for Section~\ref{pipe:GD_orthogonal}}

We start by stating a useful observation. Notice that $\sbra{\frac{\be_k}{h_k}-\frac{\be_n}{h_n}}$ is on the tangent space of $\bh$, i.e., 
$$\sbra{\bI-\bh\bh^\top}\sbra{\frac{\be_k}{h_k}-\frac{\be_n}{h_n}}=\sbra{\frac{\be_k}{h_k}-\frac{\be_n}{h_n}},$$
we have the relation
\begin{align}\label{eq:manifold_g_euclid_g}
\partial f(\bh)^{\top}\sbra{\frac{\be_k}{h_k}-\frac{\be_n}{h_n}} & =\mbra{ \sbra{\bI-\bh\bh^\top}\nabla f(\bh)}^{\top}\sbra{\frac{\be_k}{h_k}-\frac{\be_n}{h_n}}  = \nabla f(\bh)^{\top}\sbra{\frac{\be_k}{h_k}-\frac{\be_n}{h_n}} ,
\end{align}
holds for both $\partial f(\bh)$ and $\partial f_o(\bh)$.

\subsection{Proof of Lemma~\ref{thm:implicit_uniform_concentration_general}}\label{proof:thm:implicit_uniform_concentration_general}

We first prove the upper bound of $\norm{\partial f(\bh)}_2$ in \eqref{eq:bound_grad}, which is simpler. Plugging the bound for $\max_{i\in[p]} \norm{\cC(\bx_i)}$ in \eqref{eq:circulant_x} and $\norm{\bDelta}\leq 1$ ensured by the sample size requirement and Lemma~\ref{lemma:deviation op norm}, for any $\bh$ on the unit sphere, with probability at least $1- (np)^{-8}$, the manifold gradient satisfies
    \begin{align*}
        \norm{\partial f(\bh)}_2\leq \norm{\nabla f(\bh)}_2&=\norm{\frac{1}{p}\sum_{i=1}^p  \sbra{\bI+\bDelta}^\top    \cC(\bx_i)^\top   \tanh\sbra{\frac{\mathcal{C}(\bx_i) \sbra{\bI+\bDelta} \bh(\bw)}{\mu}}  }_2 \\
        &\leq \sqrt{n} \norm{\bI+\bDelta} \max_{i\in[p]} \norm{\cC(\bx_i)}\\
        &\leq 2 C_1n \sqrt{\log(np)}.
    \end{align*} 

We now move to prove the lower bound of the directional gradient in \eqref{res: sins geometry}.
We first consider the directional gradient of $f_o(\bh)$ for the orthogonal case, following the proof procedure of the Theorem \ref{thm:orthogonal_geometry} to obtain the empirical geometry of $f_o(\bh)$ in the region of interest (shown in Lemma~\ref{thm:implicit_uniform_concentration_orthogonal}), which is proved in Appendix~\ref{proof:implicit_uniform_concentration_orthogonal}.
\begin{lemma} [Uniform concentration for the orthogonal case] \label{thm:implicit_uniform_concentration_orthogonal}
Instate the assumptions of Theorem~\ref{thm:orthogonal_geometry}. There exist some constants $c_a, c_b$, such that for 
$\bh\in \mathcal{H}_k =\left\{\bh: \bh\in \mathcal{S}_{\xi_0}^{(n+)}, h_k\neq 0, h_n^2/h_k^2<4\right\}$, with probability at least $1-(np)^{-8} -2\exp\sbra{-c_an }$, 
    \begin{equation}\label{res:sins_geometry}
         \partial f_o(\bh)^\top  \sbra{\frac{\be_k}{h_k}-\frac{\be_n}{h_n}}\geq c_b \xi_0 \theta.
    \end{equation}
\end{lemma}

Based on the result, we derive the bound for the directional gradient of $f(\bh)$ in the general case by bounding the deviation between the directional gradient of $f_o(\bh)$ and $f(\bh)$. Using \eqref{eq:manifold_g_euclid_g}, we can relate the directional gradient of $f(\bh)$ to that of $f_o(\bh)$ as
\begin{equation}
   ( \partial f(\bh)-\partial f_o(\bh))^\top \sbra{\frac{\be_k}{h_k}-\frac{\be_n}{h_n} }= \sbra{\nabla f(\bh)- \nabla f_o(\bh)}^\top \sbra{\frac{\be_k}{h_k}-\frac{\be_n}{h_n} }.
\end{equation}
We have
\begin{align*}
\abs{\sbra{\partial f(\bh)-\partial f_o(\bh)}^\top  \sbra{\frac{\be_k}{h_k}-\frac{\be_n}{h_n} } } & \leq \norm{\nabla f(\bh)- \nabla f_o(\bh)}_2 \norm{\frac{\be_k}{h_k}-\frac{\be_n}{h_n} }_2 \\
& \leq \sqrt{5n}\norm{\nabla f(\bh)- \nabla f_o(\bh)}_2,
\end{align*}
where the last line follows from
\begin{equation} \label{eq:bound_directional_grad}
\norm{\frac{\be_k}{h_k}-\frac{\be_n}{h_n}}_2 = \sqrt{ \frac{1}{h_k^2} + \frac{1}{h_n^2} }\leq \sqrt{\frac{5}{h_n^2}}\leq \sqrt{5n}, 
\end{equation}
due to the assumption $ h_n^2 / h_k^2 \leq 4$ and $h_n\geq 1 / \sqrt{n}$. Therefore, it is sufficient to bound $\norm{\nabla f(\bh)- \nabla f_o(\bh)}_2$.
By \eqref{eq:gradient_phi_h}, we have
\begin{align*}
 \norm{\nabla f_o(\bh)-\nabla f(\bh)}_2   &= \norm{\frac{1}{p}\sum_{i=1}^p  \cC(\bx_i)^\top  \tanh\sbra{\frac{\mathcal{C}(\bx_i)\bh}{\mu}}-\frac{1}{p}\sum_{i=1}^p  \sbra{\bI+\bDelta}^\top    \cC(\bx_i)^\top \tanh\sbra{\frac{\mathcal{C}(\bx_i) \sbra{ \bI+\bDelta} \bh }{\mu}}    }_2\\
        &\leq \norm{\frac{1}{p}\sum_{i=1}^p \bDelta^\top \cC(\bx_i)^\top   \tanh\sbra{\frac{\mathcal{C}(\bx_i) \sbra{\bI+\bDelta} \bh}{\mu}}  }_2 \\
        &\quad + \norm{ \frac{1}{p}\sum_{i=1}^p \cC(\bx_i)^\top   \mbra{\tanh\sbra{\frac{\mathcal{C}(\bx_i)\bh}{\mu}}-\tanh\sbra{\frac{\mathcal{C}(\bx_i) \sbra{\bI+\bDelta} \bh }{\mu}} }           }_2\\
        &\leq   \max_{i\in[p]}\norm{\cC(\bx_i)}\cdot \max_{i\in[p]}\norm{\tanh\sbra{\frac{\mathcal{C}(\bx_i) \sbra{ \bI+\bDelta} \bh }{\mu}}    }_2 \norm{\bDelta}  \\
        &\quad  + \max_{i\in[p]}\norm{\cC(\bx_i)}\cdot \max_{i\in[p]}\norm{\tanh\sbra{\frac{\mathcal{C}(\bx_i) \bh }{\mu}} -\tanh\sbra{\frac{\mathcal{C}(\bx_i) \sbra{\bI+\bDelta} \bh }{\mu}}      }_2\\
        &\leq  \max_{i\in[p]}\norm{\cC(\bx_i)}\cdot \sbra{ \sqrt{n} \norm{\bDelta}  +\frac{1}{\mu} \max_{i\in[p]}\norm{\cC(\bx_i )}\norm{\bDelta}  }\\
        &\leq C_1 \frac{n \log(np) }{\mu} \norm{\bDelta}
    \end{align*}
with probability at least $1-(np)^{-8}$, where the penultimate inequality follows from \eqref{eq:lipschtiz_tanh_vector}, and the last inequality follows from \eqref{eq:circulant_x}. 
By Lemma~\ref{lemma:deviation op norm}, there exists some constant $C$, such that under the sample complexity requirement, we have 
\begin{equation}
 \abs{\sbra{\partial f(\bh)-\partial f_o(\bh)}^\top  \sbra{\frac{\be_k}{h_k}-\frac{\be_n}{h_n} } }\leq   C_{1} \frac{n^{3/2}\log(np) }{\mu} \kappa^4 \sqrt{\frac{\log^2n \log p}{\theta^2 p}}\leq \frac{c_b \xi_0 \theta}{2}.
\end{equation}
In addition, Lemma~\ref{thm:implicit_uniform_concentration_orthogonal} guarantees that $\partial f_o(\bh)^\top \sbra{\frac{\be_k}{h_k}-\frac{\be_n}{h_n}}\geq c_b \xi_0 \theta$. Putting together, we have
\begin{align*}
\partial f(\bh)^\top \sbra{\frac{\be_k}{h_k}-\frac{\be_n}{h_n} } & \geq \partial f_o(\bh)^\top \sbra{\frac{\be_k}{h_k}-\frac{\be_n}{h_n}} -  \abs{\sbra{\partial f(\bh)-\partial f_o(\bh)}^\top \sbra{\frac{\be_k}{h_k}-\frac{\be_n}{h_n} } } \\
&\geq c_b \xi_0 \theta-\frac{c_b \xi_0 \theta}{2}=\frac{c_b \xi_0 \theta}{2}
\end{align*}
with probability at least $1-2(np)^{-8} - 2\exp\sbra{-c_an }$.

\subsection{Proof of Lemma \ref{thm:implicit_stay}}\label{proof:implicit_stay}

Owing to symmetry, without loss of generality, we will show that if the current iterate $\bh\in\cS_{\xi_0}^{(n+)}$ with $\xi_0\in (0,1)$, the next iterate
$$ \bh^{+}  = \frac{\bh - \eta \partial f(\bh)}{\norm{\bh - \eta \partial f(\bh)}_2} $$
stays in $\cS_{\xi_0}^{(n+)}$ for a sufficiently small step size $\eta$. For any $i\in[n-1]$, we have
\begin{align} \label{eq:general_eq}
\sbra{\frac{ {h}^{+}_n}{h^{+}_i} }^2 & = \frac{(h_n -\eta  [ \partial f(\bh)]_n )^2}{(h_i -\eta  [\partial f(\bh)]_i )^2}  = \frac{(1 -\eta  [\partial f(\bh)]_n /h_n )^2}{(h_i/h_n -\eta [ \partial f(\bh)]_i /h_n)^2}.
\end{align}
By \eqref{eq:bound_grad} in Lemma~\ref{thm:implicit_uniform_concentration_general}, which bounds $\norm{\partial f(\bh)}_{\infty} \leq \norm{\partial f(\bh)}_2 \leq C n \sqrt{\log(np)}$ for some constant $C$, and $h_n\geq 1/\sqrt{n}$, by setting $\eta\leq  \frac{1}{10C n^{3/2}\sqrt{\log(np)} }$, we can lower bound the numerator of \eqref{eq:general_eq} as
\begin{equation}\label{eq:numerator}
\norm{\eta  \partial f(\bh)}_{\infty}/h_n\leq \frac{1}{10}  \quad \mbox{and}\quad \sbra{1 -\eta  [ \partial f(\bh)]_n /h_n }^2 \geq \frac{2}{3}.
\end{equation} 
To continue, we take a similar approach to \cite[Lemma D.1]{bai2019subgradient}, and divide our discussions of the denominator of \eqref{eq:general_eq} for different coordinates in three subsets: 
\begin{subequations}
\begin{align} 
        \cJ_0 &:= \left\{ i\in[n-1]:h_i=0               \right \} , \\
        \cJ_1 &:= \left\{ i\in[n-1]: \frac{h_n^2}{h_i^2}\geq 4 ,h_i \neq 0             \right \}, \\
        \cJ_2 &:= \left\{ i\in[n-1]: \frac{h_n^2}{h_i^2}\leq 4              \right \}. \label{eq:J_2}
\end{align}
\end{subequations}
\begin{itemize}
\item For any index $i\in \cJ_0$, we have $h_i =0$, and then by \eqref{eq:general_eq} and \eqref{eq:numerator},
$$ \sbra{\frac{ {h}^{+}_n}{h^{+}_i} }^2  = \frac{(1 -\eta [ \partial f(\bh)]_n /h_n )^2}{( \eta [ \partial f(\bh)]_i /h_n)^2} \geq \frac{2/3}{(1/10)^2} \geq 2. $$
\item For any index $i\in \cJ_1$, we have
$$ \sbra{\frac{ {h}^{+}_n}{h^{+}_i} }^2  = \frac{(1 -\eta [ \partial f(\bh)]_n /h_n )^2}{(h_i/h_n - \eta  [\partial f(\bh)]_i /h_n)^2} \geq \frac{2/3}{(1/4+1/10^2)} \geq 2. $$

\item For any index $i\in \cJ_2$, we have
$$ \sbra{\frac{ {h}^{+}_n}{h^{+}_i} }^2  = \frac{h_n^2}{h_i^2} \sbra{1 +\eta \frac{\partial f(\bh)^{\top} (\be_i/h_i -\be_n/h_n)}{1 - \eta [\partial f(\bh)]_i /h_i}}^2 . $$
Since $\bh \in \mathcal{H}_i$ as defined in Lemma~\ref{thm:implicit_uniform_concentration_general}, using \eqref{res: sins geometry}, we have $\partial f(\bh)^{\top} (\be_i/h_i -\be_n/h_n)\geq \frac{c_b}{2} \xi_0 \theta >0 $, and consequently,
    \begin{equation}
\sbra{\frac{ {h}^{+}_n}{h^{+}_i} }^2  \geq  \frac{h_n^2}{h_i^2} \sbra{1 +\eta \frac{\partial f(\bh)^{\top} (\be_i/h_i -\be_n/h_n)}{1 - \eta [\partial f(\bh)]_i /h_i}}^2 \geq \frac{h_n^2}{h_i^2} \geq 1+\xi_0,
\label{equ:xi_change}
\end{equation}
where the last inequality is due to $\bh \in\mathcal{S}_{\xi_0}^{(n+)}$.
\end{itemize}
Combining the above, we have that for all $i\in[n-1]$, $\sbra{{h}_n^{+}/h_i^+}^2\geq 1+\xi_0$, i.e, $\bh^{+}\in \cS_{\xi_0}^{(n+)}$.

\subsection{Proof of Theorem~\ref{thm:general_GD}}\label{proof:gd_for_orthogonal_case}

First, as the step size requirement satisfies that in Lemma \ref{thm:implicit_stay}, the iterates never jumps out of $\mathcal{S}_{\xi_0}^{(n+)}$, if initialized in it. Denote $\bh_u^+$ as the unnormalized update of $\bh$ with step size $\eta$ on the tangent space of $\bh$, i.e,
\begin{equation*}
    \bh_u^+=\bh- \eta \partial f(\bh)=\bh-\eta \sbra{\bI-\bh\bh^\top} \nabla_{\bh} f(\bh).
\end{equation*}
and $\bw_u^+$  the first $(n-1)$ entries of $\bh_u^{+}$, whose update can be written with respect to $\phi(\bw)$ as
\begin{align}
    \bw_u^+&=\bw-\eta \begin{bmatrix}
        \bI & \bm{0}
        \end{bmatrix}  \sbra{\bI-\bh\bh^\top} \nabla_{\bh} f(\bh) \nonumber\\
    &=\bw- \eta\sbra{\bI-\bw\bw^\top} \bJ_{\bh}(\bw) \nabla_{\bh} f(\bh) \nonumber\\
    &=\bw-\eta \sbra{\bI-\bw\bw^\top} \nabla_{\bw} \phi(\bw). \label{eq:update_w}
\end{align}
The normalized updates are respectively $\bh^+ =\bh_u^+/\norm{\bh_u^+}_2$ and $\bw^+ =  \bw_u^+/\norm{\bh_u^+}_2$. By the property that $\bh \perp \sbra{\bI-\bh\bh^\top} \nabla_{\bh} f(\bh)$, we have $\|\bh_u^+\|_2 \geq \|\bh\|_2\geq 1$.

\paragraph{Convergence in the region of $\mathcal{Q}_1 \cap \{\bw: \bh(\bw)\in \cS_{\xi_0}^{(n+)} \} $.} By \eqref{eq:update_w}, we have 
\begin{align}\label{eq:bound_in_region}
  \norm{\bw_u^+}_2^2 =\norm{\bw}_2^2 - \eta \underbrace{ h_n(\bw)^2  \bw^\top \nabla_{\bw} \phi(\bw) }_{I_1} +   \eta^2  \underbrace{ \norm{\sbra{\bI -\bw\bw^\top} \nabla_{\bw} \phi(\bw)}_2^2 }_{I_2}.
\end{align}
\begin{itemize}
\item First, $I_1$ can be bounded as 
\begin{equation*}
    I_1 = h_n^2\bw^\top \nabla_{\bw} \phi(\bw) \geq c_1 h_n^2\norm{\bw}_2 \xi_0 \theta.
\end{equation*}
for some constant $c_1$, where the last inequality owes to \eqref{eq:lg_pop_geometry}.

\item Second,  $I_2$ can be bounded as
\begin{equation*}
    I_2    \leq \norm{\begin{bmatrix}
        \bI & \bm{0}
        \end{bmatrix}} \cdot \norm{\bI-\bh\bh^\top} \norm{\nabla_{\bh} f(\bh)}_2 \leq 1\cdot (1+\norm{\bh}_2^2) \norm{\nabla_{\bh} f(\bh)}_2 \leq c_2 n \sqrt{\log(np)}
\end{equation*}
for some constant $c_2$, where the last inequality follows from \eqref{eq:bound_grad} in Lemma~\ref{thm:implicit_uniform_concentration_general}, which holds with probability at least $1- (np)^{-8}$.
\end{itemize}

Sum up the above results for $I_1$ and $I_2$, we have with probability at least $1- (np)^{-8}$,
\begin{equation*}
\norm{\bw^+}_2 \leq   \norm{\bw_u^+}_2^2 \leq \norm{\bw}_2^2 - \eta c_1 h_n^2(\bw) \norm{\bw}_2 \eta \xi_0 \theta +  c_2  \eta^2 n  \sqrt{\log (np)},
\end{equation*}
where the first inequality follows from $\norm{\bh_u^+}_2\geq 1$. Setting $\eta\leq \frac{c  \mu \xi_0\theta }{ n^2 \sqrt{\log(np)}  }  \leq  \frac{c_1 h_n^2\norm{\bw}_2\xi_0\theta }{2c_2 n \sqrt{\log(np)}  } $ for some sufficiently small $c$, we have
\begin{equation}
\norm{\bw^+}_2 \leq \norm{\bw}_2^2 - \eta \frac{c_1}{2} h_n^2(\bw) \norm{\bw}_2 \eta \xi_0 \theta  \leq \norm{\bw}_2^2 - \eta \frac{c_1}{2n}  \norm{\bw}_2   \xi_0 \theta .
\end{equation}
Denote the $k$-th iteration as $\bh^{(k)}$, we have that as long as $\bw^{(k)}\in\mathcal{Q}_1$, 
\begin{equation*}
\norm{\bw^{(k+1)}}_2^2 \leq \norm{\bw^{(k)}}_2^2- \eta \frac{c_1}{2n}  \norm{\bw^{(k)}}_2   \xi_0 \theta \leq \norm{\bw^{(k)}}_2^2- \eta \frac{c_1\mu}{8\sqrt{2} n}    \xi_0 \theta.
\end{equation*}
Telescoping the above inequality for the first $T_1$ iterations, we have
\begin{equation*} 
        T_1\cdot  \eta \frac{c_1\mu}{8\sqrt{2} n}    \xi_0 \theta   \leq \norm{\bw^{(0)}}_2^2-\norm{\bw^{(T_1)}}_2^2\leq 1, \quad \Longrightarrow \quad   T_1\leq \frac{Cn}{ \mu \eta \xi_0 \theta},
\end{equation*}
which means it takes at most $T_1$ iterations to enter $\mathcal{Q}_2\cap \{\bw: \bh(\bw)\in \cS_{\xi_0}^{(n+)} \} $.
 
\begin{remark}
Using arguments similar to  \cite{gilboa2019efficient}, the iteration complexity in this phase can be improved to $T_1\lesssim \frac{1}{\eta} \sbra{ \frac{n}{ \mu\theta}+ \log\left(\frac{1}{\xi_0}\right)} $ which only has a logarithmic dependence on $\xi_0$; we didn't pursue it here as it only leads to a logarithmic improvement to the overall complexity due to our choice of $\xi_0=O(1/\log n)$.
\end{remark}

\paragraph{Convergence in the region of $\mathcal{Q}_2\cap \{\bw: \bh(\bw)\in \cS_{\xi_0}^{(n+)} \}$:}

Denoting the unique local minima in $\mathcal{Q}_2\cap \{\bw: \bh(\bw)\in \cS_{\xi_0}^{(n+)} \} $ as $\bw^{\star}$, whose norm is bounded in Theorem~\ref{theo:general_geometry}. By setting $p$ sufficiently large, we can ensure that the iterates stay in $\mathcal{Q}_2$ following a similar argument as \eqref{eq:bound_in_region}. To begin, we note that $\nabla_{\bw}\phi(\bw) $ is $L$- Lipschitz with
\begin{equation}\label{eq:Lipschitz_phi}
    L\leq \frac{Cn^{3/2} \log (np)}{\mu},
\end{equation}
which is proved in Appendix~\ref{proof:lipschitz_phi}, and $c_1 n \theta /\mu$-strongly convex in $\mathcal{Q}_2\cap \{\bw: \bh(\bw)\in \cS_{\xi_0}^{(n+)} \}$. For $\bw\in\mathcal{Q}_2$, we have 
\begin{equation}\label{eq:spectral_bound}
  \frac{1}{2}  \leq 1-\frac{\mu^2}{32} \leq \norm{\bI-\bw\bw^\top } \leq 1+ \frac{\mu^2}{32},
\end{equation}
with $\mu< c_1 \min\{\theta,    \xi_0^{1/6} n^{-3/4} \} \leq 4$ sufficiently small.

We now consider two cases based on the size of $\|\bw^{+}\|_2$, which is the next iterate with respect to $\bw$ given in \eqref{eq:update_w}. 
\begin{enumerate}
\item $\|\bw^{+}\|_2< \|\bw^{\star}\|_2$: in this case, we already achieve 
$$\norm{\bw^+ -\bm{0} }_2 \leq \norm{\bw^{\star}}_2  \lesssim \frac{\kappa^4}{\theta^2} \sqrt{\frac{n\log^3 p \log^2 n}{p}}  .$$ 

\item $\|\bw^{+}\|_2\geq \|\bw^{\star}\|_2$: by the fundamental theorem of calculus, we have
\begin{align}
    \norm{ \bw_u^+ - \bw^{\star} }_2 &=\norm{\bw-\bw^{\star} - \eta \sbra{\bI-\bw\bw^\top} \nabla_{\bw} \phi(\bw)}_2 \nonumber\\
    &=\norm{\mbra{ \bI- \eta \sbra{\bI-\bw\bw^\top} \int_0^1 \nabla^2_{\bw} \phi(\bw(t)) dt }  \sbra{\bw-\bw^{\star}}  }_2  \nonumber\\
    &\leq \norm{\mbra{ \bI- \eta \sbra{\bI-\bw\bw^\top} \int_0^1 \nabla^2_{\bw} \phi(\bw(t)) dt }  } \norm{ \bw-\bw^{\star} }_2  \nonumber\\
    &\leq \sbra{1-  \frac{c_1n\theta \eta}{2\mu} } \norm{ \bw-\bw^{\star} }_2. \label{equ:convergence_of_tlide{w}}
\end{align}
where $\bw(t):= \bw+ t(\bw^{\star}-\bw)$, $t\in [0,1]$, and the step size $\eta\leq \frac{c  \mu \xi_0\theta }{n^2 \sqrt{\log(np)}  } \leq \frac{1}{2L}$. Moreover, since $\bw^+ = \bw_u^{+}/ \norm{\bh^+}_2 =\bw_u^{+}/(1+K)$ for some $K>0$, we have
\begin{align}
    \norm{\bw_u^+ -\bw^{\star}}_2^2&=\norm{(1+K)\bw^+-\bw^{\star}}_2^2 \nonumber\\
    &=\norm{\bw^+-\bw^{\star}}_2^2+ (2K+K^2) \norm{\bw^+}_2^2- 2K \bw^{\star\top} \bw^+ \nonumber\\
    &\geq \norm{\bw^+-\bw^{\star}}_2^2+ (2K+K^2) \norm{\bw^+}_2^2 - 2 K \norm{\bw^+}_2\norm{\bw^{\star}}_2 \geq \norm{\bw^+-\bw^{\star}}_2^2 \label{equ:convergence_of_w+}
\end{align}
where the last inequality owes to $\norm{\bw^+}_2\geq \norm{\bw^{\star}}_2$. Combining \eqref{equ:convergence_of_tlide{w}} and \eqref{equ:convergence_of_w+}, we have the update $\bw^+$ satisfies
\begin{equation}\label{equ:convergence_rate}
    \norm{\bw^+-\bw^{\star}}_2 \leq \sbra{1-  \frac{c_1n\theta \eta}{2\mu} } \norm{ \bw-\bw^{\star} }_2.
\end{equation}
 Therefore, to ensure $    \norm{\bw^+-\bw^{\star}}_2  \leq \epsilon$, it takes no more than
\begin{equation}
    T_2\leq  \frac{2c_1\mu}{n\theta\eta}\log\sbra{ \frac{ 3\mu}{2\sqrt{2}\epsilon} }
\end{equation}
iterations, since for any $\bw \in \mathcal{Q}_2$, we have $\norm{\bw -\bw^*}_2 \leq \norm{\bw}_2 + \norm{\bw^*}_2 \leq \frac{3\mu}{2\sqrt{2}}$.
\end{enumerate}
Summing up, to achieve $\|\bw^{(T)}-\bm{0}\|_2 \lesssim \frac{\kappa^4}{\theta^2} \sqrt{\frac{n\log^3 p \log^2 n}{p}}  + \epsilon$, the total number of iterates is bounded by
$$T=T_1 +T_2 \lesssim \frac{n}{ \mu \eta \xi_0 \theta}+  \frac{ \mu}{n\theta\eta}\log\sbra{ \frac{ \mu}{ \epsilon} }.$$

We next translate this bound into bounds of $ \norm{\bh^{(T)} - \be_n}_2$ and $\mathrm{dist}(\widehat{\bg}_{\inv}, \bg_{\mathrm{inv}})$, where the latter leads to Corollary~\ref{cor:main_corollary}.
For notation simplicity, we denote $\beta:= \frac{ \kappa^4}{\theta^2} \sqrt{\frac{n\log^3 p \log^2 n}{p}} + \epsilon <1$ which holds for sufficiently large sample size.\footnote{Note that $\|\bw^{(T)}\|_2<1$ by the spherical constraint, the bound is vacuous otherwise} First, observe that $\bh^{(T)}=\bh(\bw^{(T)})$ satisfies
\begin{align}
    \norm{\bh^{(T)} - \be_n}_2 &= \sqrt{\norm{\bw^{(T)}}_2^2 +  \sbra{ 1- \sqrt{1-\norm{\bw^{(T)}}_2^2}}^2} \nonumber\\
    & \leq \sqrt{\sbra{ 1- \sqrt{1-\beta^2}}^2 + \beta^2} \leq \sqrt{2\beta^2} \lesssim \frac{\kappa^4}{\theta^2} \sqrt{\frac{n\log^3 p \log^2 n}{p}} + \epsilon. \label{eq:bound_hT_en}
\end{align}

Since we consider the loss function $f(\bh) $ in \eqref{equ:equvalent_f2} throughout the proof, the estimate of $\bm{g}_{\inv}$ is given by $\widehat{\bg}_{\inv} = \bR\bU^\top \bh^{(T)}$. Hence, we have
\begin{align}
\mathrm{dist}(\widehat{\bg}_{\inv}, \bg_{\mathrm{inv}}) & =\mathrm{dist}(\bR\bU^\top \bh^{(T)}, \cC(\bg_{\mathrm{inv}})\be_n)   \leq \norm{\cC(\bg)^{-1}}\norm{\cC(\bg)\bR\bU^\top\bh^{(T)} - \be_n}_2 \nonumber\\
&\leq \norm{\cC(\bg)^{-1}} \norm{\cC(\bg)\bR\bU^\top (\bh^{(T)} - \be_n)}_2 + \norm{\cC(\bg)^{-1}}\norm{(\cC(\bg)\bR\bU^\top - \bI)\be_n}_2 \nonumber\\
& \leq \norm{\cC(\bg)^{-1}} \norm{ \bDelta + \bI } \norm{\bh^{(T)} - \be_n}_2 + \norm{\cC(\bg)^{-1}}\norm{\bDelta}\norm{\be_n}_2 \label{equ:distance bound}
\end{align}
where we used the definition of $\bDelta= \cC(\bg)\bR\bU^\top - \bI$ (cf. \eqref{equ:define_delta}). Under the sample size requirement, we have $\norm{\bDelta} \leq c_f \kappa^4 \sqrt{\frac{\log^2 n \log p}{\theta^2 p}} < 1$ with probability at least $1-2n p^{-8}$ by Lemma~\ref{lemma:deviation op norm}. Plugging it into \eqref{equ:distance bound}, we have
\begin{align}
    \mathrm{dist}(\widehat{\bg}_{\inv}, \bg_{\mathrm{inv}}) &\leq 
 \frac{2}{\sigma_n(\cC(\bg))}  \norm{\bh^{(T)}- \be_n}_2 + \frac{1}{\sigma_n(\cC(\bg))}\norm{\bDelta} \nonumber\\
& \lesssim \frac{\kappa^4}{\theta^2\sigma_n(\cC(\bg))} \sqrt{\frac{n\log^3 p \log^2 n}{p}} + \frac{\epsilon}{\sigma_n(\cC(\bg))},
\end{align}
where the last line follows from \eqref{eq:bound_hT_en}.

\subsubsection{Proof of \eqref{eq:Lipschitz_phi}}\label{proof:lipschitz_phi}

Recalling ${\psi}_\mu^{\prime}(x)=\tanh(x/\mu)$, for any $\bw_1, \bw_2 \in \mathcal{Q}_2$, using the expression
$$  \nabla_{\bw} \phi(\bw) =  \frac{1}{p}\sum_{i=1}^p  \bJ_{\bh}(\bw)  \sbra{\bI+\bDelta}^\top    \cC(\bx_i)^\top   {\psi}_\mu^{\prime}\sbra{\mathcal{C}(\bx_i) \sbra{ \bI+\bDelta} \bh(\bw)}, $$
we have
\begin{align*}
    \nabla_{\bw} \phi(\bw_1) & - \nabla_{\bw} \phi(\bw_2) \\
&= \underbrace{ \frac{1}{p}\sum_{i=1}^p \mbra{\bJ_{\bh}(\bw_1)-\bJ_{\bh}(\bw_2) } \sbra{\bI+\bDelta}^\top    \cC(\bx_i)^\top   {\psi}_\mu^{\prime}\sbra{\mathcal{C}(\bx_i) \sbra{ \bI+\bDelta} \bh(\bw_1)} }_{\bg_1}\\
&\quad + \underbrace{ \frac{1}{p}\sum_{i=1}^p \bJ_{\bh}(\bw_2)  \sbra{\bI+\bDelta}^\top    \cC(\bx_i)^\top  \mbra{ {\psi}_\mu^{\prime}\sbra{\mathcal{C}(\bx_i) \sbra{ \bI+\bDelta} \bh(\bw_1)} - {\psi}_\mu^{\prime}\sbra{\mathcal{C}(\bx_i) \sbra{ \bI+\bDelta} \bh(\bw_2)}  } }_{\bg_2}.
\end{align*}

We bound $\norm{\bg_1}_2$ and $\norm{\bg_2}_2$ respectively.
Under the sample size requirement, from Lemma~\ref{lemma:deviation op norm}, we can ensure $\|\bDelta\| \leq 1$. To bound $\norm{\bg_1}_2$, we have
\begin{align*}
 \norm{\bg_1}_2&\leq \norm{ \bJ_{\bh}(\bw_1)-\bJ_{\bh}(\bw_2) }\cdot \norm{\bI+ \bDelta}\cdot \max_{i\in[p]} \norm{\cC(\bx_i)}\cdot \norm{ \tanh\sbra{\frac{\mathcal{C}(\bx_i)\sbra{\bI+\bDelta}  \bh(\bw_1)}{\mu}} }_2\\
 &\leq C\sqrt{n^2 \log(np)} \norm{\begin{bmatrix} 
 \bm{0} & \sbra{\frac{\bw_2}{h_n (\bw_2)}-\frac{\bw_1}{h_n (\bw_1)} } \end{bmatrix} },
\end{align*}
where the second line follows from $\abs{\tanh(\cdot)}\leq 1$ and \eqref{eq:circulant_x}, which holds with probability at least $1-(np)^{-8}$. To continue, we observe that
\begin{align*}
    \norm{\begin{bmatrix} 
 \bm{0}, & \sbra{\frac{\bw_2}{h_n (\bw_2)}-\frac{\bw_1}{h_n (\bw_1)} } \end{bmatrix} }&=\norm{ \frac{\bw_2}{h_n (\bw_2)}-\frac{\bw_1}{h_n (\bw_1)} }_2\\
 &\leq \norm{ \frac{\bw_2}{h_n (\bw_2)}-\frac{\bw_1}{h_n (\bw_2)}  }_2 + \norm{  \frac{\bw_1}{h_n (\bw_2)}-\frac{\bw_1}{h_n (\bw_1)}    }_2\\
 &\overset{\mathrm{(i)}}{\leq} 2\norm{\bw_2-\bw_1}_2 + \norm{\bw_1}_2 \left| \frac{1}{h_n(\bw_2)} - \frac{1}{h_n(\bw_1)}   \right| \\
 &\overset{\mathrm{(ii)}}{\leq} 2\norm{\bw_2-\bw_1}_2 + 8  \norm{\bw_2-\bw_1}_2 \\
 &\leq 10 \norm{\bw_2-\bw_1}_2
\end{align*}
where $\mathrm{(i)}$ follows from $h_n(\bw_2) = \sqrt{1-\|\bw_2\|_2^2}\geq 1/2$ when we used the fact that $\|\bw\|_2\leq \mu/(4\sqrt{2}) \leq \sqrt{3}/2$ in $\mathcal{Q}_2$, and $\mathrm{(ii)}$ follows from the Lipschitz smoothness of $1/h_n(\bw)$ and $h_n(\bw)\geq 1/2$:
\begin{equation}\label{eq:lipschitz_hn}
\left| \frac{1}{h_n(\bw_2)} - \frac{1}{h_n(\bw_1)}   \right| \leq \left(  \max_{\bw\in\mathcal{Q}_2} \frac{1}{(1-\|\bw\|_2^2)^{3/2}} \right) \| \bw_2 -\bw_1\|_2 \leq 8  \| \bw_2 -\bw_1\|_2.
\end{equation}
Therefore, we have
\begin{equation*}
    \norm{\bg_1}_2\leq C \sqrt{n^2 \log(np)}\norm{\bw_2-\bw_1}_2 .
\end{equation*} 
To bound $\norm{\bg_2}_2$, we have
\begin{align*}
\| \bg_2 \|_2 & \leq \norm{\bJ_{\bh}(\bw_2)}  \norm{\bI+\bDelta} \cdot \max_{i\in[p]} \norm{\cC(\bx_i)}\cdot \max_{i\in[p]}\norm{\tanh\sbra{\frac{\mathcal{C}(\bx_i)\sbra{\bI+\bDelta}  \bh(\bw_1)}{\mu}} -\tanh\sbra{\frac{\mathcal{C}(\bx_i) \sbra{\bI+\bDelta} \bh(\bw_2)}{\mu}}      }_2  \nonumber \\
& \leq  \frac{1}{\mu}\norm{\bJ_{\bh}(\bw_2)}   \norm{\bI+\bDelta}^2 \cdot \max_{i\in[p]} \norm{\cC(\bx_i)}^2 \cdot \norm{\bh(\bw_1)-\bh(\bw_2)}_2\\
&\leq \frac{Cn^{3/2} \log (np)}{\mu} \norm{\bw_1-\bw_2}_2,
\end{align*}
where the second line follows from a similar argument in \eqref{eq:lipschtiz_tanh_vector}, the last line follows from \eqref{eq:jacobian_bound} and \eqref{eq:circulant_x}, which holds with probability at least $1-(np)^{-8}$, and 
\begin{align*}
\norm{\bh(\bw_1)-\bh(\bw_2)}_2\leq \|\bw_1-\bw_2\|_2 + \left| \frac{1}{h_n(\bw_2)} - \frac{1}{h_n(\bw_1)}   \right| \leq 9\|\bw_1-\bw_2\|_2 ,
\end{align*} 
following \eqref{eq:lipschitz_hn}. Combining the bounds on $\norm{\bg_1}_2$ and $\norm{\bg_2}_2$ achieve the desired result.

\subsection{Proof of Lemma \ref{thm:implicit_uniform_concentration_orthogonal}}\label{proof:implicit_uniform_concentration_orthogonal}

The proof follows a standard covering argument similar to the proof of Theorem~\ref{thm:orthogonal_geometry} in Appendix~\ref{proof:orthogonal_geometry_discretization}. To begin, we need the following propositions, proved in Appendix~\ref{proof:expectation_stay_in_region}, \ref{proof:pointwise_stay_in_region}, and \ref{proof:lipschitz stay in subsets}, respectively. 

\begin{prop} \label{pro:expectation_stay_in_region}
 For any $\xi_0\in (0,1)$, $\theta\in (0,\frac{1}{3})$, $k\in [n-1]$, there exists some constant $c_1$ such that when $\mu< c_1 \min\{\theta, \xi_0^{1/6}n^{-3/4}\}$, for any $\bh\in \mathcal{H}_k$, we have  
    \begin{equation*}
        \mathbb{E} \partial f_o(\bh)^\top \sbra{\frac{\be_k}{h_k}-\frac{\be_n}{h_n}}  \geq  \frac{\theta\xi_0}{24\sqrt{2\pi}}.
    \end{equation*}
\end{prop}

\begin{prop}\label{pro:pointwise_stay_in_region}
 For any $\xi_0\in (0,1)$, $\theta\in (0,\frac{1}{3})$, $k\in [n-1]$, there exists some constant $c_1$ such that when $\mu< c_1 \min\{\theta, \xi_0^{1/6}n^{-3/4}\}$, for any fixed $\bh\in \mathcal{H}_k$, there exists some constant $C$ such that  for any $t>0$
    \begin{equation*}
        \mathbb{P} \sbra{\abs{\partial f_o(\bh)^{\top}\sbra{\frac{\be_k}{h_k}-\frac{\be_n}{h_n}}-\mathbb{E}\partial f_o(\bh)^{\top} \sbra{\frac{\be_k}{h_k}-\frac{\be_n}{h_n}}}\geq t}
        \leq 2\exp\sbra{\frac{-p t^2}{2Cn^3 \log n+ 2t \sqrt{Cn^3\log n} }}.
    \end{equation*}
\end{prop}
 
\begin{prop}  \label{pro:lipschitz stay in subsets}
  For any $\xi_0\in (0,1)$, $\theta\in (0,\frac{1}{3})$, $k\in [n-1]$, $\partial f_o(\bh)^\top \sbra{\frac{\be_k}{h_k}-\frac{\be_n}{h_n}}$ is $L_3$-Lipschitz in the domain $\mathcal{H}_k$ with
    \begin{equation*}
        L_3\leq \max_{i\in [p]}\sbra{ \frac{\sqrt{5n}}{\mu} \norm{\cC(\bx_i)}^2 + 4 n^{3/2} \norm{\cC(\bx_i)} }.
    \end{equation*} 
\end{prop}

We now continue to the proof of Lemma \ref{thm:implicit_uniform_concentration_orthogonal}. In the subset $\mathcal{H}_k$, for any $0\leq \epsilon \leq 2\sqrt{\frac{n-1}{n+\xi_0}}$, we have an $\epsilon$-net $\mathcal{N}_3$ of size at most $\lceil 3/\epsilon \rceil^n$, where $\epsilon$ will be determined later. Under the event \eqref{eq:circulant_x} and Proposition~\ref{pro:lipschitz stay in subsets}, we have 
$$L_3\leq \frac{c_{10} n^{3/2}}{\mu}\log(np).$$
For all $\bh\in\mathcal{H}_k$, there exists $\bh'  \in\mathcal{N}_3$ such that $\norm{\bh' -\bh}_2 \leq \epsilon$. By Proposition~\ref{pro:lipschitz stay in subsets}, we have 
\begin{equation*}
    \abs{\partial f_o(\bh)^{\top} \sbra{\frac{\be_k}{h_k}-\frac{\be_n}{h_n} }-\partial f_o(\bh')^{\top} \sbra{\frac{\be_k}{h_k'}-\frac{\be_n}{h_n'} }}\leq  L_3 \norm{\bh' -\bh}_2 \leq\frac{c_{10} n^{3/2}}{\mu}\log(np) \epsilon \leq \frac{c_1\theta\xi_0}{3},
\end{equation*}
which holds when $\epsilon\leq \frac{c\mu \theta \xi_0}{n^{3/2}\log(np)}$ for some sufficiently small $c$. With this choice of $\epsilon$, the covering number of $\mathcal{N}_3$ satisfies
$$ \abs{\mathcal{N}_3} \leq \exp \sbra{ n \log\sbra{\frac{c  n^{3/2} \log(np)}{\mu \theta \xi_0}  }}.$$
Let $\mathcal{A}_3$ denote the event  
\begin{equation*}
  \mathcal{A}_3: =  \left\{ \max_{\bh\in\mathcal{H}_k }  \abs{\partial f_o(\bh)^{\top}\sbra{\frac{\be_k}{h_k}-\frac{\be_n}{h_n}}- \mathbb{E} \partial f_o(\bh)^{\top} \sbra{\frac{\be_k}{h_k}-\frac{\be_n}{h_n}}    }\leq  \frac{c_1\theta \xi_0}{3}         \right \}.
\end{equation*}
Setting $ t=\frac{c_1\theta \xi_0}{3} $ in Proposition~\ref{pro:pointwise_stay_in_region}, by the union bound,  $ \mathcal{A}_3$ holds with probability at least  
\begin{align*}
1-  \abs{\mathcal{N}_3}\cdot 2\exp\sbra{\frac{-p t^2}{2Cn^3 \log n+ 2t\sqrt{Cn^3\log n} }}
   & \geq 1- 2\exp\sbra{\frac{-c_{11} p \theta^2 \xi_0^2}{n^3\log n}+n \log\sbra{\frac{c_{10} n^{3/2} \log(np)}{\mu \theta \xi_0}  }} \\
    &  \geq 1- 2\exp\sbra{-c_{12} n},
\end{align*}
provided $p\geq C \frac{n^4\log n }{\theta^2\xi_0^2} \log\sbra{\frac{n^{3/2} \log(np) }{ \mu\theta\xi_0} }$. Finally, we have for all $\bh \in\mathcal{H}_k$,
\begin{align*} 
 \partial f_o(\bh)^{\top} \sbra{\frac{\be_k}{h_k}-\frac{\be_n}{h_n}}
    &=  \mbra{\partial f_o(\bh)^{\top}\sbra{\frac{\be_k}{h_k}-\frac{\be_n}{h_n}} - \partial f_o(\bh')^{\top} \sbra{\frac{\be_k}{h'_k}-\frac{\be_n}{h'_n}}  }\\
    &\quad+\mbra{\partial f_o(\bh')^{\top}\sbra{\frac{\be_k}{h'_k}-\frac{\be_n}{h'_n}}- \mathbb{E} \partial f_o(\bh')^{\top}  \sbra{\frac{\be_k}{h'_k}-\frac{\be_n}{h'_n}} } +\mathbb{E} \partial f_o(\bh')^{\top}\sbra{\frac{\be_k}{h'_k}-\frac{\be_n}{h'_n}}\\
    &\geq -\frac{c_1\theta \xi_0}{3}-\frac{c_1\theta \xi_0}{3}+ c_1\theta \xi_0 =\frac{c_1\theta \xi_0}{3}.
\end{align*}

\subsubsection{Proof of Proposition~\ref{pro:expectation_stay_in_region}}\label{proof:expectation_stay_in_region} 

First, recall a few notation introduced in Appendix~\ref{proof:population_geometry_orthogonal}. For $\bx = \bOmega \odot \bz \sim_{iid} \mathrm{BG}(\theta)\in \mathbb{R}^n$, we denote the first $n-1$ dimension of $\bx$, $\bz$ and $\bOmega$ as $\bar{\bx}$, $\bar{\bz}$ and $\bar{\bOmega}$, respectively. Denote $\mathcal{I}$ as the support of $\bm{\Omega}$ and $\mathcal{J}$ as the support of $\bar{\bOmega}$. For any $k\in[n-1]$ with $h_k\neq0$, by \eqref{eq:manifold_g_euclid_g} and \eqref{eq:gradient_phi_h}, we have
\begin{align}\label{eq:direction_gradf_exp} 
        \mathbb{E} \partial f_o(\bh)^\top \sbra{\frac{\be_k}{h_k}-\frac{\be_n}{h_n}}&=\mathbb{E}\nabla f_o(\bh)^\top \sbra{\frac{\be_k}{h_k}-\frac{\be_n}{h_n}} 
=n \cdot \mathbb{E}\nabla{\psi_{\mu}(\bx^\top \bh)}^\top  \sbra{\frac{\be_k}{h_k}-\frac{\be_n}{h_n}},
\end{align}
since the rows of $\mathcal{C}(\bx)$ has the same distribution as $\bx\sim_{iid} \mathrm{BG}(\theta)$. Further plugging in (\ref{equ:gradient_psi}), we rewrite it as:
\begin{align}
        \mathbb{E}\nabla{\psi_{\mu} (\bx^\top \bh) }^\top \sbra{\frac{\be_k}{h_k}-\frac{\be_n}{h_n}}   &= \mbE{\sbra{\frac{\be_k}{h_k}-\frac{\be_n}{h_n}}^\top \tanh\sbra{\frac{\bx^\top \bh}{\mu}} \bx } \nonumber\\
        &= \mbE{ \tanh\sbra{\frac{\bx^\top \bh}{\mu}}\frac{x_k}{h_k} } - \mbE{ \tanh\sbra{\frac{\bx^\top \bh}{\mu}}\frac{x_n}{h_n}  } \nonumber\\
        &=  \underbrace{\mathbb{E}_{\bOmega} \mathbb{E}_{\bz}\mbra{\tanh\sbra{\frac{\bx^\top \bh}{\mu}}\frac{x_k}{h_k} } }_{I_1}
        -  \underbrace{\mathbb{E}_{\bOmega } \mathbb{E}_{\bz}\mbra{\tanh\sbra{\frac{\bx^\top \bh}{\mu}}\frac{x_n}{h_n} } }_{I_2}. \label{equ:I1_I_2}
\end{align}

Evaluating $\mathbb{E}_{\bOmega}$ over $\Omega_k$, $\Omega_n$, and  $\bar{\bOmega}_{\backslash \{k\} }$ sequentially, we can express $I_1$, $I_2$ respectively as:
\small
\begin{align*}
    I_1&=\theta \mathbb{E}_{\bOmega_{\setminus \{ k\}}} \mathbb{E}_{\bz} \mbra{\tanh\sbra{\frac{\bx_{\setminus \{k\}}^\top \bh_{\backslash\{k\}} + h_k z_k}{\mu}} \frac{z_k}{h_k}  }\\
    &=\theta (1-\theta) \underbrace{ \mathbb{E}_{\bar{\bOmega}_{\setminus \{ k\}}} \mathbb{E}_{\bz} \mbra{\tanh\sbra{\frac{\bar{\bx}_{\backslash\{k\}}^\top \bw_{\backslash\{k\}} + h_k z_k}{\mu}} \frac{z_k}{h_k} } }_{I_{11}}
    +\theta^2 \underbrace{ \mathbb{E}_{\bar{\bOmega}_{\setminus \{ k\}}} \mathbb{E}_{\bz} \mbra{\tanh\sbra{\frac{\bar{\bx}_{\backslash\{k\}}^\top \bw_{\backslash\{k\}} + h_k z_k + h_n z_n}{\mu}} \frac{z_k}{h_k} } }_{I_{12}}, \\
    I_2&=(1-\theta) \mathbb{E}_{\bOmega_{\setminus \{ k\}}} \mathbb{E}_{\bz} \mbra{\tanh\sbra{\frac{\bx_{\setminus \{k\}}^\top \bh_{\backslash\{k\}}}{\mu}} \frac{x_n}{h_n}  }
    +\theta \mathbb{E}_{\bOmega_{\setminus \{ k\}}} \mathbb{E}_{\bz} \mbra{\tanh\sbra{\frac{\bx_{\setminus \{k\}}^\top \bh_{\backslash\{k\}} + h_k z_k}{\mu}} \frac{x_n}{h_n}  }\\
    &=\theta (1-\theta) \underbrace{ \mathbb{E}_{\bar{\bOmega}_{\setminus \{ k\}}} \mathbb{E}_{\bz} \mbra{\tanh\sbra{\frac{\bar{\bx}_{\backslash\{k\}}^\top \bw_{\backslash\{k\}} + h_n z_n}{\mu}} \frac{z_n}{h_n} }}_{I_{21}}
     +\theta^2 \underbrace{ \mathbb{E}_{\bar{\bOmega}_{\setminus \{ k\}}} \mathbb{E}_{\bz} \mbra{\tanh\sbra{\frac{\bar{\bx}_{\backslash\{k\}}^\top \bw_{\backslash\{k\}} + h_k z_k + h_n z_n}{\mu}} \frac{z_n}{h_n} }}_{I_{22}}.
\end{align*}
\normalsize
Introduce the short-hand notation $X_1=h_k z_k\sim \mathcal{N}(0,h_k^2)$, $Y_1= \bar{\bx}_{\backslash\{k\}}^\top \bw_{\backslash\{k\}} + h_n z_n $, $X_2=h_n z_n\sim \mathcal{N}(0,h_n^2)$,  $Y_2=\bar{\bx}_{\backslash\{k\}}^\top \bw_{\backslash\{k\}} + h_k z_k $. Invoking Lemma~\ref{lemma:loss_expectation}, the difference of the second terms of $I_1$ and $I_2$ is
\begin{align*}
I_{12}-I_{22}&= \mathbb{E}_{\bar{\bOmega}_{\setminus \{ k\}}} \mbra{ \mathbb{E}_{X_1, Y_1} \sbra{ \tanh\sbra{\frac{X_1+ Y_1}{\mu}} \frac{X_1}{h_k^2}  }}  -  \mathbb{E}_{\bar{\bOmega}_{\setminus \{ k\}}} \mbra{ \mathbb{E}_{X_2, Y_2} \sbra{ \tanh\sbra{\frac{X_2+ Y_2}{\mu}} \frac{X_2}{h_n^2}  }}\\
     &=\frac{1}{\mu}\mathbb{E}_{\bar{\bOmega}_{\setminus \{ k\}}} \mbra{\mathbb{E}_{X_1, Y_1}\sbra{1-\tanh^2\sbra{\frac{X_1+Y_1}{\mu} }     }   }-\frac{1}{\mu}\mathbb{E}_{\bar{\bOmega}_{\setminus \{ k\}}} \mbra{\mathbb{E}_{X_2, Y_2}\sbra{1-\tanh^2\sbra{\frac{X_2+Y_2}{\mu} }     }   } = 0.
\end{align*}
Consequently, we have 
\begin{align*}
& \mathbb{E}\nabla \psi_{\mu}(\bx^\top \bh)^\top \sbra{\frac{\be_k}{h_k}-\frac{\be_n}{h_n}}  =\theta (1-\theta) ( I_{11} - I_{21}) \\
    &\qquad=\frac{\theta (1-\theta)}{\mu}\mathbb{E}_{\bar{\bOmega}_{\setminus \{ k\}}}\mathbb{E}_{\bz}  \mbra{\sbra{1-\tanh^2\sbra{\frac{\bar{\bx}_{\backslash\{k\}}^\top \bw_{\backslash\{k\}} + h_k z_k}{\mu}}     }  -   \sbra{1-\tanh^2\sbra{\frac{\bar{\bx}_{\backslash\{k\}}^\top \bw_{\backslash\{k\}} + h_n z_n}{\mu}}     }   }  \\
    &\qquad \geq \frac{\theta (1-\theta)}{\mu}  \frac{   \xi_0}{16 \sqrt{2\pi}n }  = \frac{\theta\xi_0}{24n\sqrt{2\pi}},
\end{align*}
where the second line follows from Lemma~\ref{lemma:loss_expectation}, and the last line follows from \eqref{equ:bound_for_K} and $\theta\in(0,1/3)$. Finally, we have
\begin{equation*}
    \mathbb{E}  \partial f_o(\bh)^\top \sbra{\frac{\be_k}{h_k}-\frac{\be_n}{h_n}}=n  \mathbb{E}\nabla \psi_{\mu}(\bx^\top \bh)^\top \sbra{\frac{\be_k}{h_k}-\frac{\be_n}{h_n}}  \geq  \frac{\theta\xi_0}{24\sqrt{2\pi}} .
\end{equation*}

\subsubsection{Proof of Proposition~\ref{pro:pointwise_stay_in_region}}\label{proof:pointwise_stay_in_region}

We start by writing the directional gradient as a sum of $p$ i.i.d. random variables: 
\begin{equation} \label{eq:directional_gradh}
    \partial f_o(\bh)^{\top} \sbra{\frac{\be_k}{h_k}-\frac{\be_n}{h_n}}= \nabla f_o(\bh)^{\top}\sbra{\frac{\be_k}{h_k}-\frac{\be_n}{h_n}} =\frac{1}{p}\sum_{i=1}^p\underbrace{ \tanh\sbra{\frac{\mathcal{C}(\bx_i)\bh }{\mu}}^{\top}\mathcal{C}(\bx_i)  \sbra{\frac{\be_k}{h_k}-\frac{\be_n}{h_n}} }_{Z_i},
\end{equation}
where the first equality is due to \eqref{eq:manifold_g_euclid_g} and the second equality is due to \eqref{eq:gradient_phi_h}. Moreover,
\begin{align*} 
    |   Z_i    |  & \leq \norm{ \tanh\sbra{\frac{\mathcal{C}(\bx_i)\bh}{\mu}}}_2   \norm{\frac{\be_k}{h_k}-\frac{\be_n}{h_n}}_2   \norm{\cC(\bx_i)}   \leq \sqrt{n}\norm{\frac{\be_k}{h_k}-\frac{\be_n}{h_n}}_2   \norm{\cC(\bx_i)}   \leq \sqrt{5} n \norm{\cC(\bx_i)} ,
\end{align*}
where the second inequality follows from $|\tanh(\cdot)|\leq 1$ and the third inequality follows from \eqref{eq:bound_directional_grad}. Therefore, for any $m\geq 2$, the moments of $|Z_i|$ can be controlled by Lemma \ref{lemma:Operator norm of Circulant Matrix} as
\begin{equation*}
    \mathbb{E}\abs{Z_i}^m \leq \sbra{\sqrt{5} n }^m \mathbb{E} \norm{\cC(\bx)}^m \leq  \frac{m!}{2}  \sbra{C n^3 \log n}^{m/2}.
\end{equation*}
The proof is then completed by setting $\sigma^2=C n^3 \log n$, $R=\sqrt{C n^3 \log n}$ and applying the Bernstein's inequality in Lemma \ref{lemma:matrix_bernstein}.

\subsubsection{Proof of Proposition \ref{pro:lipschitz stay in subsets}}\label{proof:lipschitz stay in subsets}

Using \eqref{eq:directional_gradh}, we have for any $\bh, \bh'$,
\begin{align*} 
    &\quad \abs{ \partial f_o(\bh)^{\top} \sbra{\frac{\be_k}{h_k}-\frac{\be_n}{h_n}}-\partial f_o(\bh')^{\top} \sbra{\frac{\be_k}{h'_k}-\frac{\be_n}{h'_n}} } \\
    &  =\frac{1}{p}\sum_{i=1}^p \abs{ \tanh\sbra{\frac{\mathcal{C}(\bx_i)\bh }{\mu}}^{\top}\mathcal{C}(\bx_i)  \sbra{\frac{\be_k}{h_k}-\frac{\be_n}{h_n}}  -  \tanh\sbra{\frac{\mathcal{C}(\bx_i)\bh' }{\mu}}^{\top}\mathcal{C}(\bx_i)  \sbra{\frac{\be_k}{h'_k}-\frac{\be_n}{h'_n}} }  \\
 & \leq \frac{1}{p}\sum_{i=1}^p \underbrace{ \abs{ \mbra{ \tanh\sbra{\frac{\mathcal{C}(\bx_i)\bh }{\mu} } - \tanh\sbra{\frac{\mathcal{C}(\bx_i)\bh' }{\mu} } }^{\top}\mathcal{C}(\bx_i)  \sbra{\frac{\be_k}{h_k}-\frac{\be_n}{h_n}}  } }_{A_i} \\
 & \quad\quad\quad\quad + \frac{1}{p}\sum_{i=1}^p \underbrace{ \abs{  \tanh\sbra{\frac{\mathcal{C}(\bx_i)\bh' }{\mu}}^{\top}\mathcal{C}(\bx_i) \mbra{ \sbra{\frac{\be_k}{h'_k}-\frac{\be_n}{h'_n}} -\sbra{\frac{\be_k}{h_k}-\frac{\be_n}{h_n}} }}  }_{B_i},
\end{align*}
 where the second line follows by the triangle inequality. In the sequel, we'll bound $A_i$ and $B_i$ respectively.

\begin{itemize}
\item To bound $A_i$, we have
\begin{align*}
A_i & \leq \norm{ \tanh\sbra{\frac{\mathcal{C}(\bx_i)\bh }{\mu} } - \tanh\sbra{\frac{\mathcal{C}(\bx_i)\bh' }{\mu} } }_2  \norm{\mathcal{C}(\bx_i) }  \norm{\frac{\be_k}{h_k}-\frac{\be_n}{h_n}}_2 \\
& \leq    \frac{\sqrt{5n}}{\mu} \norm{\cC(\bx_i)}^2 \norm{ \bh - \bh' }_2,
\end{align*}
where the second line follows from \eqref{eq:lipschtiz_tanh_vector} and \eqref{eq:bound_directional_grad}.
\item To bound $B_i$, we have
\begin{align*}
B_i & \leq \norm {\tanh\sbra{\frac{\mathcal{C}(\bx_i)\bh' }{\mu}} }_2   \norm{ \mathcal{C}(\bx_i) }  \norm{ \sbra{\frac{\be_k}{h'_k}-\frac{\be_n}{h'_n}} -\sbra{\frac{\be_k}{h_k}-\frac{\be_n}{h_n}} }_2 \\
& \leq  \sqrt{n} \norm{\cC(\bx_i)} \sqrt{\sbra{\frac{1}{h'_k} - \frac{1}{h_k}}^2 +\sbra{\frac{1}{h'_n} - \frac{1}{h_n}}^2 } \\
& \leq  4 n^{3/2}  \norm{\cC(\bx_i)} \sqrt{\sbra{h_k - h'_k }^2 +\sbra{ h_n -  h'_n}^2 } \leq 4 n^{3/2}  \norm{\cC(\bx_i)} \norm{\bh -\bh'}_2
\end{align*}
where the first line follows from $|\tanh(\cdot)|\leq 1$, and the second line follows from $h_n, h_n'\geq1/\sqrt{n}$ and $ h_n^2/h_k^2<4, (h_n')^2/(h_k')^2<4$.
\end{itemize}

  Combining terms, we have 
\begin{align*}
 \abs{ \partial f_o(\bh)^{\top} \sbra{\frac{\be_k}{h_k}-\frac{\be_n}{h_n}}-\partial f_o(\bh')^{\top} \sbra{\frac{\be_k}{h'_k}-\frac{\be_n}{h'_n}} }  & \leq  \max_{i\in [p]}\sbra{ \frac{\sqrt{5n}}{\mu} \norm{\cC(\bx_i)}^2 + 4 n^{3/2} \norm{\cC(\bx_i)} }  \norm{\bh -\bh'}_2 .
\end{align*}

\end{document}